\documentclass{article}
\usepackage{iclr2027_conference,times}
\usepackage[T1]{fontenc}
\usepackage[utf8]{inputenc}
\usepackage{amsmath,amssymb,amsthm,booktabs,tabularx,longtable,array}
\usepackage{graphicx,xcolor,microtype}
\usepackage{colortbl}
\definecolor{auditShadeA}{gray}{1.00}
\definecolor{auditShadeB}{gray}{0.96}
\definecolor{auditShadeC}{gray}{0.92}
\definecolor{auditShadeD}{gray}{0.86}
\definecolor{auditShadeE}{gray}{0.80}
\definecolor{auditDropA}{HTML}{EDF3FA}
\definecolor{auditDropB}{HTML}{DAE8F5}
\definecolor{auditDropC}{HTML}{BED6EB}
\definecolor{auditDropD}{HTML}{98BDDE}
\definecolor{auditDropE}{HTML}{6EA2CC}
\definecolor{auditRiseA}{HTML}{FEF3E8}
\definecolor{auditRiseB}{HTML}{FCE1C4}
\definecolor{auditRiseC}{HTML}{F6C996}
\definecolor{auditRiseD}{HTML}{EFAF6C}
\definecolor{auditRiseE}{HTML}{E39643}
\usepackage{float}
\usepackage{wrapfig}
\usepackage{etoolbox}
\AtBeginEnvironment{table}{\setlength{\abovecaptionskip}{10pt}\setlength{\belowcaptionskip}{\baselineskip}}
\AtBeginEnvironment{wraptable}{\setlength{\abovecaptionskip}{10pt}\setlength{\belowcaptionskip}{\baselineskip}}
\AtBeginEnvironment{figure}{\setlength{\abovecaptionskip}{10pt}}
\usepackage{placeins} 
\floatstyle{ruled}
\newfloat{algorithm}{tbp}{loa}
\floatname{algorithm}{Algorithm}
\usepackage{hyperref}
\usepackage{xurl}
\hypersetup{colorlinks=true,linkcolor=blue!45!black,citecolor=blue!45!black,urlcolor=blue!45!black,
pdftitle={What Paired Evaluations Reveal under Visual Perturbations},pdfauthor={Yongda Wei, Chen Zhang, Yifei Wang, Xinyu Wang, Bosen Shao, Hanxi Li, Liping Di}}
\newtheorem{theorem}{Theorem}
\newtheorem{lemma}{Lemma}

\newtheorem{corollary}{Corollary}
\theoremstyle{definition}
\newtheorem{definition}{Definition}
\theoremstyle{plain}
\newcommand{\E}{\mathbb{E}}
\newcommand{\Prb}{\mathbb{P}}
\newcommand{\ind}{\mathrm{ind}}

\newcolumntype{Y}{>{\raggedright\arraybackslash}X}
\title{What Paired Evaluations Reveal under Visual Perturbations}
\author{\textbf{Yongda Wei \quad Chen Zhang \quad Yifei Wang \quad Xinyu Wang}\\
\textbf{Bosen Shao \quad Hanxi Li \quad Liping Di}}
\iclrfinalcopy
\makeatletter
\patchcmd{\@maketitle}
  {Published as a conference paper at ICLR 2027}
  {Preprint}
  {}
  {\PackageError{arxiv-preprint}{Failed to set the preprint header}{}}
\makeatother
\newcommand{\enumClassesOrig}{19,200}
\newcommand{\enumNarrowerOrig}{318}
\newcommand{\enumNarrowerOrigPct}{1.7}
\newcommand{\enumClassesReal}{1,834,272}
\newcommand{\enumNarrowerReal}{1,708}
\newcommand{\enumCellsReal}{9,504}

\begin{document}
\raggedbottom
\maketitle
\begin{abstract}
To understand the reliability of computer vision models in real-world applications, robustness evaluation must examine diverse visual perturbations, yet benchmarks cover only some real-world conditions and physical testing is costly. Paired evaluations link clean and perturbed predictions for the same image, retaining changes in correctness, confidence, and acceptance that aggregate accuracy does not describe. We investigate how this image correspondence supports two needs in robustness evaluation: \textbf{interpreting paired evaluation results} and \textbf{prioritizing samples for physical testing}. To interpret paired evaluation results, we fix both sets of prediction records and vary their correspondence within each class. We prove that classwise correct--correct counts give the same sharp bounds on lost acceptance and mean true-class probability decrease among retained-correct inputs as any feasible five-state refinement. Distinguishing persistent from changed wrong answers can further constrain accepted-error transitions, while shared correspondence can establish policy orderings left unresolved by separate cost intervals. To prioritize samples for physical testing, we retain each image's synthetic responses and rank clean-correct images by their mean true-class probability under corruption. Across 44 classifiers, testing the highest-risk 20\% finds 67\% and 45\% of failures under mild screen and print recaptures, versus 58\% and 36\% for clean confidence and 60\% and 37\% for an equal-size natural-transformation average. With both probability averaging and an A3Rank scoring adaptation, the tested corruption set yields higher mean failure recall than the natural-transform set; differences between scores depend on the source and budget. Together, these findings show that the value of correspondence depends on the evaluation objective: classwise counts suffice for specified reliability bounds, while image-specific synthetic responses improve the allocation of physical tests within the evaluated pool.
\end{abstract}

\section{Introduction}
Understanding the reliability of computer vision models in real-world applications requires examining their responses to diverse visual perturbations, from digital corruptions to physical changes in image capture \citep{hendrycks2019,baek2024,baek2025,nugent2025procedural}. A fixed benchmark covers only a subset of these conditions, while collecting physical measurements across additional settings requires resources. This creates two needs: extracting useful reliability information from completed evaluations and deciding which images to prioritize in subsequent tests. Aggregate accuracy describes overall correctness but leaves important image-level changes unresolved. For example, an image may remain correctly classified after perturbation while its confidence falls below an acceptance threshold, changing an accepted prediction into a rejection without changing accuracy. Paired evaluations link clean and perturbed predictions for the same image, making such changes observable.

How does image correspondence in paired evaluations support the judgments needed in robustness evaluation? For completed tests, without their matching, knowing each side's predictions does not necessarily reveal which accepted correct predictions became accepted errors or lost acceptance. The answer depends on how the records correspond and what information about that correspondence has been retained. For subsequent physical tests, associating synthetic responses with their source images may help identify which images are likely to fail under real perturbations. We examine this question in image classification through two subquestions: first, which reliability changes and policy comparisons can be determined from separate prediction records and selected joint counts; second, whether image-specific synthetic responses improve failure discovery at a fixed physical testing budget. The first question concerns what the available records establish about completed tests; the second concerns their empirical usefulness for directing further tests.

The first part presents our analysis for \textbf{interpreting paired evaluation results}. We hold the clean and perturbed prediction records fixed and vary their correspondence within each true class (Section~\ref{sec:retro}). Classwise counts are then added: first the number of images classified correctly in both conditions, followed by the full five-state breakdown of correctness and answer changes. Given the fixed records, correctness overlap provides all five-state endpoint constraints for lost acceptance among retained-correct inputs and their mean true-class probability decrease. Distinguishing persistent from changed wrong answers can further narrow accepted-error transitions. Comparing policies under a common feasible correspondence also establishes some orderings that separate cost intervals leave unresolved. The value of the additional counts depends on the policy cost, including whether it tracks where accepted errors originate or counts all accepted errors. These results specify which retained information constrains each target and how those constraints support policy comparison.

The second part focuses on \textbf{prioritizing samples for physical testing}. We retain each image's synthetic responses and rank clean-correct images by their mean true-class probability under corruption (Section~\ref{sec:prosp}). On mild screen and print recaptures, a 20\% testing budget finds, on average, 67\% and 45\% of failures, versus 60\% and 37\% for an equal-size natural-transformation average. Reassigning synthetic scores to other same-class images reduces screening performance, directly demonstrating the value of matching responses to source images. Crossing both probe sets with probability averaging and adapted A3Rank yields higher mean failure recall for the tested corruption set under both scores, with score preferences varying by source and budget. These comparisons assess image identity and synthetic-probe choice for allocating real tests.

Together, the analyses show that the value of correspondence depends on the judgment an evaluation is intended to support. Classwise correctness overlap already captures all five-state endpoint constraints for two reliability targets given fixed prediction records, whereas image-specific synthetic responses improve real-failure screening within the evaluated pool. More detailed information therefore calls for target-specific assessment: tighter bounds, additional policy orderings and improved predictive performance are distinct outcomes. Section~\ref{sec:prelim} defines the records, response states and information levels; Section~\ref{sec:retro} develops the identification and policy results; Section~\ref{sec:prosp} evaluates failure screening; and Section~\ref{sec:boundaries} examines the different outcome of aggregating responses for model-level prediction.

\section{Related Work}
\paragraph{Robustness and performance prediction.}
Corruption benchmarks, replicated test sets and deployment datasets measure different distribution changes \citep{hendrycks2019,recht2019,koh2021}. Synthetic robustness need not transfer uniformly to natural or perceptually different shifts \citep{hendrycks2021faces,mintun2021,kar2022}. ImageNet-ES and ES Diverse provide controlled physical capture variation \citep{baek2024,baek2025}. Effective robustness and accuracy-on-the-line motivate strong clean-performance baselines \citep{taori2020,miller2021}; pretraining and reference distributions affect these relations \citep{djolonga2021,shi2023}.

\paragraph{Paired transitions and confidence.}
Matched outcomes underpin correlated-proportion tests \citep{mcnemar1947}, negative-flip analysis \citep{yan2021,sawada2026compound} and error-consistency studies \citep{geirhos2020,kim2025,klein2025}. Bench-C examines confidence and correctness changes under corruption \citep{sui2026}, while VLM-RobustBench analyzes harmful and helpful answer flips \citep{saxena2026}. Paired evaluations also assess reward-model robustness for language-model alignment by comparing preference rankings or confidence before and after input transformations \citep{wu2025rewordbench,zang2026reward}. We study what joint counts determine beyond separate records, without claiming novelty for the partition. Calibration, failure detection and selective classification measure different aspects of confidence \citep{guo2017,minderer2021,hendrycks2017baseline,corbiere2019,geifman2017,geifman2019}. Their behavior can change under shift or without an accuracy change \citep{ovadia2019,galil2021,liang2024}, requiring explicit protocols and baselines \citep{jaeger2023,cattelan2024,traub2024}.

\paragraph{Test prioritization and augmentation-based confidence.}
PRIMA uses input and model mutations with learning to rank to prioritize test inputs \citep{wang2021prima}; A3Rank uses augmentation alignment to prioritize overconfident failures \citep{wei2024a3rank}. Simple uncertainty scores can be competitive test-prioritization baselines \citep{weiss2022simple}. Averaging predicted-class probabilities over test-time transformations is an existing training-free confidence score \citep{bahat2020tta}. We use corruption responses of known clean images to prioritize subsequent physical recaptures. At equal inference counts, we compare corruption and natural-transformation averages and adapt the A3Rank scoring rule to both probe sets; its original rejection component and current-input error target are distinct from our future-recapture task.

\paragraph{Identification and scoring.}
Joint distributions constrained by their marginals are a classical problem \citep{strassen1965,tchen1980,kellerer1984}. Stratified intersection bounds also quantify treatment harm \citep{kallus2022harm}, while partially identified counterfactual utilities support policy comparison \citep{benmichael2024asymmetric}. We use elementary intersection counts and extremal subset sums; our structural question is which correctness and wrong-answer constraints alter these endpoints. Finite-record identification differs from sampling inference for partially identified parameters \citep{imbens2004}. Proper scores and CORP quantify probability quality \citep{gneiting2007,dimitriadis2021corp}; totals need no correspondence, unlike allocation to response states. Alternative ImageNet annotations delimit what correctness means \citep{beyer2020real,vasudevan2022imagenetmistakes}.

\section{Records, Targets, and Missing Correspondence}\label{sec:prelim}
\subsection{Why Counts Need Not Determine a Joint Change}
\begingroup
\setlength{\intextsep}{0pt}
\setlength{\columnsep}{12pt}
\begin{wraptable}{r}{0.50\linewidth}
\setlength{\abovecaptionskip}{0pt}
\caption{Five disjoint response states and aggregate metrics under the original single-label definition. The state symbols also denote empirical proportions.}
\label{tab:response-states}
\centering

\setlength{\tabcolsep}{2pt}
\renewcommand{\arraystretch}{1.18}

\fontsize{8}{9.5}\selectfont
\begin{tabularx}{\linewidth}{@{}l Y c c c@{}}
\toprule
State & Response & $C_i$ & $C'_i$ & Labels\\
\midrule
$h$ & Harmful flip & $1$ & $0$ & $\hat y_i\ne\hat y'_i$\\
$r$ & Correction & $0$ & $1$ & $\hat y_i\ne\hat y'_i$\\
$s$ & Persistent error & $0$ & $0$ & $\hat y_i=\hat y'_i$\\
$w$ & Changed error & $0$ & $0$ & $\hat y_i\ne\hat y'_i$\\
$k$ & Retained correctness & $1$ & $1$ & $\hat y_i=\hat y'_i$\\
\midrule
\end{tabularx}
\begin{tabularx}{\linewidth}{@{}*{2}{>{\centering\arraybackslash}X}@{}}
Clean accuracy & Perturbed accuracy\\[-1pt]
$a_0=h+k$ & $a_1=r+k$\\[2pt]
Accuracy decrease & Label-change rate\\[-1pt]
$a_0-a_1=h-r$ & $F=h+r+w$\\
\bottomrule
\end{tabularx}
\end{wraptable}
Consider two images of the same true class $y$ with clean probability vectors $(p_y,p_{\tilde y})=(0.9,0.1),(0.6,0.4)$ and perturbed vectors $(0.6,0.4),(0.2,0.8)$ for alternative class $\tilde y$. At threshold 0.8, only the first clean prediction and the wrong perturbed prediction are accepted. Both correspondences give one retained-correct pair and one harmful flip, with identical classwise five-state counts. Yet matching the accepted clean prediction to the correct perturbed one loses acceptance while retaining correctness; the other matching produces an accepted error. The counts therefore need not determine which accepted prediction failed. We ask what remains possible as increasingly detailed joint counts supplement the report.\label{prop:nonid}

\paragraph{Records and response states.}\label{def:records}\label{def:states}
For $N$ paired inputs with shared original labels $y_i$, unprimed and primed quantities refer to clean and perturbed inputs. Each record carries a predicted label $\hat y_i$ or $\hat y'_i$ under a fixed tie rule, maximum probability $S_i$ or $S'_i$, and true-class probability $M_i$ or $M'_i$. Correctness is $C_i=\mathbf1\{\hat y_i=y_i\}$. To examine whether predictions remain accepted after perturbation, we accept a prediction when its maximum probability reaches a fixed threshold $\tau$ and reject it otherwise. Thus $Z_i=\mathbf1[S_i\geq\tau]$ indicates acceptance; primed counterparts use the same definitions and threshold.

Table~\ref{tab:response-states} defines the five disjoint states $T_i\in\{h,r,s,w,k\}$ and their relationships to clean accuracy $a_0$, perturbed accuracy $a_1$, and the predicted-label change rate $F$. The same letters denote empirical proportions and, in events, $j$ means $T=j$. Thus $s,w,k$ contribute zero directly to the accuracy difference, although their probabilities may change. Given $a_0,a_1,F$, the state proportions have at most one remaining degree of freedom.
\par
\endgroup

\begin{table}[!t]
\centering
\caption{Reliability targets. The first two are empirical conditional means, undefined for empty states and positive for decreases. The last three are empirical event proportions with all $N$ evaluated inputs as denominator.}
\label{tab:reliability-targets}

\setlength{\tabcolsep}{4pt}
\renewcommand{\arraystretch}{1.15}

\fontsize{9}{10.5}\selectfont
\begin{tabularx}{\linewidth}{@{}>{\raggedright\arraybackslash}p{0.09\linewidth} >{\raggedright\arraybackslash}p{0.37\linewidth} Y@{}}
\toprule
Target & Definition & Interpretation\\
\midrule
$dM_k$ & $\E[M-M'\mid T=k]$ & Mean true-class probability decrease among inputs correct on both sides.\\[3pt]
$ds_j$ & $\E[S-S'\mid T=j],\ j\in\{s,w\}$ & Mean decrease in maximum probability within the specified wrong-answer state.\\
\midrule
$E_\tau$ & $\Prb(C'=0,Z'=1)$ & Fraction of inputs with an accepted error after perturbation.\\[3pt]
$L=L_k$ & $\Prb(T=k,Z=1,Z'=0)$ & Fraction remaining correct but changing from accepted to rejected.\\[3pt]
$H$ & $\Prb(C=1,Z=1,C'=0,Z'=1)$ & Fraction changing from an accepted correct prediction to an accepted error.\\
\bottomrule
\end{tabularx}
\end{table}
\paragraph{Reliability targets.}\label{def:targets}
Table~\ref{tab:reliability-targets} distinguishes empirical conditional score changes ($dM_k,ds_j$) from acceptance-event proportions ($E_\tau,L=L_k,H$). In the opening example, $(H,L,dM_k)$ is either $(0,1/2,0.3)$ or $(1/2,0,0)$, which means the counts need not point-identify these targets. $E_\tau$, in contrast, needs only perturbed records.

\emph{Paired evaluation} denotes a matched-input design, not pairwise comparison of two models. A statistic may use both sides without depending on their correspondence, as a difference of marginal means does. Here correspondence dependence means that a statistic can change when the same two sets of records are reassigned. Besides the targets above, we use $A=\E|S-S'|$ and $Q=\operatorname{Corr}(S,S')$ as simple paired statistics.

\subsection{Information Available without the Actual Pairing}
\begin{definition}[Correspondence and information levels]\label{def:levels}
Fix one model and condition, with $n_y$ clean and $n_y$ perturbed records in each true class $y$. Within each class, match every clean record to exactly one perturbed record, using each perturbed record once. Write $\pi_y(i)=j$ for this match and $\pi=(\pi_y)_y$ for all classwise matchings. Changing the matching changes only which records are linked; their labels, scores and acceptance decisions remain fixed.

The information levels $\mathcal I_0$, $\mathcal I_k$ and $\mathcal I_5$ retain, respectively, all prediction records on both sides, including true class labels, without paired counts; the same records plus the observed correct--correct counts $k_y$ for each class; and the same records plus all five observed state counts $h_y,r_y,s_y,w_y,k_y$ for each class.
Whereas $\pi$ is one matching, $\Pi(\mathcal I)$ contains all within-class matchings consistent with the retained information. Because $\Pi(\mathcal I_5)\subseteq\Pi(\mathcal I_k)\subseteq\Pi(\mathcal I_0)$, retaining more counts can rule out possible matchings, while the observed matching remains feasible at every level.
\end{definition}
For clean and perturbed correct counts $a_y,b_y$, overlap $k_y$ is feasible exactly when $\max(0,a_y+b_y-n_y)\leq k_y\leq\min(a_y,b_y)$. Let $q(\pi)$ denote a target's value under matching $\pi$. When defined throughout the feasible set, its \emph{identified interval} spans the smallest and largest allowed values:
\begin{equation}
[\underline q(\mathcal I),\overline q(\mathcal I)]
=\left[\min_{\pi\in\Pi(\mathcal I)}q(\pi),\max_{\pi\in\Pi(\mathcal I)}q(\pi)\right].\label{def:identified}
\end{equation}
Endpoints attained by feasible correspondences are sharp, and the width measures ambiguity for fixed records, not sampling uncertainty. Classwise choices form a Cartesian product, allowing extrema of additive counts to be summed without assuming statistical independence between classes.

For $H$ and $L$, marginal reference intervals follow from classwise intersection bounds. Class-conditional independent matching gives reference values. Appendix~\ref{app:d2} provides the formulas and detailed sharpness proof, and Appendix~\ref{app:correspondence-value} also examines accuracy-only reports.

\subsection{Models, Sources, and Protocol}\label{sec:design}
There are 44 models in ten fixed statistical groups, covering image--text, self-supervised and supervised models (Appendix~\ref{app:models}). Syn is a synthetic-corruption evaluation set that applies nine corruption types at severities 2 and 4 to 10,000 ImageNet images with 1,000 candidate classes. Syn-B applies these same 18 conditions to the 1,000 original images underlying ES \citep{baek2024} and Diverse \citep{baek2025}, spanning 200 classes with five images per class. ES provides 54 screen-recapture conditions and Diverse provides 162 print-recapture conditions. This links synthetic and physical responses for the same images, enabling sample prioritization. 

Images are split into development and evaluation sets within each class. Identification analyses use evaluation records, while acceptance thresholds are set using development images. Real capture conditions are grouped by mean development accuracy across models into mild (at least 40\%), severe (20\% to below 40\%), and extreme (below 20\%) tiers. Appendix~\ref{app:protocol} provides further details on data construction and evaluation protocols.

\section{Retrospective Diagnosis: Information Requirements and Policy Costs}\label{sec:retro}
With both sets of prediction records fixed, we examine which correspondence information constrains reliability changes (Section~\ref{sec:overlap-targets}), supports policy comparisons under different costs (Section~\ref{sec:joint-policy}), and reduces ambiguity in the observed records (Section~\ref{sec:retro-empirical}). The empirical analysis distinguishes certified policy orderings from realized selection outcomes.

\begin{table}[!b]
\centering\setlength{\tabcolsep}{4pt}
\caption{Information and supported conclusions for fixed single-label records and within-class reassignment. Both sides' records remain available; $dM_k$ requires positive retained-correct support.}
\label{tab:information-guide}

\small
\begin{tabularx}{\linewidth}{@{}l >{\raggedright\arraybackslash}p{0.25\linewidth} Y@{}}\toprule
Target & Information used & Supported conclusion \\\midrule
$E_\tau$ & Perturbed records & Exact accepted-error mass. \\
$L,\ dM_k$ & Classwise overlaps $\mathcal I_k$ & Sharp bounds equal those under every feasible five-state refinement. \\
$H$ & Five-state counts $\mathcal I_5$ & Can tighten overlap-conditioned bounds. \\
$J_j-J_\ell$ & Shared correspondence across policies & Can certify an ordering throughout the feasible set despite overlapping cost intervals. \\\bottomrule
\end{tabularx}
\end{table}

\subsection{Correctness Overlap Suffices for Two Targets}\label{sec:overlap-targets}
\begin{theorem}[Endpoint sufficiency of correctness overlap]\label{prop:sufficiency}
For fixed single-label prediction records and feasible classwise overlaps $k_y$, the identified intervals for $L$ and, when $K=\sum_y k_y>0$, $dM_k$ are identical under $\mathcal I_k$ and every feasible classwise five-state refinement $\mathcal I_5$.
\end{theorem}
\noindent\emph{Proof idea.} Fix a correspondence with the requested five-state counts. Independently permuting clean and perturbed correct identities among their correct positions preserves the wrong--wrong matching and every state count. These permutations select any $k_y$ correct records on each side and any bijection between them. Subset intersections and extreme score sums give the sharp $L$ and $dM_k$ bounds, respectively (Equations~\eqref{eq:Lk-bound} and~\eqref{eq:dmk-bound}, Appendix~\ref{app:d2}). Overlaps suffice for the endpoints, but observed values may still vary within them (Table~\ref{tab:information-guide}).

\paragraph{Why accepted-error transitions differ.}\label{prop:H}
For $H$, changing the wrong perturbed endpoint can change the persistent/changed-wrong split. Five-state refinement can therefore strictly narrow its overlap-conditioned interval. Appendix~\ref{app:d2} gives the bounds and a three-record counterexample.

\subsection{Shared Correspondence and the Choice of Policy Cost}\label{sec:joint-policy}
We compare fixed rejection policies under partially identified outcomes \citep{benmichael2024asymmetric}. Keep the clean threshold fixed and let policy $j$ use perturbed threshold $\tau_j$. Substituting $Z'_j=\mathbf1[S'\geq\tau_j]$ defines $H_j,L_j$ and transition cost $J_j=\lambda H_j+(1-\lambda)L_j$, $\lambda\in[0,1]$. This cost separately penalizes accepted errors and lost correct acceptances originating from initially correct accepted predictions. Within each class, these events involve disjoint outcomes of the same records. We consider a cost counting all accepted errors below.

Given their sharp classwise count bounds, constraining their sum to at most the number $u_y$ of correct accepted clean records determines the sharp single-policy upper cost bound (Lemma~\ref{lem:cap}, Appendix~\ref{app:d2}). No additional joint enumeration is needed once those bounds are known; computing five-state $H$ bounds can still require enumeration, and the lower cost bound need not separate this way.

\paragraph{Comparisons use one common correspondence.}\label{prop:certificate}
For policies $j,\ell$, the elementary inequality
\begin{equation}
\overline\Delta_{j\ell}(\mathcal I)=\max_{\pi\in\Pi(\mathcal I)}[J_j(\pi)-J_\ell(\pi)]\leq\overline J_j(\mathcal I)-\underline J_\ell(\mathcal I)\label{eq:joint-policy-difference}
\end{equation}
can be strict because separate endpoints may require different correspondences. A negative left side establishes that $j$ costs less throughout the feasible set, even if the separate intervals overlap. Appendix~\ref{app:joint-policy} reports these additional comparisons. Their information requirements depend on what the cost counts:

\begin{corollary}[A cost based on total accepted errors]\label{cor:natural}
Let $J'_j(\pi)=\lambda E_{\tau_j}+(1-\lambda)L_j(\pi)$, where $E_{\tau_j}$ is the perturbed-side accepted-error mass. The identified intervals of $J'_j$ and the upper difference bounds $\overline\Delta'_{j\ell}$ are identical under $\mathcal I_k$ and $\mathcal I_5$.
\end{corollary}
Here $E_{\tau_j}$ is marginal, and correct-record symmetry preserves the attainable set of the entire vector $(L_j)_j$ across both levels. Thus the value of wrong-answer refinement depends on whether the cost tracks errors through $H$ or counts all accepted errors through $E_\tau$ (proof in Appendix~\ref{app:d2}).

Table~\ref{tab:information-guide} summarizes the information requirements and the conclusions they support for the targets considered in this section.

\subsection{Findings: How Much Ambiguity Remains in the Observed Records?}\label{sec:retro-empirical}
Acceptance uses the same $\tau$ on both sides, targeting 80\% clean-development coverage and accepting cutoff ties. Policy comparisons fix this clean threshold; five perturbed-side thresholds target 60--100\% clean-development coverage. Analytic bounds cover 11,088 model--condition cells; additional bound computations for five-state $H$ and positive-overlap marginal $dM_k$ cover a frozen, nonrandom 450-cell panel spanning all 44 models and four sources (Appendices~\ref{app:empirical} and~\ref{app:correspondence-value}).

\begin{figure}[!b]
\centering\includegraphics[width=\linewidth]{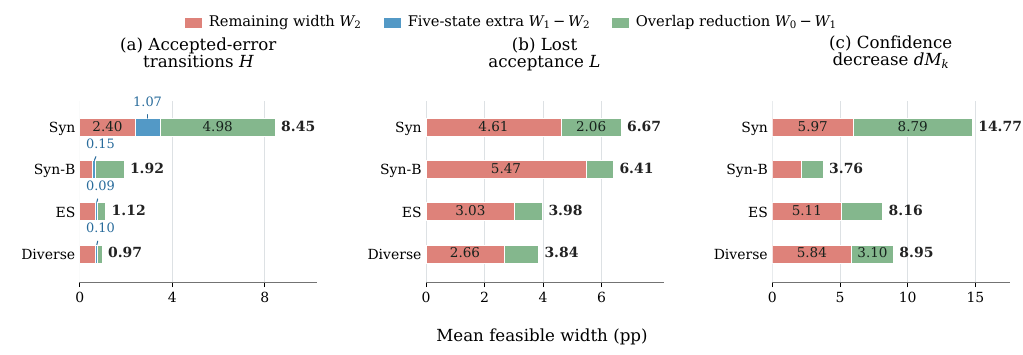}
\caption{Mean identified widths in the 450-cell panel measure correspondence ambiguity. $W_0$, $W_1$ and $W_2$ use marginal records, added classwise overlap and added five-state counts. Colors partition $W_0$ into remaining width and two reductions; bold labels give $W_0$. Five-state refinement narrows $H$, but not $L$ or $dM_k$. Marginal $dM_k$ ranges over positive overlaps; later levels fix observed overlaps.}
\label{fig:d2}
\end{figure}

Overlap reduces full-grid mean $L$ widths by 14--31\% across sources. Five-state refinement further narrows $H$ in 195/450 cells, but not $L$ or $dM_k$, often leaving values undetermined (Figure~\ref{fig:d2}). Appendix~\ref{app:empirical} gives full-grid results and matched-support checks for cross-source comparisons.

Table~\ref{tab:policy-overview} distinguishes certified orderings from realized selection outcomes. Its first three rows count cells with an additional policy pair for which one policy has lower $J$ under every feasible correspondence. The last row compares observed costs of minimax-regret and feasible-mean selection; lower, higher and equal refer to the former relative to the latter. Under $J'$, five-state refinement adds no comparisons beyond overlap (Corollary~\ref{cor:natural}). Appendix~\ref{app:joint-policy} gives source-level results, the weight grid and a worked print-recapture example.

\begin{table}[!tb]
\centering
\caption{Policy comparisons and selection on the frozen 450-cell panel, $\lambda=0.5$. Ordering rows count cells with additional strict comparisons; these counts can overlap. Counted orderings require both correct and incorrect predictions among threshold-switching records. Feasible-mean selection minimizes mean cost over uniformly weighted feasible correspondences.}
\label{tab:policy-overview}
\setlength{\tabcolsep}{4pt}

\small\fontsize{8.5}{10}\selectfont
\begin{tabularx}{\linewidth}{@{}>{\hsize=.9\hsize\linewidth=\hsize}Y >{\hsize=1.1\hsize\linewidth=\hsize}Y l@{}}
\toprule
Question & Comparison & Result \\\midrule
Shared correspondence: more orderings? & Shared correspondence vs. separate intervals ($\mathcal I_5$) & 312/450 (69.3\%) \\
Correctness overlap: more orderings? & $\mathcal I_0\rightarrow\mathcal I_k$ & 241/450 (53.6\%) \\
Five-state counts: more orderings? & $\mathcal I_k\rightarrow\mathcal I_5$ & 42/450 (9.3\%) \\\midrule
Minimax regret: lower realized cost? & Minimax regret vs. feasible mean ($\mathcal I_5$) & \shortstack[l]{4 lower / 3 higher\\443 equal} \\\bottomrule
\end{tabularx}
\end{table}

\begin{figure}[!b]
\centering\includegraphics[width=\linewidth]{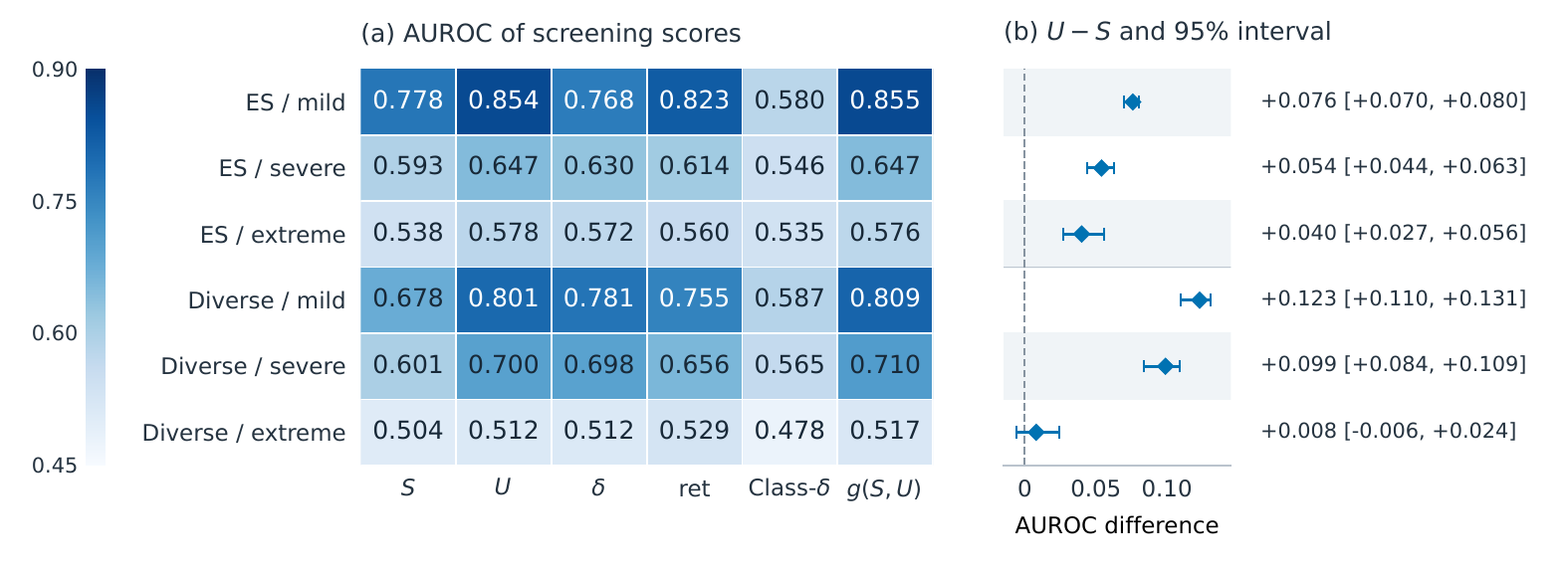}
\caption{Screening scores are evaluated by how well they distinguish images that fail under real recapture from those that remain correct. All scores use identical clean-correct image sets. (a) AUROC on a common color scale. Class-$\delta$ is the development-set class mean of $M-U$; $g(S,U)$ is fitted on development images. (b) Paired AUROC differences $U-S$ with 95\% model-group resampling intervals; the dashed line marks zero gain. Numerical results, class-resampled intervals and cell counts are in Appendix~\ref{app:screening}.}
\label{fig:screening-auroc}
\end{figure}

\section{Prospective Screening: Which Images Deserve Real Recapture?}\label{sec:prosp}
After examining what correspondence establishes about completed evaluations, we test whether image-specific synthetic responses help prioritize physical tests. We define scores and test image identity (Section~\ref{sec:incrementmethod}), measure failure discovery at fixed budgets (Section~\ref{sec:screening-budget}), compare probe sets and scores at equal inference counts (Section~\ref{sec:tta}), and assess scope and inference cost (Section~\ref{sec:screening-scope}).

\subsection{Scores, Image Sets, and Baselines}\label{sec:incrementmethod}

For each model's clean-correct images among the 500 real evaluation base images, define $U_i$ as mean true-class probability over the 18 Syn-B conditions, $\mathrm{ret}_i$ as the fraction remaining correct, and $\delta_i=S_i-U_i$. Here true and original predicted labels coincide, so $U$ uses the predicted-class averaging construction of \citet{bahat2020tta} on corruptions excluding the original image. We rank subsequent recapture failures ($h$ rather than $k$) by $-U$, $-S$, $\delta$ and $-\mathrm{ret}$ on identical image sets.

Controls include development-class means and a frozen logistic predictor $g(S,U)$ fitted on development images. Intervals separately resample ten model groups or 200 true classes to assess within-pool sensitivity. Appendix~\ref{app:screening} details fitting, identity ablation and matched class baselines.

Figure~\ref{fig:screening-auroc} compares six scores across sources and tiers and the AUROC gain of $U$ over $S$. On mild ES/Diverse, $U$ reaches 0.854/0.801, above clean confidence, retention and the class-mean drop baseline; $g(S,U)$ differs by at most 0.01. The drop $\delta$ ranks worse by 0.086/0.020, with exceptions at other budgets or splits. On extreme Diverse, $U$ reaches only 0.512 with an inconclusive gain over $S$.

Image identity contributes to screening: assigning $U$ to another clean-correct image of the same class reduces mild AUROC from about 0.85/0.80 to 0.55/0.58 on the common non-singleton subset. A development-class-mean $U$ predictor reaches about 0.60 on both sources on its supported subset. These prediction controls complement the fixed-record identification analysis (protocols in Appendix~\ref{app:screening}).

\subsection{Failure Coverage at a Fixed Testing Budget}\label{sec:screening-budget}
For a budget $q$, select the highest-risk fraction of clean-correct images and measure the share of all real failures found, with cutoff ties prorated. This \emph{failure coverage} is recall, not the error rate within the selected images. At the primary $q=0.2$, $U$ covers 67\% and 45\% of mild ES and Diverse failures, against 58\% and 36\% for clean confidence (Figure~\ref{fig:screening-budget}). The $U-S$ differences exclude zero under both resampling schemes at all tested budgets $q\in\{0.1,0.2,0.3\}$ on mild conditions. On Diverse, coverage is nearly tied with $\delta$ at $q=0.2$, and both $\delta$ and retention are better at $q=0.1$ (Table~\ref{tab:screening} from Appendix~\ref{app:screening}).

\begin{figure}[!t]
\centering\includegraphics[width=0.88\linewidth]{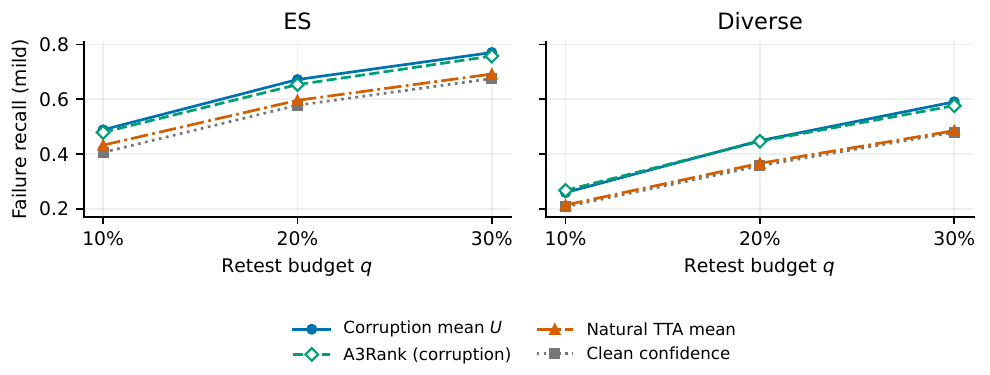}
\caption{Different rankings are evaluated by the fraction of real failures they uncover at a given recapture budget. Results cover three budgets on mild conditions. Curves average the same model--condition cells; $A^3$ uses the same corruptions as $U$. Intervals and other tiers are in Appendix~\ref{app:screening}.}
\label{fig:screening-budget}
\end{figure}

On extreme conditions, high failure prevalence leaves little room to improve coverage despite near-chance ranking. Appendix~\ref{app:screening} quantifies this prevalence-imposed ceiling.

\subsection{Probe Sets and Scoring Rules at Equal Inference Counts}\label{sec:tta}
Following \citet{bahat2020tta}, $\mathrm{TTA}_{18}$ averages true-class probabilities over 18 geometric and gamma transformations without label-based selection, matching $U_{18}$'s inference count. The protocol fixed mild-condition coverage at $q=0.2$ as primary and 0.03 as the practical-effect threshold. Figure~\ref{fig:probe-score} crosses the probe sets with probability averaging and an adapted A3Rank alignment score (Appendix~\ref{app:alignment-baseline}).

\begin{figure}[!tb]
\centering\includegraphics[width=\linewidth]{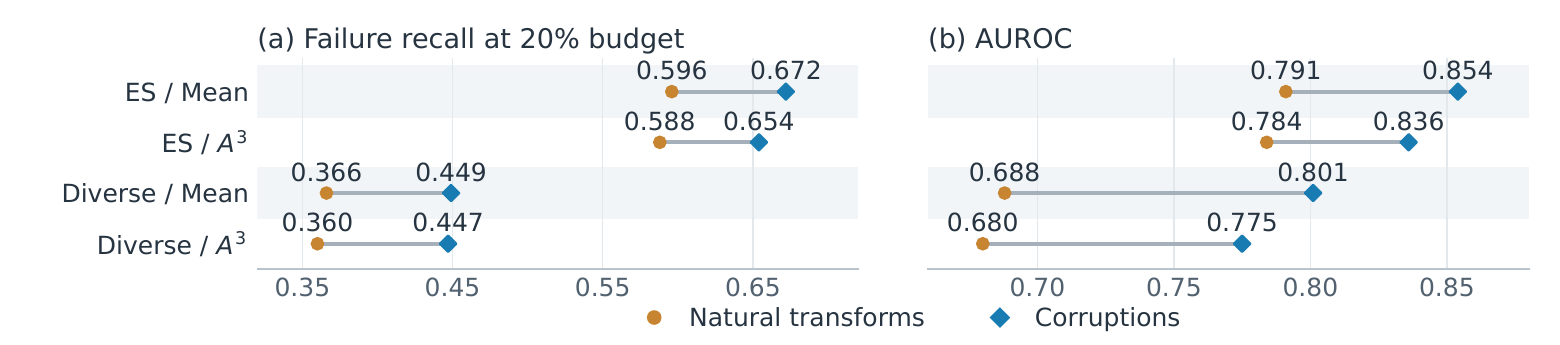}
\caption{At equal inference counts, the tested corruptions yield higher mean failure recall and AUROC than natural transforms under both scoring rules. Comparisons use identical mild-condition image sets and 18 transformed inputs each. Mean denotes probability averaging; $A^3$, adapted A3Rank. Lines connect point estimates; selected difference intervals appear in Appendices~\ref{app:alignment-baseline} and~\ref{app:probe-diagnostics}.}
\label{fig:probe-score}
\end{figure}

The corruption average exceeds the natural average by 0.076 coverage on ES and 0.083 on Diverse, with both model-group and class-resampled interval lower endpoints above 0.03; the AUROC gaps are 0.063 and 0.114. Under either scoring rule, the tested corruption set yields higher mean recall and AUROC than the natural set. Score preferences depend on source, metric and budget: on Diverse, $U$ and the corruption-based $A^3$ have an inconclusive recall difference at 20\%, and $A^3$ leads by 0.008 at 10\% (Figure~\ref{fig:screening-budget}), while $U$ has higher AUROC.

Mild gains persist against additional confidence baselines; checks of averaging, transform selection and model scoring retain the advantage over TTA (Appendices~\ref{app:simple-baselines} and~\ref{app:screening}); Appendix~\ref{app:probe-diagnostics} examines individual probes and capture conditions. Because the sets differ in transformation family and response strength, these results establish a tested-set advantage rather than isolate corruption type.

\subsection{Scope and Additional Inference Cost}\label{sec:screening-scope}
Screening was initially explored on evaluation images. Swapping split roles and defining tiers on the opposite half preserve mild $U-S$ gaps within the same image, model and capture pool. Recapture budgets exclude the additional synthetic inference. Appendix~\ref{app:screening} reports split and scoring checks, including $\delta$, and gains from prefixed probe subsets (Table~\ref{tab:cost-curve} from Appendix~\ref{app:screening}); these characterize the tested configurations, not arbitrary smaller sets.

\section{Boundaries of Other Uses}\label{sec:boundaries}
\paragraph{Do paired summaries improve model-level prediction?}
Model-level summaries of synthetic responses are used to predict performance under mild real perturbations. Comparisons with unpaired baselines across held-out model groups test whether retaining image correspondence adds predictive value. The gains depend on the prediction target and baseline. This analysis separates the usefulness of correspondence for predicting model-level performance from its usefulness for prioritizing individual images. Appendix~\ref{app:new-p8} reports the detailed comparisons and group-removal refits.

\section{Scope and Conclusion}\label{sec:limits}
The bounds assume fixed labeled records, unrestricted within-class reassignment and fixed record-level decisions. ReaL checks follow this convention; human recapture-label review remains pending. Resampling assesses sensitivity within the shared image, model and capture pool. Policy conclusions depend on thresholds and costs, with certified comparisons but mixed realized-selection outcomes.

Image correspondence supports \textbf{interpreting paired evaluation results} and \textbf{prioritizing samples for physical testing}. Correctness overlap gives the same sharp $L$ and $dM_k$ bounds as every feasible five-state refinement; distinguishing persistent from changed wrong answers can tighten $H$. Shared correspondence certifies policy orderings unresolved by separate cost intervals.

On mild screen and print conditions, image-specific corruption responses improve mean failure recall over clean confidence and the tested equal-size natural-transformation average at a fixed recapture budget. Identity ablations establish the contribution of image correspondence. Both probability averaging and adapted A3Rank favor the tested corruption set, with score preferences varying by source and budget. These findings guide information retention, reliability judgments and physical testing within the evaluated pool.
\label{main-end}
\clearpage
\section*{Statement on AI Assistance}
Large language models assisted with research ideation and experimental implementation, literature retrieval and reference organization and polishing manuscript sections, and developing and checking mathematical proofs. They also supported analysis code and typesetting. No generative models were used to synthesize the experimental images. Verification included separate computational implementations and agent-assisted reviews; these checks do not constitute independent human replication. The authors take responsibility for the final text, mathematical claims, code, and reported results.

\section*{Reproducibility Statement}
We provide supplementary code and documentation to support reproduction of the experiments and analyses. The supplementary materials specify the fixed image identities, development/evaluation splits, perturbation conditions, scoring protocol, and the 44 evaluated models and their statistical groupings, together with the 450-cell panel described in Section~\ref{sec:retro}. Detailed inference settings, split definitions, and threshold conventions are provided in Appendix~\ref{app:protocol}, while the accompanying documentation provides instructions for obtaining the data and running the experiments. Model checkpoints and the datasets are obtained from their original sources; access to some weights (e.g., DINOv3) might require the official access procedure. The clean images are identical to those used in the study, and predictions on the physical recaptures reproduce the recorded results. For synthetic perturbations, the original per-image corruption seeds were not retained. Regeneration therefore follows the documented experimental protocol without guaranteeing pixel-level identity with the original perturbed images.

\begingroup
\interlinepenalty=10000
\bibliography{references}
\bibliographystyle{iclr2027_conference}
\endgroup
\clearpage
\appendix
\makeatletter
\setlength{\@fptop}{0pt}
\setlength{\@fpsep}{18pt plus 2pt minus 2pt}
\setlength{\@fpbot}{0pt plus 1fil}
\makeatother
\raggedbottom
The appendices support the two main questions separately. Appendices~\ref{app:d2}--\ref{app:joint-policy} give proofs, numerical checks and full policy comparisons. Appendix~\ref{app:screening} gives screening protocols and controls, followed by inference details and the model inventory. The final appendices present model-level prediction comparisons and annotation-sensitivity analyses.

\section{Proofs}\label{app:d2}
This appendix proves the endpoint and policy-cost results of Section~\ref{sec:retro} under record-attached correctness and within-class reassignment.

\subsection{Feasible records and count constraints}
Fix one true label $y$, $n_y$ clean records and $n_y$ perturbed records. Correctness means that the single predicted label equals $y$. Probabilities, predicted labels and acceptance decisions are attached to records; they are not recomputed after reassignment. A feasible correspondence is any bijection within the true class. No geometric, temporal or physical compatibility constraints are imposed.

Let $a_y,b_y$ be the two correct counts. A correct--correct overlap $k_y$ is feasible exactly when
\begin{equation}
\max(0,a_y+b_y-n_y)\leq k_y\leq\min(a_y,b_y).
\end{equation}
The correctness-block sizes are then $k_y$, $h_y=a_y-k_y$, $r_y=b_y-k_y$, and $s_y+w_y=n_y-a_y-b_y+k_y$. A proposed $s_y/w_y$ split must also be realizable given the wrong predicted labels; nonnegativity and the displayed sum alone are insufficient. We condition only on nonempty feasible sets, including the observed counts.

\subsection{Marginal intersection bounds}
\label{lem:frechet}
For $q\in\{H,L\}$, let $u_y$ count correct accepted clean records and $v_{qy}$ count wrong accepted perturbed records for $H$, or correct rejected ones for $L$. Classical intersection bounds give
\begin{equation}
[\underline q_0,\overline q_0]=\frac1N\left[\sum_y\max(0,u_y+v_{qy}-n_y),\ \sum_y\min(u_y,v_{qy})\right],\qquad
q_{\ind}=\frac1N\sum_y\frac{u_yv_{qy}}{n_y},\label{eq:frechet}
\end{equation}
where $N=\sum_y n_y$ and $q_{\ind}$ uses class-conditional independent matching. The construction below proves sharpness; Appendix~\ref{app:correspondence-value} considers reports retaining only per-class accuracies.

For two record subsets of sizes $u_y,v_y$, pairing their members together as much as possible attains intersection count $\min(u_y,v_y)$. Pairing them outside one another as much as possible attains $\max(0,u_y+v_y-n_y)$. Arbitrary bijections complete the unused positions. Choices in distinct true classes form a Cartesian product, so classwise extrema can be summed and divided by $N=\sum_y n_y$. This gives Equation~\eqref{eq:frechet}. The product assignment with edge mass $1/n_y$ gives the independent value, which is also the expected count under a uniform permutation. These are finite-record identification statements, distinct from sampling inference for partially identified population parameters \citep{imbens2004}.

\subsection{A symmetry of the correct records}
Start from any bijection satisfying the specified classwise five-state counts. Permute clean correct identities among the $a_y$ correct row positions and, independently, perturbed correct identities among the $b_y$ correct column positions. Keeping all incorrect identities fixed leaves the entire wrong--wrong matching unchanged and therefore preserves every five-state count, even if the constraint retained every individual wrong--wrong predicted-label pair.

There are $k_y$ correct row positions and $k_y$ correct column positions in the correct--correct block. These permutations can select any $k_y$ correct clean records, any $k_y$ correct perturbed records, and any bijection between the selected subsets. This establishes the sufficiency of classwise correctness overlap for the endpoints below. The statement requires both row and column permutations.

\subsection{Sharp lost-acceptance bounds}
Write $v_y=v_{Ly}$. Under $\mathcal I_k$, the sharp lost-acceptance bounds are
\begin{equation}
[\underline L_k,\overline L_k]=\frac1N\left[\sum_y\max\{0,k_y-(a_y-u_y)-(b_y-v_y)\},\ \sum_y\min\{u_y,v_y,k_y\}\right].\label{eq:Lk-bound}
\end{equation}
Let $u_y$ count correct accepted clean records and $v_y$ correct rejected perturbed records. Among the $k_y$ retained-correctness pairs, the number $x$ with an accepted clean endpoint can range over
\[
\max\{0,k_y-(a_y-u_y)\}\leq x\leq\min\{u_y,k_y\}.
\]
The number $z$ with a rejected perturbed endpoint ranges analogously from $\max\{0,k_y-(b_y-v_y)\}$ to $\min\{v_y,k_y\}$. For fixed $x,z$, their overlap in the $k_y$ positions ranges from $\max\{0,x+z-k_y\}$ to $\min\{x,z\}$. Selecting the smallest or largest feasible $x,z$ yields Equation~\eqref{eq:Lk-bound}. The preceding symmetry makes both extrema attainable under every additional feasible $s_y/w_y$ split. Classwise count extrema are normalized by the total number of inputs $N$. No assertion is made that extrema of different targets are simultaneously attainable.

\subsection{Sharp conditional-confidence bounds}
Order the true-class probabilities on each correct side as $x_{y(1)}\leq\cdots\leq x_{y(a_y)}$ and $z_{y(1)}\leq\cdots\leq z_{y(b_y)}$. For $K=\sum_y k_y>0$, the sharp interval is
\begin{equation}
[\underline d_k,\overline d_k]=\frac1K\left[\sum_y\left(\sum_{i=1}^{k_y}x_{y(i)}-\sum_{i=b_y-k_y+1}^{b_y}z_{y(i)}\right),\ \sum_y\left(\sum_{i=a_y-k_y+1}^{a_y}x_{y(i)}-\sum_{i=1}^{k_y}z_{y(i)}\right)\right].\label{eq:dmk-bound}
\end{equation}
 For a selected correct--correct block, its signed difference sum is the sum of selected clean scores minus the sum of selected perturbed scores. Its value is independent of the internal matching. Selecting the smallest clean $k_y$ scores and largest perturbed $k_y$ scores minimizes this sum; reversing the choices maximizes it. The correct-record symmetry permits both choices under every feasible wrong-state refinement. This proves Equation~\eqref{eq:dmk-bound}.

Empty sums are zero. When $K=\sum_y k_y>0$, class numerator extrema are summed and then divided by $K$; an unweighted average of class conditional means would be incorrect. When $K=0$, the conditional mean and its bounds are undefined. Knowing only total $K$ instead of each $k_y$ defines a different feasible set and is not the claim proved here. The argument uses elementary finite selection and separable costs.

Algorithm~\ref{alg:fixed-overlap-bounds} summarizes the computation of Equations~\eqref{eq:Lk-bound} and~\eqref{eq:dmk-bound}. It retains the valid $L$ bounds when $K=0$ and applies to fixed classwise overlaps.

\begin{algorithm}[tbp]
\caption{Sharp bounds with fixed classwise correctness overlap}\label{alg:fixed-overlap-bounds}

\small
\begin{list}{\arabic{enumi}:}{\usecounter{enumi}\setlength{\leftmargin}{1.9em}\setlength{\labelwidth}{1.3em}\setlength{\labelsep}{0.5em}\setlength{\itemsep}{1pt}\setlength{\parsep}{0pt}\setlength{\topsep}{4pt}}
\item \textbf{Require:} Classwise counts $n_y,a_y,b_y,u_y,v_y$, correct-side scores $x_y,z_y$, and feasible overlaps $k_y$; $N>0$.
\item $N\gets\sum_y n_y$, $K\gets\sum_y k_y$; $L^-\gets0$, $L^+\gets0$, $D^-\gets0$, $D^+\gets0$.
\item \textbf{for each class} $y$ \textbf{do}
\item \hspace*{1em} Sort $x_y$ and $z_y$ in ascending order as $x_{y(i)}$ and $z_{y(i)}$.
\item \hspace*{1em} $L^-\gets L^-+\max\{0,k_y-(a_y-u_y)-(b_y-v_y)\}$.
\item \hspace*{1em} $L^+\gets L^++\min\{u_y,v_y,k_y\}$.
\item \hspace*{1em} \textbf{if} $k_y>0$ \textbf{then}
\item \hspace*{2em} $D^-\gets D^-+\sum_{i=1}^{k_y}x_{y(i)}-\sum_{i=b_y-k_y+1}^{b_y}z_{y(i)}$.
\item \hspace*{2em} $D^+\gets D^++\sum_{i=a_y-k_y+1}^{a_y}x_{y(i)}-\sum_{i=1}^{k_y}z_{y(i)}$.
\item \hspace*{1em} \textbf{end if}
\item \textbf{end for}
\item $[\underline L_k,\overline L_k]\gets[L^-/N,L^+/N]$.
\item \textbf{if} $K>0$ \textbf{then}
\item \hspace*{1em} $[\underline d_k,\overline d_k]\gets[D^-/K,D^+/K]$.
\item \textbf{else}
\item \hspace*{1em} Mark the conditional-confidence bounds as undefined.
\item \textbf{end if}
\item \textbf{Ensure:} Lost-acceptance bounds $[\underline L_k,\overline L_k]$ and conditional-confidence bounds, when defined.
\end{list}
\end{algorithm}

\FloatBarrier
\subsection{A counterexample for accepted-error transitions}
Let the true class be $T$, with alternative labels $A,B$, and accept maximum probability at least $0.6$. The two sides contain the probability vectors in Table~\ref{tab:counterexample-records}, in the order $(T,A,B)$.
\begin{table}[H]
\centering
\caption{Prediction records for the accepted-error counterexample.}
\label{tab:counterexample-records}
\begin{tabular}{@{}lll@{}}\toprule
Record & Clean & Perturbed\\\midrule
1 & $(0.6,0.2,0.2)$ & $(0.6,0.2,0.2)$\\
2 & $(0.1,0.8,0.1)$ & $(0.1,0.8,0.1)$\\
3 & $(0.1,0.8,0.1)$ & $(0.3,0.3,0.4)$\\\bottomrule
\end{tabular}
\end{table}
Fix $k_y=0$. The correct clean record may pair with either wrong perturbed record, so the $H$ count can be zero or one. Imposing $s_y=1,w_y=0$ forces the single wrong--wrong edge to terminate at perturbed record 2, leaving record 3 for the correct clean record; consequently $H=0$. Imposing $s_y=0,w_y=1$ forces the opposite allocation, giving count one and rate $H=1/3$. Both are feasible bijections. Adding one correct accepted record to each side and fixing $k_y=1$ gives the same distinction with positive $K$ and denominator four.

Before the wrong-state refinement, let $t_y$ count wrong accepted perturbed records and put $h_y=a_y-k_y$. The sharp $H$ count bounds are
\begin{equation}
\max\{0,h_y-(a_y-u_y)-((n_y-b_y)-t_y)\},\qquad
\min\{u_y,t_y,h_y\}.\label{eq:Hk-bound}
\end{equation}
This is the same subset-intersection argument in the correct--wrong block. Attaining its endpoints may require permuting incorrect perturbed identities, which can change $s_y/w_y$; the invariance proof for $L$ therefore does not apply.

\subsection{Single-policy upper bounds and proof}
Let $h_{jy},l_{jy}$ count the events $H_j,L_j$ within class $y$, and let $u_y$ count correct accepted clean records. In class $y$, the event counts $h_{jy},l_{jy}$ draw from the same $u_y$ clean records and have disjoint perturbed endpoints, so $h_{jy}+l_{jy}\leq u_y$.

\begin{lemma}[Single-policy upper bounds]\label{lem:cap}
For every $\mathcal I\in\{\mathcal I_0,\mathcal I_k,\mathcal I_5\}$, policy $j$ and $\lambda\in[0,1]$,
\begin{equation}
\max_{\pi\in\Pi(\mathcal I)}J_j(\pi)=\frac1N\sum_y\max\{\lambda h+(1-\lambda)l:\ h\in[\underline h_{jy},\overline h_{jy}],\ l\in[\underline l_{jy},\overline l_{jy}],\ h+l\leq u_y\},\label{eq:cap}
\end{equation}
where the intervals are the sharp classwise count bounds for the two events at level $\mathcal I$.
\end{lemma}

Fix a class $y$, a policy $j$ and a level $\mathcal I$. Let $A_y$ be the $u_y$ clean records that are correct and accepted, $W_{jy}$ the perturbed records that are wrong and accepted by policy $j$, and $R_{jy}$ the perturbed records that are correct and rejected by $j$; $W_{jy}$ and $R_{jy}$ are disjoint. For any correspondence, $h_{jy}(\pi)=|\pi_y(A_y)\cap W_{jy}|$ and $l_{jy}(\pi)=|\pi_y(A_y)\cap R_{jy}|$, so $h_{jy}+l_{jy}\leq u_y$, and each count lies in its bounding interval. Hence every feasible $(h_{jy},l_{jy})$ lies in the region on the right of Equation~\eqref{eq:cap}, which proves the inequality $\leq$ after summing over classes.

For $\geq$, in the case $\overline h+\overline l>u_y$, the points $P_1=(\overline h,u_y-\overline h)$ and $P_2=(u_y-\overline l,\overline l)$ lie in the region: a correspondence attaining $\overline h$ has $l\leq u_y-\overline h$, so $\underline l\leq u_y-\overline h$, and symmetrically $\underline h\leq u_y-\overline l$. The maximum of a linear function with nonnegative coefficients over the region is attained at $(\overline h,\overline l)$ if $\overline h+\overline l\leq u_y$, and otherwise on the segment $h+l=u_y$ between $P_1$ and $P_2$, hence at $P_1$ or $P_2$; in both cases it suffices to exhibit correspondences in $\Pi(\mathcal I)$ with $h\geq\overline h$ and $l\geq\min(\overline l,u_y-\overline h)$, and symmetrically with $l\geq\overline l$ and $h\geq\min(\overline h,u_y-\overline l)$. Under $\mathcal I_0$, match $\overline h=\min(u_y,|W_{jy}|)$ records of $A_y$ to $W_{jy}$, match $\min(|R_{jy}|,u_y-\overline h)=\min(\overline l,u_y-\overline h)$ of the remaining records of $A_y$ to $R_{jy}$, and complete the bijection arbitrarily; this is feasible because the two target sets are disjoint. Under $\mathcal I_k$ or $\mathcal I_5$, start from any $\pi^\ast\in\Pi(\mathcal I)$ attaining $h_{jy}(\pi^\ast)=\overline h$; its correct--wrong block contains at least $\overline h$ columns in $W_{jy}$. Apply the correct-record symmetry of Appendix~\ref{app:d2}, which preserves the level: permute clean correct identities so that $\overline h$ records of $A_y$ occupy those columns and $\min(\overline l,u_y-\overline h)$ further records of $A_y$ occupy rows of the correct--correct block, and permute perturbed correct identities so that $\min(k_y,|R_{jy}|)\geq\min(\overline l,u_y-\overline h)$ columns of that block lie in $R_{jy}$ and are matched to those rows. The resulting correspondence has $h\geq\overline h$ and $l\geq\min(\overline l,u_y-\overline h)$. The symmetric vertex is built from the same $\pi^\ast$: place $\overline l=\min(u_y,|R_{jy}|,k_y)$ records of $A_y$ on $R_{jy}$-columns of the correct--correct block (reachable by the row and column permutations, since that block is free at every level) and $\min(\overline h,u_y-\overline l)$ further records of $A_y$ on the $W_{jy}$-columns of the correct--wrong block of $\pi^\ast$, of which there are at least $\overline h$. Starting instead from a correspondence that attains $\overline l$ can fail at $\mathcal I_5$, because the wrong perturbed records in the correct--wrong block are not free at that level. Classes are independent, so the classwise maxima add. The identity holds for every cost weight and every level; the panel computation of Appendix~\ref{app:joint-policy} confirms it numerically with a maximum residual of $1.39\times10^{-15}$ percentage points. The corresponding lower bound and the difference of two policies are not covered: the difference $J_j-J_\ell$ depends on which records switch acceptance between $\tau_j$ and $\tau_\ell$, and its maximum over a shared correspondence can be smaller than the difference of separate endpoints (the shared-correspondence comparison in Section~\ref{sec:joint-policy}).

\subsection{Proof of Corollary~\ref{cor:natural} (comparisons under a natural cost)}
$E_{\tau_j}$ is a function of the perturbed records alone and is therefore constant on $\Pi(\mathcal I)$ at every level. For any $\pi$,
\[
L_j(\pi)-L_\ell(\pi)=\frac1N\sum_y\sum_{i\in A_y,\ C'_{\pi_y(i)}=1}(Z'_{\pi_y(i),\ell}-Z'_{\pi_y(i),j}),
\]
which depends on $\pi$ only through its correct--correct block: which clean correct records are matched to which perturbed correct records. By the correct-record symmetry, from any correspondence in $\Pi(\mathcal I_5)$ every choice of $k_y$ clean correct records, $k_y$ perturbed correct records and bijection between them is reachable inside $\Pi(\mathcal I_5)$, and $\Pi(\mathcal I_k)$ admits exactly the same set of correct--correct blocks because $\Pi(\mathcal I_5)\subseteq\Pi(\mathcal I_k)$ and a correct--correct block of size $k_y$ is completable at level $\mathcal I_k$ whenever $k_y$ is feasible. Hence the vector $(L_j(\pi))_j$ ranges over the same set under both levels, and so does every function of it, including $J'_j$ and $J'_j-J'_\ell$. The same argument shows that the $L_j$ endpoints coincide at the two levels, which is the $L$ part of Theorem~\ref{prop:sufficiency}. Under the cost $J_j=\lambda H_j+(1-\lambda)L_j$ the argument fails because $H_j$ depends on the correct--wrong block, whose reachable set differs between the levels (the accepted-error result in Section~\ref{sec:retro}).

\subsection{Scope of the identification experiment}
The classwise choices vary independently within the mathematical feasible set, without implying statistical independence between classes or between experimental cells. Neither the endpoints nor their width are a confidence interval, a causal effect, or a guarantee that a camera can realize the extremizing correspondence. With multiple acceptable labels, two correct predictions can differ. If correctness is attached to each record under its own accepted-label set, the symmetry argument applies unchanged, and enumeration under ReaL labels confirms identical $L$ and $dM_k$ endpoints at $\mathcal I_k$ and $\mathcal I_5$ (Appendix~\ref{app:labels}); reassigning a record to another image and re-judging its correctness under that image's label set is a different model and is not covered. Additional pair-dependent constraints or acceptance rules can also invalidate the symmetry.

\subsection{Variable-support confidence bounds and finite verification}
For the expanded panel, we additionally compute a marginal-only $dM_k$ range over all assignments with positive total $K$. Unlike Equation~\eqref{eq:dmk-bound}, the denominator is allowed to vary. For each class and each feasible integer $k_y$, the preceding score-subset formulas give the minimum and maximum numerator. For a candidate ratio $r$, minimizing $\sum_y[D_y(k_y)-rk_y]$ over class choices detects whether a positive-support ratio below $r$ exists; the maximization version detects a ratio above it. The all-zero choice cannot supply a strict inequality. Bisection on $[-1,1]$ yields the endpoints. If no positive support is feasible, the ratio is undefined. These extra panel ranges complement the fixed-support theorem.

Numerical verification uses independent permutation enumeration, including empty overlap, tied probabilities, threshold equality and unequal class supports. The five-state endpoints in the expanded panel are obtained by exact enumeration of integer permutations. Verification covers both algebraic cases and actual prediction records.

\section{Coverage, Matched Support, and Numerical Verification}\label{app:empirical}
The analyses below support Section~\ref{sec:retro-empirical} by documenting the full analytic grid, the 450-cell panel and the matched-support comparisons.

\subsection{Frozen panel and complete analytic grid}
The panel contains 440 deterministically selected expanded cells plus the 16 historical cells, with six duplicates removed. It includes 90 Syn, 90 Syn-B, 135 ES and 135 Diverse cells. Each of the 44 models appears, with all ten fixed groups represented. Conditions cover 18/18 Syn, 18/18 Syn-B, 54/54 ES and 113/162 Diverse settings. The model--condition pairs were fixed after prior exploration and before computing the new outcomes, with no outcome-based replacements.

The selected real conditions cover every lighting setting and all three levels of ISO, shutter speed and aperture, with incomplete coverage of their combinations. Diverse light-code coverage is 27/27 for each of l1, l2 and l3; 17/27 for l4; 10/27 for l6; and 5/27 for l7. The panel omits 49 Diverse combinations, yielding unequal coverage across lighting settings.

The full grid has 11,088 cells and 133,056 target--information-level rows; Tables~\ref{tab:d2-events} and~\ref{tab:d2-panel} report the full-grid and panel results, respectively. Analytic event bounds and fixed-overlap confidence bounds cover the full grid. Five-state $H$ bounds and marginal variable-support confidence bounds are evaluated on the 450-cell panel. Conditional means with zero overlap are undefined. With overlap fixed, 461 of the 10,968 positive-overlap $dM_k$ intervals straddle zero; the 120 zero-overlap cells remain undefined.

\begin{table}[!htbp]
\centering
\caption{Full-grid mean identification widths (percentage points). $\mathcal I_0$ retains complete classwise marginal records; $\mathcal I_k$ additionally fixes classwise correct--correct counts. For $L,dM_k$, the five-state bounds equal $\mathcal I_k$ by Theorem~\ref{prop:sufficiency}. Conditional-confidence means exclude zero-$K$ cells. Cells are weighted equally within each source.}
\label{tab:d2-events}

\small
\begin{tabular}{@{}lrrrrrr@{}}
\toprule
& & \multicolumn{2}{c}{$H$} & \multicolumn{2}{c}{$L$} & $dM_k$\\ Source & Cells & $\mathcal I_0$ & $\mathcal I_k$ & $\mathcal I_0$ & $\mathcal I_k$ & $\mathcal I_k$\\\midrule
Syn & 792 & 8.44 & 3.44 & 6.73 & 4.63 & 6.01\\
Syn-B & 792 & 1.96 & 0.66 & 6.33 & 5.42 & 2.11\\
ES & 2,376 & 1.56 & 0.79 & 5.52 & 4.44 & 3.52\\
Diverse & 7,128 & 0.86 & 0.69 & 2.86 & 1.99 & 5.76\\
\bottomrule
\end{tabular}
\end{table}

\begin{table}[!htbp]
\centering
\caption{Expanded panel: mean $H$ widths (pp), additional narrowing from wrong-state refinement, and number of cells with strictly narrower bounds. Panel means use the same cells at every level.}
\label{tab:d2-panel}

\small
\begin{tabular}{@{}lrrrrrl@{}}
\toprule
Source & Cells & $\mathcal I_0$ & $\mathcal I_k$ & $\mathcal I_5$ & Extra & Narrower\\\midrule
Syn & 90 & 8.45 & 3.47 & 2.40 & 1.07 & 90/90\\
Syn-B & 90 & 1.92 & 0.69 & 0.54 & 0.15 & 25/90\\
ES & 135 & 1.12 & 0.77 & 0.68 & 0.09 & 38/135\\
Diverse & 135 & 0.97 & 0.78 & 0.68 & 0.10 & 42/135\\
\bottomrule
\end{tabular}
\end{table}

\FloatBarrier
\subsection{Matching protocol and event support}
\begin{table}[!htbp]
\centering
\caption{Mean five-state width before and after matching to 200 common true classes and two evaluation records per class (pp). Matched values first average defined cells within each seed and then the five seeds equally. Candidate-probability spaces retain their original sizes of 1,000 classes for Syn and 200 for the other sources.}
\label{tab:matching}

\small
\begin{tabular}{@{}lrrrrrr@{}}
\toprule
& \multicolumn{2}{c}{$H$} & \multicolumn{2}{c}{$L$} & \multicolumn{2}{c}{$dM_k$}\\Source & Original & Matched & Original & Matched & Original & Matched\\\midrule
Syn & 2.40 & 0.87 & 4.61 & 2.52 & 5.97 & 2.76\\
Syn-B & 0.54 & 0.41 & 5.47 & 4.74 & 2.13 & 1.76\\
ES & 0.68 & 0.54 & 3.03 & 2.38 & 5.11 & 3.90\\
Diverse & 0.68 & 0.52 & 2.66 & 2.04 & 5.84 & 4.27\\
\bottomrule
\end{tabular}
\end{table}
The matched comparison (Table~\ref{tab:matching}) fixes the 200 common true classes and selects two evaluation records per class using an outcome-independent SHA-256 ordering of seed, source, synset and image identifier. The five seeds are 20260910--20260914. Every selected model--condition cell then has 400 records. All original clean-development thresholds are retained. The sample membership is fixed across models within each source/condition selection; no score or outcome enters the selection hash. Candidate scoring spaces remain unchanged: 1,000 for Syn and 200 elsewhere. Matching to two images per class reverses the native Syn-versus-real $dM_k$ width ordering in all five seeds, so the original ordering cannot be attributed to capture medium alone.

Seed summaries first average defined cell widths within each seed, then weight the five seed means equally. Their minimum and maximum are sensitivity summaries, not confidence intervals. There are seven undefined matched cell--seed confidence means: five Diverse and two ES. Original-panel $dM_k$ is defined in 449/450 cells. Positive remaining five-state widths occur in 301/450 cells for H, 422/450 for L and 444/449 for $dM_k$. After matching, the corresponding counts are 1,244/2,250, 2,036/2,250 and 2,202/2,243. Identical endpoints at two information levels therefore rarely imply point identification of the actual value.

Event scale also affects width comparisons. Dividing each H width by its maximum feasible event mass, when that denominator is positive, gives mean native-panel ratios 0.2054 for Syn, 0.2649 for ES and 0.3381 for Diverse. The matched ratios are 0.0892, 0.2245 and 0.2576. Thus the ordering of absolute widths is not an ordering of this normalized ambiguity. The supplementary tables stratify by $K/N$ in $[0,.1),[.1,.5),[.5,1]$ and by left/right event fractions in $[0,.1),[.1,.3),[.3,1]$. Mean widths for empty strata and ratios with zero denominators are undefined. Counts of exposed records across conditions are not numbers of unique independent images.

\FloatBarrier
\subsection{Independent verification and computational cost}
Formula checks independently enumerate 252 normalized record pairs, 212,868 permutations and 905 feasible overlap/persistent-wrong strata, including tied probabilities, threshold equality, zero overlap and unequal class supports. Maximum formula discrepancy is $4.44\times10^{-16}$ in probability units. A second implementation independently recomputes 192 endpoints from eight native and eight matched original-record cells spanning all sources and different overlap supports; the maximum discrepancy is $1.60\times10^{-14}$ percentage points. Its marginal conditional-mean check uses explicit feasible-support dynamic programming rather than the primary ratio-bisection procedure. Additional checks confirm information-set nesting, inclusion of observed values, identical matched samples across models, and consistent aggregation by support stratum.

For the historical 16 cells, all 48 target comparisons between the previous constrained LP and new integer endpoints agree to $2.49\times10^{-14}$ points. This is a finite comparison, not a universal LP-integrality claim. The enlarged panel uses integer permutations directly for five-state endpoints. Original H/L marginal bounds and actual statistics also agree over the full grid with the frozen earlier outputs.

The hierarchy analysis of 9,108,000 paired evaluations took 218.35 seconds on CPU, and the matched-support analysis took 88.07 seconds. These calculations use existing predictions; the reported times exclude model inference and image acquisition.
\subsection{Signed decomposition and survivor support}
Beyond interval widths, this decomposition examines the difference between the observed retained-correct confidence decrease and its class-conditional independent counterpart.

Let $K_y$ be the observed count of pairs correct on both sides, $K=\sum_yK_y$, and $\widetilde K_y=a_yb_y/n_y$ its independent mass; $a_y,b_y$ are the two sides' correct counts. Write $b_y^M=\mu_y-\mu_y'$ for the difference of the two marginal correct-record means and $\delta_y$ for the observed mean difference on surviving pairs. When $K>0$,
\begin{equation}
dM_k-dM_{k,\ind}=
\underbrace{\sum_y\left(\frac{K_y}{K}-\frac{\widetilde K_y}{\sum_z\widetilde K_z}\right)b_y^M}_{W:\ \text{class composition}}
+\underbrace{\sum_{y:K_y>0}\frac{K_y}{K}(\delta_y-b_y^M)}_{V:\ \text{within-class selection}}.\label{eq:selection}
\end{equation}
Define $b_y^M$ only when $a_yb_y>0$ and set it to zero otherwise; those classes have zero observed and independent overlap weights. Independence uses a ratio of expected numerator and denominator, not an expectation of a random conditional ratio. The two terms are signed contributions to the observed difference.

Equation~\eqref{eq:selection} follows by adding and subtracting $\sum_y(K_y/K)b_y^M$. The composition term changes class weights; the within-class term compares the surviving correct records to each side's complete correct set. Both terms may have either sign and may cancel. Their absolute magnitudes do not define explained-variance or causal-contribution fractions. Absolute-component comparisons (Table~\ref{tab:d2-decomp}) use a numerical tie tolerance of $10^{-8}$ percentage points: a selection term is counted as larger only when $|V|>|W|+10^{-8}$ in these units.

\begin{table}[!htbp]
\caption{Signed confidence decomposition on positive-$K$ cells. Means are cell-weighted, confidence differences are in percentage points, and the last column reports the percentage of cells where the absolute within-class selection term exceeds the absolute composition term.}
\label{tab:d2-decomp}\centering

\small
\begin{tabular}{@{}lrrrrr@{}}
\toprule
Source & $K>0$ cells & Pair $-$ ind. & Mean $|W|$ & Mean $|V|$ & $|V|>|W|$\\
\midrule
Syn & 792 & +1.00 & 0.08 & 0.96 & 98.6\%\\
Syn-B & 792 & +0.38 & 0.09 & 0.39 & 85.0\%\\
ES & 2,374 & +0.52 & 0.21 & 0.67 & 86.3\%\\
Diverse & 7,010 & +0.66 & 0.50 & 1.57 & 76.8\%\\
\bottomrule
\end{tabular}
\space
\end{table}

There are 120 zero-$K$ cells: 118 Diverse and two ES. The independent overlap mass is zero in 102 cells; in another 18, observed overlap is zero but independent overlap is positive. Diverse has 2,235 cells with $1\leq K\leq5$ and 4,042 with $K\geq20$. Its mean $dM_k$ is 43.84 over positive-$K$ cells, versus 21.06 over the latter subset; these changing subsets are descriptive and are not support-corrected estimates of a common population mean. In that $K\geq20$ subset, mean absolute selection/composition terms remain 0.89/0.30 points.

\FloatBarrier
\section{Additional Checks of Correspondence Information}\label{app:correspondence-value}
Three checks extend the correspondence analysis: verifying attainable endpoints by enumeration, measuring the precision gain from pairing, and bounding overlap when only accuracies are reported.

The analyses use the fixed evaluation records described in Section~\ref{sec:design}. Unless a table states resampling of the 200 true classes, intervals are percentile intervals from resampling the ten model groups (2,000 draws); both summarize sensitivity within this pool, not population uncertainty. Tiers are the development-defined tiers of Section~\ref{sec:design}.

\subsection{Theorem~\ref{prop:sufficiency} at scale}
Direct permutation enumeration, independent of the analytic formulas, covers every class in 48 cells spanning all four sources and four models (\enumClassesOrig{} classes; Syn evaluation classes hold five records, real classes two or three, so every class is enumerable). In every class the $L$ and $dM_k$ extrema at $\mathcal I_k$ equal those at $\mathcal I_5$ and match Equations~\eqref{eq:Lk-bound} and~\eqref{eq:dmk-bound}. The five-state $H$ interval is strictly narrower than the $\mathcal I_k$ interval in \enumNarrowerOrig{} classes (\enumNarrowerOrigPct\%), most often on Syn, because a class with few records rarely contains both a persistent and a changed wrong answer. Appendix~\ref{app:labels} repeats the enumeration under ReaL correctness on all real cells.

\subsection{Variance reduction from pairing}
Actual pairing also improves the precision of a measured change. The variance of the paired difference divided by the sum of marginal variances is about one third for correctness under digital corruptions, versus 0.55 on ES and 0.87 on Diverse. It approaches one on extreme recapture, where perturbed correctness variance collapses. Pairing improves precision while preserving the expected change.

Table~\ref{tab:variance-ratio-full} reports the ratios with group-resampled intervals by source and tier. The ratio varies across model groups (for example, 0.42 to 0.63 for correctness on ES) and rises with tier within each real source. The standard error of an estimated accuracy change scales with the square root of the ratio, so pairing shortens it by about 40\% under digital corruptions and by less than 10\% on extreme recapture.

\begin{table}[!htbp]
\centering
\caption{Paired-to-unpaired variance ratio with group-resampled 95\% intervals, by source and tier, for correctness $C$, maximum probability $S$ and true-class probability $M$. Cells are model--condition pairs, weighted equally.}
\label{tab:variance-ratio-full}

\small
\begin{tabular}{@{}llrlll@{}}\toprule
Source & Tier & Cells & $C$ & $S$ & $M$\\\midrule
Syn & all tiers & 792 & 0.32 [0.28, 0.35] & 0.32 [0.27, 0.37] & 0.20 [0.16, 0.23]\\
Syn & mild & 792 & 0.32 [0.28, 0.35] & 0.32 [0.27, 0.37] & 0.20 [0.16, 0.23]\\
Syn-B & all tiers & 792 & 0.33 [0.29, 0.37] & 0.31 [0.25, 0.36] & 0.20 [0.17, 0.23]\\
Syn-B & mild & 792 & 0.33 [0.29, 0.37] & 0.31 [0.25, 0.36] & 0.20 [0.17, 0.23]\\
ES & all tiers & 2376 & 0.55 [0.53, 0.58] & 0.55 [0.50, 0.59] & 0.44 [0.41, 0.46]\\
ES & mild & 1936 & 0.47 [0.44, 0.50] & 0.47 [0.41, 0.51] & 0.34 [0.31, 0.36]\\
ES & severe & 220 & 0.87 [0.86, 0.88] & 0.89 [0.87, 0.91] & 0.82 [0.80, 0.83]\\
ES & extreme & 220 & 0.95 [0.94, 0.96] & 0.96 [0.95, 0.97] & 0.93 [0.92, 0.94]\\
Diverse & all tiers & 7128 & 0.87 [0.86, 0.89] & 0.89 [0.87, 0.90] & 0.83 [0.82, 0.85]\\
Diverse & mild & 2244 & 0.71 [0.68, 0.74] & 0.72 [0.67, 0.76] & 0.61 [0.57, 0.64]\\
Diverse & severe & 1100 & 0.85 [0.84, 0.87] & 0.88 [0.86, 0.90] & 0.80 [0.78, 0.83]\\
Diverse & extreme & 3784 & 0.98 [0.97, 0.98] & 0.99 [0.98, 0.99] & 0.98 [0.97, 0.98]\\
\bottomrule
\end{tabular}

\end{table}

\FloatBarrier
\subsection{Identification from marginals only}
Table~\ref{tab:marginal-K} treats the coarser information level in which only per-class, or only global, accuracies are reported. With $a_y,b_y$ the clean and perturbed correct counts, each $k_y$ lies in $[\max(0,a_y+b_y-n_y),\min(a_y,b_y)]$; summing over classes bounds $K$, hence $h=A-K$ and $r=B-K$ for total correct counts $A,B$, while the net change $A-B$ is always identified. Per-class accuracies reduce the width obtained from global accuracies by 40 to 60\%. On Syn, with five evaluation records per class, the width of $K$ is of the same order as the observed gross flip rate, so gross flips are identified only to within a range comparable to their observed magnitude. On the real sources, the two or three records per class sharply limit feasible overlaps, fixing the overlap in most classes. This level sits below $\mathcal I_0$, which retains complete two-sided records.

\begin{table}[!htbp]
\centering
\caption{Identification of the total correct--correct overlap $K$ from accuracies alone (percentage points of $N$). Classwise sums per-class Fr\'echet intervals; global uses only $A$, $B$ and $N$. Gross flips $h+r=A+B-2K$ inherit twice the width of $K$. Pinned is the share of classes whose interval is a single value.}
\label{tab:marginal-K}

\footnotesize
\begin{tabular}{@{}llrrrr@{}}\toprule
Source & Tier & $K$ width, classwise & $K$ width, global & Observed $h{+}r$ & Pinned (\%)\\\midrule
Syn & all tiers & 13.78 & 23.31 & 13.87 & 48.4\\
Syn & mild & 13.78 & 23.31 & 13.87 & 48.4\\
Syn-B & all tiers & 6.44 & 11.32 & 9.33 & 83.9\\
Syn-B & mild & 6.44 & 11.32 & 9.33 & 83.9\\
ES & all tiers & 5.16 & 10.39 & 26.04 & 87.1\\
ES & mild & 5.90 & 11.15 & 16.02 & 85.2\\
ES & severe & 2.85 & 9.14 & 59.83 & 92.9\\
ES & extreme & 0.91 & 4.97 & 80.38 & 97.7\\
Diverse & all tiers & 2.18 & 5.76 & 64.39 & 94.6\\
Diverse & mild & 4.72 & 10.17 & 32.33 & 88.2\\
Diverse & severe & 3.00 & 9.18 & 57.75 & 92.5\\
Diverse & extreme & 0.43 & 2.15 & 85.35 & 98.9\\
\bottomrule
\end{tabular}

\end{table}

\FloatBarrier
\section{Shared Correspondence in Rejection-Policy Comparisons}\label{app:joint-policy}

\paragraph{Evaluation protocol.}
This appendix gives the protocol behind Section~\ref{sec:joint-policy}: whether the information hierarchy of Definition~\ref{def:levels} changes which rejection policies can be ruled out on fixed records. It uses the previously frozen 450-cell panel: 90 Syn, 90 Syn-B, 135 ES and 135 Diverse cells, spanning all 44 models. Syn contributes 5,000 evaluation records per cell, five per class; the other sources contribute 500, two or three per class. With true labels known on both sides, decisions are made separately within each model--condition cell, with no common correspondence constraint across conditions. The resulting comparisons characterize policy costs under the feasible correspondences of these labeled records.

The clean acceptance rule remains $Z_i=\mathbf1[S_i\geq\tau_0]$, with $\tau_0$ fixed at the original 80\% clean-development threshold. The five perturbed-side policies use $Z_{b,j}'=\mathbf1[S_b'\geq\tau_j]$, where $\tau_j$ is the descending clean-development order statistic at rank $\operatorname{round}(q_j n_{\rm dev})$, for $q_j\in\{.6,.7,.8,.9,1\}$. All cutoff ties are accepted, although a threshold for 100\% development coverage need not accept every evaluation record. Neither the candidate thresholds nor the cost weights are tuned using evaluation outcomes. The frozen weights are $\lambda\in\{0,.1,\ldots,1\}$, with $.5$ used for the primary comparison.

\paragraph{Events under one shared correspondence.}
Let $\pi_y$ map clean to perturbed records within class $y$, and $\pi=(\pi_y)_y$. Extend the definitions of $H$ and $L$ in Table~\ref{tab:reliability-targets} to policy $j$ by
\begin{align}
H_j(\pi)&=\frac1N\sum_y\sum_{i\in y} C_iZ_i(1-C'_{\pi_y(i)})Z'_{\pi_y(i),j},\nonumber\\
L_j(\pi)&=\frac1N\sum_y\sum_{i\in y} C_iZ_i C'_{\pi_y(i)}(1-Z'_{\pi_y(i),j}),\nonumber\\
J_j(\pi)&=\lambda H_j(\pi)+(1-\lambda)L_j(\pi).
\label{eq:joint-policy-cost}
\end{align}
All rates use $N=\sum_y n_y$. The cost measures changes to clean predictions that were correct and accepted. Evaluation coverage and the total accepted-error rate provide complementary perturbed-side summaries.

For $\mathcal I\in\{\mathcal I_0,\mathcal I_k,\mathcal I_5\}$, write $\Pi_{\mathcal I,y}$ for all classwise bijections satisfying that information, and $\Pi_{\mathcal I}=\prod_y\Pi_{\mathcal I,y}$. These are the feasible sets of Definition~\ref{def:levels}: unrestricted within-class permutations, permutations with the observed classwise correct--correct counts, and permutations with all five observed classwise state counts. Scores and acceptance decisions stay attached to records. Every policy in a comparison faces the \emph{same} $\pi$.

The sharp upper difference bound and worst-case regret are
\begin{align}
D^+_{j\ell}(\mathcal I)&=\sup_{\pi\in\Pi_{\mathcal I}}\{J_j(\pi)-J_\ell(\pi)\},\nonumber\\
\mathcal R_j(\mathcal I)&=\sup_{\pi\in\Pi_{\mathcal I}}\{J_j(\pi)-\min_\ell J_\ell(\pi)\}
=\max_\ell D^+_{j\ell}(\mathcal I),\nonumber\\
&=\frac1N\max_\ell\sum_y\max_{\pi_y\in\Pi_{\mathcal I,y}}
\{c_{jy}(\pi_y)-c_{\ell y}(\pi_y)\},
\label{eq:joint-policy-regret}
\end{align}
where $c_{jy}$ is the classwise weighted event \emph{count} in Equation~\eqref{eq:joint-policy-cost}. The comparator $\ell$ is a single policy for the cell: its maximization stays outside the class sum. Exact enumeration of classwise record permutations gives these extrema.

Policy $j$ has strictly lower cost than $\ell$ throughout the feasible set if $D^+_{j\ell}<0$. Separate sharp intervals establish this only when $\overline J_j-\underline J_\ell<0$; their two endpoints can require different correspondences. We count comparisons established by the shared difference but missed by that same-information interval test. Numerical strictness uses $D^+_{j\ell}<-10^{-12}$. Establishing that a policy is no worse than every candidate is a separate property.

\paragraph{Selection rules and controls.}
Minimax regret is a classical criterion used, for example, in statistical treatment choice \citep{manski2004}. Here uncertainty concerns correspondence within fixed labeled records, rather than sampling from a population. The minimax-regret rule minimizes $\mathcal R_j$. Its primary comparator minimizes $\E_{\pi\sim\mathrm{Unif}(\Pi_{\mathcal I})}[J_j(\pi)]$ using the \emph{same information}. This mean weights feasible record permutations uniformly, retaining their multiplicities; it does not weight distinct count tables uniformly. Under $\mathcal I_0$ it equals the classwise product-independent value. That unconstrained product assignment need not satisfy the $\mathcal I_k$ or $\mathcal I_5$ constraints and is retained only as a weaker control. Other controls keep the original 80\% threshold or minimize the marginal proxy $\lambda\Prb(C'=0,Z_j'=1)+(1-\lambda)\Prb(C'=1,Z_j'=0)$. For all minimizations, values within $10^{-12}$ of the minimum are tied; ties favor nominal coverage closest to $.8$, then lower coverage, then policy index. Actual paired costs play no role in this tie rule. Realized regret is $J_j(\pi^{\rm obs})-\min_\ell J_\ell(\pi^{\rm obs})$.

\paragraph{Which comparisons become determined?}
At $\lambda=.5$, $\mathcal I_5$ establishes 641 directed comparisons in 312/450 cells that the separate $J$ intervals miss. These counts require both correct and incorrect perturbed records among those whose acceptance changes between the policies, excluding switches with only one correctness outcome. The cell counts are 82/90 Syn, 60/90 Syn-B, 92/135 ES and 78/135 Diverse. A different comparison measures information increments: $\mathcal I_0\to\mathcal I_k$ adds 586 mixed-outcome strict comparisons in 241 cells, and $\mathcal I_k\to\mathcal I_5$ adds 52 in 42 cells. These information increments use a different comparator from the 641 gains over separate intervals. The latter control also strengthens with information, so the number it misses need not increase across information levels.

After removing candidates strictly worse than another throughout the feasible set, a \emph{post hoc diagnostic} merges thresholds with identical evaluation acceptance vectors, treating each distinct vector as one action. The number of cells with one remaining distinct action is 148, 189 and 191 under $\mathcal I_0,\mathcal I_k,\mathcal I_5$. This merging of identical actions was not a prespecified primary endpoint; the five nominal-policy results are also retained. Additional strict comparisons often leave several actions undominated. The $L_j$ endpoints remain identical under $\mathcal I_k$ and $\mathcal I_5$, consistent with Theorem~\ref{prop:sufficiency}; five-state increments in the combined comparisons do not contradict target-specific sufficiency.

Figure~\ref{fig:joint-policy-sensitivity} shows how the fraction of cells with additional determined policy comparisons varies across the frozen cost-weight grid. It measures which additional policy orderings can be established; realized policy-selection gains are reported separately below.

\paragraph{Shared-correspondence difference bound.}
The certificate $\overline\Delta_{j\ell}(\mathcal I)$ of the shared-correspondence comparison in Section~\ref{sec:joint-policy}, Equation~\eqref{eq:joint-policy-difference}, is the quantity written $D^+_{j\ell}(\mathcal I)$ above; Table~\ref{tab:joint-policy-main} counts the cells in which it certifies an ordering that the separate sharp intervals leave open.

\begin{table}[!htbp]
\centering
\caption{Cells with at least one additional strict policy comparison at $\lambda=0.5$, restricted to threshold-switching records containing both correct and incorrect predictions. Shared $\pi$ compares common-correspondence differences against separate sharp policy intervals within $\mathcal I_5$; $+k$ and $+5$ instead add information ($\mathcal I_0\to\mathcal I_k$ and $\mathcal I_k\to\mathcal I_5$). Counts summarize cells of the frozen panel.}
\label{tab:joint-policy-main}

\small
\begin{tabular}{@{}lrrrr@{}}\toprule
Source & Cells & Shared $\pi$ & $+k$ & $+5$\\\midrule
Syn & 90 & 82 & 77 & 28\\
Syn-B & 90 & 60 & 75 & 6\\
ES & 135 & 92 & 47 & 2\\
Diverse & 135 & 78 & 42 & 6\\
\midrule
Total & 450 & 312 & 241 & 42\\\bottomrule
\end{tabular}
\end{table}

\begin{figure}[t]
\centering\includegraphics[width=\linewidth]{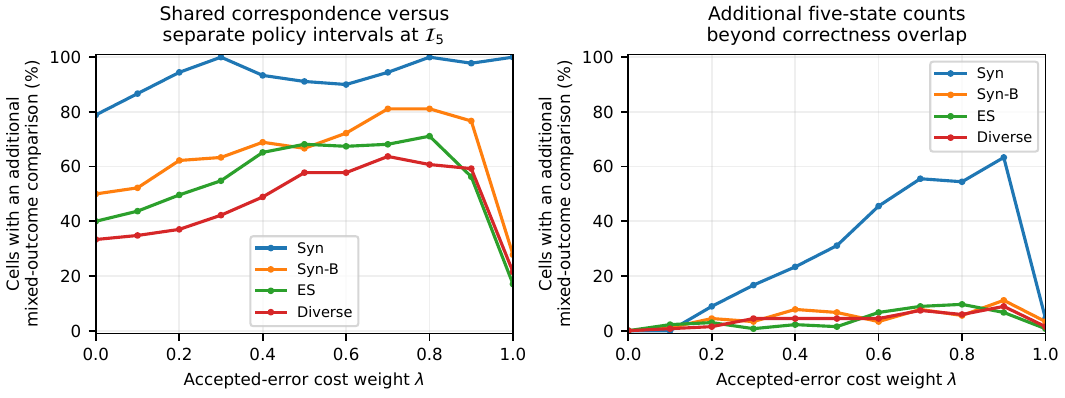}
\caption{Sensitivity to the frozen cost-weight grid. Fractions give the share of cells in each source's fixed panel with an additional certified ordering. Shared-correspondence comparisons use separate sharp cost intervals as the control; information increments compare successive information levels. The primary weight is $\lambda=.5$.}
\label{fig:joint-policy-sensitivity}
\end{figure}

\FloatBarrier
\paragraph{Single-policy bounds and policy-selection results.}
Within each class, the two event counts satisfy $h_{jy}+l_{jy}\leq u_y$, where $u_y$ counts correct accepted clean records, and Lemma~\ref{lem:cap} states that maximizing the weighted cost over the separate sharp $H_j/L_j$ count intervals intersected with this cap equals the joint single-policy upper bound. The panel computation confirms the identity: the maximum residual over 74,250 policy--weight--information rows is $1.39\times10^{-15}$ percentage points. Once the sharp event bounds and the support cap are supplied, joint optimization adds nothing to the single-policy upper bound. Shared-correspondence differences can still strengthen comparisons. The $+5$ column of Table~\ref{tab:joint-policy-main} is specific to the cost $J_j$, which charges accepted errors by their transition origin $H_j$: under the natural cost $\lambda E_{\tau_j}+(1-\lambda)L_j$, Corollary~\ref{cor:natural} makes the $\mathcal I_k\to\mathcal I_5$ increment zero in every cell, because $E_{\tau_j}$ is marginal and the $L_j$ differences range over the same set at both levels.

At $\lambda=.5$ under $\mathcal I_5$, minimax regret yields lower realized cost than the same-information feasible-permutation mean rule in four cells, higher cost in three and equal cost in 443. Table~\ref{tab:joint-policy-regret} reports the correspondingly small realized-regret differences. Gains over the fixed 80\% threshold are larger, but that control was not optimized for $J$. Worst-case optimality follows from the minimax definition; the realized-cost comparison evaluates selection performance. The evidence supports additional finite-record comparisons, while realized-selection gains are small and mixed.

\begin{table}[!htbp]
\centering
\caption{Mean realized regret in percentage points at $\lambda=.5$, $\mathcal I_5$, averaging cells equally within source. Both rules use the same feasible permutations and measure regret relative to the lowest-cost policy under the observed correspondence. Better/Worse/Tied count cells where minimax has lower/higher/equal realized cost than the feasible-permutation mean rule.}
\label{tab:joint-policy-regret}

\small
\begin{tabular}{@{}lrrrrrr@{}}\toprule
Source & Cells & Feasible mean & Minimax & Better & Worse & Tied\\\midrule
Syn & 90 & 0.004778 & 0.004222 & 1 & 1 & 88 \\
Syn-B & 90 & 0.005556 & 0.005556 & 0 & 0 & 90 \\
ES & 135 & 0.012593 & 0.004444 & 3 & 0 & 132 \\
Diverse & 135 & 0.003704 & 0.005185 & 0 & 2 & 133 \\
\bottomrule
\end{tabular}
\end{table}

Table~\ref{tab:joint-policy-weight-selection} reports the full frozen weight grid. Mean realized-regret reduction is negative at six of the eleven weights. These small differences change sign across weights and do not establish a consistent actual-selection advantage.

\begin{table}[!htbp]
\centering
\caption{Realized policy-selection results over the full frozen weight grid at $\mathcal I_5$. $\Delta$ is feasible-permutation mean regret minus minimax regret, in percentage points, averaged equally over all 450 cells; negative values favor the mean rule. Better/Worse/Tied count lower/higher/equal realized costs for minimax. }
\label{tab:joint-policy-weight-selection}

\small
\begin{tabular}{@{}rrrrr@{}}\toprule
$\lambda$ & $\Delta$ (pp) & Better & Worse & Tied\\\midrule
0.0 & 0.000000 & 0 & 0 & 450 \\
0.1 & 0.000044 & 1 & 0 & 449 \\
0.2 & -0.000542 & 0 & 3 & 447 \\
0.3 & -0.002400 & 2 & 5 & 443 \\
0.4 & -0.003111 & 4 & 7 & 439 \\
0.5 & 0.002111 & 4 & 3 & 443 \\
0.6 & 0.001182 & 7 & 9 & 434 \\
0.7 & -0.000076 & 3 & 9 & 438 \\
0.8 & -0.001396 & 2 & 7 & 441 \\
0.9 & -0.002791 & 5 & 8 & 437 \\
1.0 & 0.000000 & 0 & 0 & 450 \\
\bottomrule
\end{tabular}
\end{table}

\FloatBarrier
\paragraph{Numerical verification.}
Thirteen numerical checks cover small global Cartesian-product enumerations, non-simultaneously attainable extrema, ties and boundary cases. Across all 450 cells, verification confirms the 80\% development thresholds, the earlier $H/L$ results, finite values, inclusion of observed values, information nesting, comparison direction and regret aggregation; the largest discrepancy from the earlier results is $7.11\times10^{-15}$ percentage points. A separate implementation shares the original record-reading procedure and independently computes every weight, policy and information level for eight original-record cells with zero discrepancy. These cells are the first and last in lexicographic order within each source. The analysis covers 74,250 bounds, 297,000 directed comparisons and 74,250 policy selections, with complete and unique coverage of the planned combinations. The calculations use existing predictions and take 167.29 seconds on CPU. Models and images recur across cells, and the results describe comparisons within the fixed panel.

\subsection{A reproducible real-record illustration}\label{app:worked-case}
The example referenced in Section~\ref{sec:retro-empirical} uses the MAE ViT-L classifier and Diverse print condition 14 at lighting level 1 (500 evaluation images, 200 classes). Its clean threshold is the historical 80\% development-coverage threshold. The selection protocol was fixed before choosing the case: among real panel cells with positive retained-correct support and a shared-correspondence certificate missed by separate intervals at $\lambda=0.5$, choose the cell nearest the median marginal $H$ width over the real panel; break ties by source, model and condition. Table~\ref{tab:worked-intervals} reports its intervals and observed values. The case illustrates the certified ordering and identified intervals for one observed condition.

\begin{table}[!htbp]
\centering
\caption{Sharp intervals and observed values in the real-record illustration. Units are percentage points for $H,L$ and probability points for $dM_k$. Actual correspondence determines the observed column.}
\label{tab:worked-intervals}

\small
\begin{tabular}{lrrrr}\toprule
Target & $\mathcal I_0$ & $\mathcal I_k$ & $\mathcal I_5$ & Observed\\\midrule
$H$ & [0.4, 0.8] & [0.4, 0.8] & [0.4, 0.6] & 0.6\\
$L$ & [17.8, 26.2] & [19.0, 26.0] & [19.0, 26.0] & 23.2\\
$dM_k$ & [19.52, 30.02] & [20.58, 29.43] & [20.58, 29.43] & 26.24\\\bottomrule
\end{tabular}
\end{table}

Choose the first policy pair in index order with mixed outcomes and a strict shared certificate missed by separate intervals. Its nominal clean-development coverages are 70\% and 60\%, with score thresholds 0.669638 and 0.719598. At $\mathcal I_5$, their $J$ intervals are [14.9,18.4] and [16.7,20.3] points. Subtracting separate endpoints gives an upper difference of $+1.7$, whereas optimizing the difference under one common correspondence gives $-1.5$. The 70\%-coverage policy therefore has lower transition cost for every feasible correspondence. This ordering already holds at $\mathcal I_0$; the example demonstrates the shared-correspondence comparison alongside the target-specific narrowing of $H$. Independent classwise enumeration reproduces all nine target intervals and 75 policy-difference bounds from the underlying records.

\FloatBarrier
\section{Screening Protocols and Additional Results}\label{app:screening}
These analyses support Sections~\ref{sec:incrementmethod}--\ref{sec:screening-scope} through score and identity controls, budgeted failure recall, equal-inference probe comparisons, and split and inference-cost checks.

The analyses use the fixed evaluation records, with additional model evaluations for the natural-transformation comparison described below. Unless a table states resampling of the 200 true classes, intervals are percentile intervals from resampling the ten model groups (2,000 draws); both summarize sensitivity within this pool, not population uncertainty. Tiers are the development-defined tiers of Section~\ref{sec:design}.

\subsection{Image-level scores: protocol, resampling, screening and identity ablation}
\paragraph{Scores and sets.}
On clean-correct images, $M_i=S_i$, so $(S,\delta)$ and $(S,U)$ contain identical information. 

For each model, $U_i$ averages the true-class probability $M'_{i,c}$ over the 18 Syn-B conditions $c$ on the 500 real evaluation base images, $\mathrm{ret}_i$ is the share of these conditions on which the image remains correct, and $\delta_i=S_i-U_i$, which on clean-correct images equals the mean clean-minus-corrupted drop. All scores are oriented so that larger values indicate failure ($-S$, $-U$, $\delta$, $-\mathrm{ret}$); the outcome for a real condition is failure among clean-correct evaluation images, and cells in which all such images share one outcome have no AUROC (118 Diverse and 2 ES cells, all but one extreme). Every score is evaluated on the same images of a cell. The class-level baseline assigns to each evaluation image the development-set mean of $M-U$ over all development images of its class. The frozen predictors $g(S)$, $g(S,U)$ and $g(S,U,\mathrm{ret})$ are logistic regressions fitted per model and source on development clean-correct images, pooling observations across the real conditions of that source and standardizing features on the development data (L-BFGS optimization, inverse regularization strength $C=1$). A fit on $(S,\delta)$ under the same standardization and penalty is not algebraically identical to the fit on $(S,U)$, because the columnwise standardization and the $\ell_2$ penalty are not invariant to the linear reparameterization; recomputed on all cells, the two predictors differ in AUROC by at most 0.004 (mean $5.6\times10^{-5}$), so the choice is immaterial here. $g(S)$ reproduces the ranking of $S$ exactly in 87 of the 88 model--source fits; for one model on Diverse the development fit has a positive slope and the ranking reverses, which is reported rather than corrected. Adding $\mathrm{ret}$ to $g(S,U)$ changes mean AUROC by at most 0.001. Within a model, $\delta$ and $-U$ have Spearman correlation between 0.23 and 0.91 across models (median 0.77).

\begin{table}[!htbp]
\centering\setlength{\tabcolsep}{3pt}
\caption{AUROC for ranking real failures among clean-correct images, on identical sets for every score. Class-$\delta$ is the development-set class mean of $M-U$; $g(S,U)$ is fitted on development images. The gap uses 95\% model-group resampling intervals. Class-resampled intervals and cell counts are in Appendix~\ref{app:screening}.}
\label{tab:image-level}

\footnotesize
\begin{tabular}{@{}llrrrrrrl@{}}\toprule
Source & Tier & $S$ & $U$ & $\delta$ & ret & Class-$\delta$ & $g(S,U)$ & $U-S$ [95\%]\\\midrule
ES & mild & 0.778 & 0.854 & 0.768 & 0.823 & 0.580 & 0.855 & +0.076 [+0.070, +0.080]\\
ES & severe & 0.593 & 0.647 & 0.630 & 0.614 & 0.546 & 0.647 & +0.054 [+0.044, +0.063]\\
ES & extreme & 0.538 & 0.578 & 0.572 & 0.560 & 0.535 & 0.576 & +0.040 [+0.027, +0.056]\\
Diverse & mild & 0.678 & 0.801 & 0.781 & 0.755 & 0.587 & 0.809 & +0.123 [+0.110, +0.131]\\
Diverse & severe & 0.601 & 0.700 & 0.698 & 0.656 & 0.565 & 0.710 & +0.099 [+0.084, +0.109]\\
Diverse & extreme & 0.504 & 0.512 & 0.512 & 0.529 & 0.478 & 0.517 & +0.008 [-0.006, +0.024]\\
\bottomrule
\end{tabular}

\end{table}

\FloatBarrier
\paragraph{Two resampling schemes.}
Figure~\ref{fig:screening-auroc} and Table~\ref{tab:image-level} use intervals from resampling the ten model groups. Table~\ref{tab:image-level-classboot} instead resamples the 200 true classes with replacement, applying one class draw to every cell of a source (200 draws); cells whose resampled positives or negatives vanish are skipped in that draw. Both schemes exclude zero for $U-S$ on mild and severe conditions and for $\delta-U$ on mild conditions; neither is a population confidence interval.

\begin{table}[!htbp]
\centering
\caption{Per-cell AUROC differences with 95\% intervals from resampling the 200 true classes (point estimates are the plain cell means; cells is the mean number of cells retained per draw).}
\label{tab:image-level-classboot}

\footnotesize
\begin{tabular}{@{}llrlll@{}}\toprule
Source & Tier & Cells & $U-S$ & $\delta-U$ & $g(S,U)-S$\\\midrule
ES & mild & 1935 & +0.076 [+0.060,+0.093] & -0.086 [-0.107,-0.065] & +0.077 [+0.060,+0.097]\\
ES & severe & 219 & +0.054 [+0.035,+0.070] & -0.016 [-0.028,-0.005] & +0.054 [+0.034,+0.072]\\
ES & extreme & 214 & +0.040 [+0.020,+0.057] & -0.005 [-0.020,+0.011] & +0.038 [+0.017,+0.056]\\
Diverse & mild & 2244 & +0.123 [+0.107,+0.142] & -0.020 [-0.031,-0.007] & +0.131 [+0.110,+0.153]\\
Diverse & severe & 1099 & +0.099 [+0.084,+0.116] & -0.002 [-0.013,+0.009] & +0.108 [+0.091,+0.128]\\
Diverse & extreme & 3174 & +0.008 [-0.013,+0.039] & +0.001 [-0.015,+0.017] & +0.013 [-0.013,+0.052]\\
\bottomrule
\end{tabular}

\end{table}

\FloatBarrier
\paragraph{Screening levels.}
Table~\ref{tab:screening} fixes in advance the fraction $q$ of clean-correct images to recapture and reports the share of a condition's real failures among the images ranked highest by each score (ties broken at random in this table); random selection captures $q$ in expectation. On mild conditions, $U$ captures 49\% (ES) and 26\% (Diverse) of failures at $q=0.1$ and 67\% and 45\% at $q=0.2$, against 41\%/21\% and 58\%/36\% for clean confidence. This evaluation ranks labeled images within the existing pool.

\begin{table}[!htbp]
\centering\setlength{\tabcolsep}{2pt}
\caption{Screening: share of real failures among the top-$q$ fraction of clean-correct images ranked by each score (cells weighted equally within source and tier). Random selection is one seeded draw per cell; its expectation is $q$.}
\label{tab:screening}

\footnotesize
\begin{tabular}{@{}llrrrrrrrrrr@{}}\toprule
& & \multicolumn{5}{c}{$q=0.1$} & \multicolumn{5}{c}{$q=0.2$}\\
Source & Tier & $S$ & $U$ & $\delta$ & ret & random & $S$ & $U$ & $\delta$ & ret & random\\\midrule
ES & mild & 0.406 & 0.489 & 0.395 & 0.487 & 0.096 & 0.578 & 0.672 & 0.568 & 0.648 & 0.199\\
ES & severe & 0.126 & 0.133 & 0.134 & 0.135 & 0.100 & 0.241 & 0.254 & 0.252 & 0.248 & 0.201\\
ES & extreme & 0.104 & 0.106 & 0.106 & 0.106 & 0.100 & 0.208 & 0.210 & 0.209 & 0.208 & 0.200\\
Diverse & mild & 0.208 & 0.260 & 0.271 & 0.271 & 0.099 & 0.357 & 0.449 & 0.452 & 0.441 & 0.200\\
Diverse & severe & 0.130 & 0.144 & 0.150 & 0.147 & 0.100 & 0.248 & 0.274 & 0.280 & 0.270 & 0.199\\
Diverse & extreme & 0.102 & 0.102 & 0.103 & 0.103 & 0.100 & 0.203 & 0.204 & 0.204 & 0.204 & 0.200\\
\bottomrule
\end{tabular}

\end{table}

\FloatBarrier
\paragraph{Screening differences, attainable ceilings and budgets.}
Table~\ref{tab:screening-main} reports budget $q=0.2$; Table~\ref{tab:screening} gives the coverage levels at $q=0.1$ and $q=0.2$ with a random-selection draw, and Table~\ref{tab:ceiling} the failure share $p$ and the attainable ceiling $\min(1,k/n_{\mathrm{fail}})$ with $k=\max(1,\operatorname{round}(qn))$ recaptures. Paired differences use tie-prorated coverage (the expected coverage under random tie-breaking, which matters for the retention rate with its 19 possible values); class-resampled intervals (Table~\ref{tab:screening-boot}) recompute the coverage with resampled class weights and the same proration. On extreme conditions $U$ attains 95\% to 98\% of the ceiling, but so does random selection in expectation (a fraction $p$ of the ceiling: 91\% on ES and 96\% on Diverse, where $U$ reaches 98\%), so this table bounds the room for improvement without establishing retained ranking ability; the AUROC columns of Table~\ref{tab:image-level} carry that information.

\begin{table}[!htbp]
\centering\setlength{\tabcolsep}{2.5pt}
\caption{Failure coverage at a 20\% recapture budget and paired differences with 95\% model-group resampling intervals (ties prorated; cells equally weighted). Random selection covers approximately 20\% in expectation, subject to integer budget rounding. Other budgets and class-resampled intervals are in Appendix~\ref{app:screening}.}
\label{tab:screening-main}

\footnotesize
\begin{tabular}{@{}llrrlll@{}}\toprule
Source & Tier & $S$ & $U$ & $U-S$ [95\%] & $U-\delta$ [95\%] & $U-\mathrm{ret}$ [95\%]\\\midrule
ES & mild & 0.578 & 0.672 & +0.094 [+0.078,+0.106] & +0.104 [+0.079,+0.140] & +0.024 [+0.013,+0.042]\\
ES & severe & 0.241 & 0.254 & +0.013 [+0.011,+0.017] & +0.002 [-0.001,+0.007] & +0.007 [+0.001,+0.016]\\
ES & extreme & 0.208 & 0.210 & +0.002 [+0.002,+0.003] & +0.001 [-0.000,+0.002] & +0.001 [+0.000,+0.004]\\
Diverse & mild & 0.357 & 0.449 & +0.092 [+0.075,+0.109] & -0.003 [-0.009,+0.007] & +0.008 [-0.006,+0.036]\\
Diverse & severe & 0.248 & 0.274 & +0.026 [+0.021,+0.032] & -0.006 [-0.008,-0.005] & +0.004 [-0.002,+0.016]\\
Diverse & extreme & 0.203 & 0.204 & +0.001 [+0.001,+0.002] & -0.000 [-0.001,-0.000] & +0.000 [-0.000,+0.001]\\
\bottomrule
\end{tabular}

\end{table}

\begin{table}[!htbp]
\centering
\caption{Class-resampled 95\% intervals for the paired coverage differences at each budget (200 draws of the 200 true classes, one draw shared by every cell of a source).}
\label{tab:screening-boot}

\footnotesize
\begin{tabular}{@{}llllll@{}}\toprule
Source & Tier & Budget & $U-S$ [classes] & $U-\delta$ [classes] & $U-\mathrm{ret}$ [classes]\\\midrule
ES & mild & 0.1 & [+0.064,+0.111] & [+0.077,+0.116] & [-0.009,+0.013]\\
ES & mild & 0.2 & [+0.075,+0.114] & [+0.076,+0.125] & [+0.008,+0.038]\\
ES & mild & 0.3 & [+0.079,+0.120] & [+0.057,+0.115] & [+0.015,+0.048]\\
ES & severe & 0.1 & [+0.004,+0.011] & [-0.004,+0.004] & [-0.005,+0.001]\\
ES & severe & 0.2 & [+0.006,+0.020] & [-0.003,+0.006] & [+0.001,+0.011]\\
ES & severe & 0.3 & [+0.008,+0.024] & [-0.001,+0.009] & [+0.008,+0.021]\\
ES & extreme & 0.1 & [+0.000,+0.002] & [-0.001,+0.001] & [-0.001,+0.000]\\
ES & extreme & 0.2 & [+0.000,+0.004] & [-0.001,+0.002] & [+0.000,+0.003]\\
ES & extreme & 0.3 & [+0.001,+0.006] & [-0.001,+0.002] & [+0.001,+0.005]\\
Diverse & mild & 0.1 & [+0.040,+0.065] & [-0.021,-0.002] & [-0.017,-0.003]\\
Diverse & mild & 0.2 & [+0.073,+0.113] & [-0.015,+0.008] & [-0.002,+0.020]\\
Diverse & mild & 0.3 & [+0.090,+0.133] & [-0.003,+0.018] & [+0.019,+0.048]\\
Diverse & severe & 0.1 & [+0.010,+0.017] & [-0.008,-0.003] & [-0.005,-0.000]\\
Diverse & severe & 0.2 & [+0.020,+0.034] & [-0.011,-0.002] & [-0.001,+0.010]\\
Diverse & severe & 0.3 & [+0.029,+0.046] & [-0.008,+0.002] & [+0.009,+0.022]\\
Diverse & extreme & 0.1 & [+0.001,+0.001] & [-0.001,+0.000] & [-0.000,+0.000]\\
Diverse & extreme & 0.2 & [+0.001,+0.002] & [-0.001,+0.000] & [-0.000,+0.001]\\
Diverse & extreme & 0.3 & [+0.001,+0.003] & [-0.001,+0.000] & [+0.000,+0.003]\\
\bottomrule
\end{tabular}

\end{table}

\begin{table}[!htbp]
\centering\setlength{\tabcolsep}{3pt}
\caption{Failure share $p$ among clean-correct images, attainable coverage ceiling $\min(1,k/n_{\mathrm{fail}})$ for budgets $q$, and the coverage of $U$ as a fraction of the ceiling. Each column averages cellwise quantities equally within source and tier; the last is the mean of cellwise ratios, not the ratio of the displayed means.}
\label{tab:ceiling}

\footnotesize
\begin{tabular}{@{}llrrrrrr@{}}\toprule
Source & Tier & $p$ & \multicolumn{3}{c}{attainable ceiling} & \multicolumn{2}{c}{$U$ / ceiling}\\
& & & $q{=}0.1$ & $q{=}0.2$ & $q{=}0.3$ & $q{=}0.1$ & $q{=}0.2$\\\midrule
ES & mild & 0.167 & 0.748 & 0.883 & 0.941 & 0.676 & 0.760\\
ES & severe & 0.675 & 0.160 & 0.319 & 0.478 & 0.864 & 0.830\\
ES & extreme & 0.908 & 0.112 & 0.223 & 0.335 & 0.956 & 0.947\\
Diverse & mild & 0.360 & 0.373 & 0.643 & 0.815 & 0.777 & 0.739\\
Diverse & severe & 0.652 & 0.170 & 0.339 & 0.500 & 0.891 & 0.853\\
Diverse & extreme & 0.964 & 0.104 & 0.209 & 0.314 & 0.984 & 0.980\\
\bottomrule
\end{tabular}

\end{table}

\FloatBarrier
\paragraph{Protocol sensitivity without re-inference.}
All stored probabilities use one scoring protocol (a fixed multiplier of 100 and categorical softmax for OpenCLIP models). The stored records contain no logits, so a first check uses only rank information; the native scoring rules are recomputed from the re-inference described later in this appendix. As the rank check, $U_{\mathrm{rank}}$ replaces each condition's true-class probability by its percentile rank among the clean-correct images before averaging, which is invariant to any strictly monotone within-condition rescaling, and $U^{\gamma}$ averages $(M')^{\gamma}$ for $\gamma\in\{0.5,2\}$. Table~\ref{tab:protocol-sensitivity} shows that $U_{\mathrm{rank}}$ remains above clean confidence by 0.05 (ES) and 0.08 (Diverse) on mild conditions, up to 0.044 below $U$, and that the power transforms move $U$ by at most 0.007. The direction of the result does not depend on the probability scale; its magnitude does.

\begin{table}[!htbp]
\centering
\caption{Protocol sensitivity: AUROC of $U$, of its rank-based version and of power-transformed versions, with the $U_{\mathrm{rank}}-S$ gap and a 95\% interval from resampling model groups.}
\label{tab:protocol-sensitivity}

\footnotesize
\begin{tabular}{@{}llrrrrrl@{}}\toprule
Source & Tier & $S$ & $U$ & $U_{\mathrm{rank}}$ & $U^{0.5}$ & $U^{2}$ & $U_{\mathrm{rank}}-S$ [groups]\\\midrule
ES & mild & 0.778 & 0.854 & 0.832 & 0.854 & 0.851 & +0.054 [+0.049,+0.058]\\
ES & severe & 0.593 & 0.647 & 0.628 & 0.649 & 0.644 & +0.035 [+0.027,+0.044]\\
ES & extreme & 0.538 & 0.578 & 0.565 & 0.580 & 0.575 & +0.028 [+0.016,+0.045]\\
Diverse & mild & 0.678 & 0.801 & 0.757 & 0.805 & 0.794 & +0.079 [+0.069,+0.086]\\
Diverse & severe & 0.601 & 0.700 & 0.662 & 0.704 & 0.695 & +0.061 [+0.052,+0.069]\\
Diverse & extreme & 0.504 & 0.512 & 0.506 & 0.511 & 0.512 & +0.003 [-0.008,+0.016]\\
\bottomrule
\end{tabular}

\end{table}

\FloatBarrier
\paragraph{Check on the other half of the image pool.}
The image-level comparisons were explored on the 500 evaluation base images. Table~\ref{tab:dev-confirmation} repeats them on the 500 development images, which were used only for thresholds and calibrators, with the fitting roles exchanged (class-level baseline and $g$ fitted on evaluation images) and no other change. The tier labels are those defined from development-set accuracy (Section~\ref{sec:design}), so on this half they were fixed using the same images that are evaluated; regrouping the same cells by tiers defined from evaluation-set accuracy (per-model correct counts over evaluation images, equal-weight mean, cuts 40\% and 20\%) moves 5 Diverse conditions from severe to mild (evaluation pool means 40.1\% to 41.7\%) and leaves the mild-tier gaps essentially unchanged: $U-S$ AUROC +0.128 and coverage +0.088 on Diverse against +0.128 and +0.093 under the development-defined tiers, with ES unchanged (Table~\ref{tab:retier}). The gaps reproduce: $U-S$ in AUROC is $+0.078$ and $+0.128$ on mild ES and Diverse, and the budget-0.2 coverage difference is $+0.111$ and $+0.093$. This is the other half of the same pool with the same models and capture settings, not a new collection; on this half $\delta$ slightly exceeds $U$ on severe and extreme conditions.

\begin{table}[!htbp]
\centering\setlength{\tabcolsep}{2pt}
\caption{Check on the development images with fitting roles exchanged: AUROC of each score, the $U-S$ gap, and the budget-0.2 coverage difference, with 95\% intervals from resampling model groups (class-resampled intervals also exclude zero on mild and severe conditions).}
\label{tab:dev-confirmation}

\footnotesize
\begin{tabular}{@{}llrrrrrll@{}}\toprule
Source & Tier & $S$ & $U$ & $\delta$ & ret & $g(S{,}U)$ & AUROC $U{-}S$ & Cov.\ $U{-}S$\\\midrule
ES & mild & 0.778 & 0.856 & 0.779 & 0.831 & 0.856 & +0.078 [+0.072,+0.083] & +0.111 [+0.097,+0.129]\\
ES & severe & 0.572 & 0.641 & 0.646 & 0.608 & 0.644 & +0.069 [+0.059,+0.075] & +0.012 [+0.010,+0.015]\\
ES & extreme & 0.504 & 0.560 & 0.577 & 0.553 & 0.563 & +0.056 [+0.034,+0.072] & +0.003 [+0.002,+0.003]\\
Diverse & mild & 0.673 & 0.801 & 0.790 & 0.754 & 0.815 & +0.128 [+0.119,+0.135] & +0.093 [+0.075,+0.120]\\
Diverse & severe & 0.588 & 0.701 & 0.709 & 0.663 & 0.718 & +0.113 [+0.096,+0.128] & +0.026 [+0.020,+0.036]\\
Diverse & extreme & 0.475 & 0.507 & 0.536 & 0.518 & 0.520 & +0.032 [+0.019,+0.049] & +0.001 [+0.001,+0.002]\\
\bottomrule
\end{tabular}

\end{table}

\begin{table}[!htbp]
\centering\setlength{\tabcolsep}{3pt}
\caption{Swapped-split confirmation regrouped by tiers defined from evaluation-set accuracy (development images evaluated; fitting on evaluation images). AUROC of $S$ and $U$, and the $U-S$ gaps in AUROC and in coverage at $q=0.2$, with 95\% intervals from resampling model groups.}
\label{tab:retier}

\footnotesize
\begin{tabular}{@{}llrrrll@{}}\toprule
Source & Tier (eval-defined) & Conditions & $S$ & $U$ & AUROC $U{-}S$ [95\%] & Cov.\ $U{-}S$ [95\%]\\\midrule
ES & mild & 44 & 0.778 & 0.856 & +0.078 [+0.072,+0.083] & +0.111 [+0.097,+0.129]\\
ES & severe & 5 & 0.572 & 0.641 & +0.069 [+0.059,+0.075] & +0.012 [+0.010,+0.015]\\
ES & extreme & 5 & 0.504 & 0.560 & +0.056 [+0.034,+0.072] & +0.003 [+0.002,+0.003]\\
Diverse & mild & 56 & 0.667 & 0.794 & +0.128 [+0.118,+0.135] & +0.088 [+0.071,+0.114]\\
Diverse & severe & 20 & 0.583 & 0.695 & +0.112 [+0.093,+0.127] & +0.023 [+0.018,+0.032]\\
Diverse & extreme & 86 & 0.475 & 0.507 & +0.032 [+0.019,+0.049] & +0.001 [+0.001,+0.002]\\
\bottomrule
\end{tabular}

\end{table}

\FloatBarrier
\paragraph{Identity ablation on $U$ and a class-level baseline matched to $U$.}
Table~\ref{tab:identity-ablation} reassigns $U_i$ among clean-correct images of the same class with no fixed points (a derangement; ten seeded draws, AUROC averaged) or replaces it by the leave-one-out class mean, and compares with $U_i$ on the same non-singleton subset; classes with a single clean-correct image have no control value and are excluded from this comparison (10 to 55 images per model). A uniform within-class permutation is not used because, with two or three images per class, it would leave about 40\% of values in place. These controls answer what happens when an image receives another same-class image's response; they are not class-level predictors. The leave-one-out mean decreases in the image's own value and with two images per class it swaps the two, so the gap is an identity ablation, not the increment over class-level information. On mild conditions the derangement control falls to AUROC 0.55 on ES and 0.58 on Diverse against 0.85 and 0.80 for $U_i$, and $U_i$ exceeds its control in 99.9\% and 100.0\% of cells; the same ablation applied to $\delta_i$ gives the same ordering. On extreme Diverse conditions the control degenerates in 14\% of cells (all retained-correct images of the cell lie in one class, so the derangement only exchanges values among them and equals $U_i$ in AUROC), which weakens the contrast there; excluding those cells moves the gap from $+0.047$ to $+0.055$. Table~\ref{tab:class-baseline-U} evaluates development-fitted class-level predictors: the development-set class mean of $U$ over clean-correct development images of the class, and over all development images, both evaluated on the images whose class has such a value. They reach AUROC 0.60 and 0.60, below clean confidence (0.78 and 0.68), and cover 0.30 and 0.27 of failures at the 20\% budget against 0.67 and 0.45 for $U$. Table~\ref{tab:image-level-strata} stratifies by terciles of clean confidence within model: $U_i$ retains AUROC 0.63 to 0.77 inside each stratum, where clean confidence retains much less separation (0.52 to 0.68), and images with above-median $U_i$ have higher retention rates than those below the median within every stratum.

\begin{table}[!htbp]
\centering\setlength{\tabcolsep}{3.5pt}
\caption{Identity ablation: AUROC of clean confidence $S$, of $U_i$, of its within-class derangement control and of its leave-one-out class mean for ranking failing against retained-correct clean-correct images (failure as the positive class), all on the same non-singleton subset; gap is $U_i$ minus the derangement control with a 95\% interval from resampling the ten model groups; wins is the share of cells in which $U_i$ exceeds its control.}
\label{tab:identity-ablation}

\small
\begin{tabular}{@{}llrrrrrlr@{}}\toprule
Source & Tier & Cells & $S$ & $U$ & Derangement & LOO mean & Gap [95\%] & Wins (\%)\\\midrule
ES & mild & 1936 & 0.775 & 0.853 & 0.553 & 0.548 & +0.299 [+0.275,+0.319] & 99.9\\
ES & severe & 218 & 0.594 & 0.647 & 0.523 & 0.519 & +0.124 [+0.119,+0.130] & 99.1\\
ES & extreme & 216 & 0.543 & 0.581 & 0.504 & 0.504 & +0.077 [+0.058,+0.103] & 80.5\\
Diverse & mild & 2244 & 0.675 & 0.800 & 0.581 & 0.582 & +0.219 [+0.205,+0.230] & 100.0\\
Diverse & severe & 1099 & 0.599 & 0.699 & 0.537 & 0.536 & +0.162 [+0.148,+0.171] & 99.2\\
Diverse & extreme & 3485 & 0.498 & 0.508 & 0.461 & 0.453 & +0.047 [+0.038,+0.059] & 56.0\\
\bottomrule
\end{tabular}

\end{table}

\begin{table}[!htbp]
\centering\setlength{\tabcolsep}{2.5pt}
\caption{Class-level baselines matched to $U$: AUROC of $S$, $U$ and the development-set class mean of $U$ over clean-correct development images (cc.\ dev) or over all development images (all dev), on the images whose class has such a value; the last two columns give the $U$ minus clean-correct development class-mean gap in AUROC and in coverage at $q=0.2$ with 95\% intervals from resampling model groups.}
\label{tab:class-baseline-U}

\footnotesize
\begin{tabular}{@{}llrrrrll@{}}\toprule
Source & Tier & $S$ & $U$ & \shortstack{Class\\(cc.\ dev)} & \shortstack{Class\\(all dev)} & \shortstack{$U-$Class\\{[95\%]}} & \shortstack{Cov.\ $U-$Class\\{[95\%]}}\\\midrule
ES & mild & 0.778 & 0.854 & 0.597 & 0.590 & +0.257 [+0.238,+0.277] & +0.374 [+0.332,+0.422]\\
ES & severe & 0.593 & 0.646 & 0.557 & 0.543 & +0.089 [+0.080,+0.100] & +0.037 [+0.027,+0.052]\\
ES & extreme & 0.540 & 0.578 & 0.545 & 0.541 & +0.033 [+0.012,+0.050] & +0.006 [+0.005,+0.008]\\
Diverse & mild & 0.678 & 0.801 & 0.604 & 0.596 & +0.197 [+0.185,+0.206] & +0.177 [+0.148,+0.213]\\
Diverse & severe & 0.601 & 0.701 & 0.567 & 0.554 & +0.134 [+0.122,+0.143] & +0.051 [+0.041,+0.063]\\
Diverse & extreme & 0.502 & 0.510 & 0.449 & 0.419 & +0.061 [+0.036,+0.093] & +0.003 [+0.003,+0.004]\\
\bottomrule
\end{tabular}

\end{table}

\begin{table}[!htbp]
\centering\setlength{\tabcolsep}{4pt}
\caption{Image-level separation within clean-confidence terciles (all conditions of each source, cells weighted equally; failure as the positive class). Gap is $U_i$ minus its derangement control on the non-singleton subset with a 95\% interval from resampling the ten model groups; the last column is the retention rate of images with above-median $U_i$ minus that of images below the median within the stratum.}
\label{tab:image-level-strata}

\small
\begin{tabular}{@{}llrrrlr@{}}\toprule
Source & Stratum & $S$ & $U$ & Derangement & Gap [95\%] & Retention gap\\\midrule
ES & low & 0.683 & 0.771 & 0.517 & +0.255 [+0.235,+0.274] & +0.225\\
ES & mid & 0.548 & 0.749 & 0.455 & +0.294 [+0.255,+0.332] & +0.100\\
ES & high & 0.530 & 0.755 & 0.460 & +0.295 [+0.252,+0.325] & +0.075\\
Diverse & low & 0.581 & 0.676 & 0.541 & +0.136 [+0.110,+0.156] & +0.150\\
Diverse & mid & 0.527 & 0.633 & 0.490 & +0.144 [+0.117,+0.170] & +0.127\\
Diverse & high & 0.517 & 0.646 & 0.507 & +0.138 [+0.122,+0.158] & +0.116\\
\bottomrule
\end{tabular}

\end{table}

\FloatBarrier
\paragraph{Number of synthetic conditions.}
Table~\ref{tab:cost-curve} recomputes $U$ from subsets of the 18 Syn-B conditions fixed before any result was seen (seed-drawn, up to 20 distinct subsets per size; all 18 singletons for $m=1$), plus the nine severity-2 and the nine severity-4 conditions. No subset was selected on outcomes; the table reports performance for each tested subset size. With six random conditions the mean coverage gain over clean confidence is 81\% (ES) and 92\% (Diverse) of the 18-condition gain, with minima of 67\% and 68\% over the prefixed draws; the worst three-condition subset still retains 51\% and 41\%; a single condition can fall below clean confidence on ES (worst gain -0.009). The severity-4 conditions alone reach coverage 0.669 and 0.461 against 0.672 and 0.449 for all 18, and the severity-2 conditions 0.658 and 0.416.

\begin{table}[!htbp]
\centering\setlength{\tabcolsep}{2pt}
\caption{Coverage at $q=0.2$, coverage gain over $S$ and AUROC of $U_m$ computed from $m$ of the 18 Syn-B conditions on mild conditions (cells weighted equally); for random subsets the mean over the prefixed subsets is followed by the minimum and maximum over subsets (minimum only for the gain).}
\label{tab:cost-curve}

\footnotesize
\begin{tabular}{@{}lrlrllll@{}}\toprule
Source & $m$ & Subsets & $n$ & Coverage [min,max] & $U_m-S$ [min] & AUROC [min,max]\\\midrule
ES & 1 & random & 18 & 0.597 [0.569,0.619] & +0.019 [-0.009] & 0.799 [0.785,0.816]\\
ES & 2 & random & 20 & 0.625 [0.609,0.638] & +0.047 [+0.030] & 0.818 [0.801,0.833]\\
ES & 3 & random & 20 & 0.640 [0.626,0.652] & +0.062 [+0.048] & 0.830 [0.818,0.844]\\
ES & 6 & random & 20 & 0.654 [0.641,0.666] & +0.076 [+0.063] & 0.840 [0.826,0.852]\\
ES & 9 & random & 20 & 0.663 [0.655,0.669] & +0.085 [+0.077] & 0.849 [0.837,0.854]\\
ES & 9 & severity 2 only & 1 & 0.658 & +0.080 & 0.837\\
ES & 9 & severity 4 only & 1 & 0.669 & +0.091 & 0.854\\
ES & 18 & all 18 & 1 & 0.672 & +0.094 & 0.854\\
Diverse & 1 & random & 18 & 0.399 [0.363,0.449] & +0.042 [+0.005] & 0.730 [0.685,0.792]\\
Diverse & 2 & random & 20 & 0.419 [0.380,0.465] & +0.061 [+0.023] & 0.753 [0.711,0.809]\\
Diverse & 3 & random & 20 & 0.429 [0.395,0.459] & +0.071 [+0.037] & 0.766 [0.726,0.807]\\
Diverse & 6 & random & 20 & 0.442 [0.419,0.457] & +0.084 [+0.062] & 0.784 [0.759,0.804]\\
Diverse & 9 & random & 20 & 0.444 [0.427,0.453] & +0.087 [+0.070] & 0.792 [0.775,0.803]\\
Diverse & 9 & severity 2 only & 1 & 0.416 & +0.059 & 0.760\\
Diverse & 9 & severity 4 only & 1 & 0.461 & +0.103 & 0.811\\
Diverse & 18 & all 18 & 1 & 0.449 & +0.092 & 0.801\\
\bottomrule
\end{tabular}

\end{table}

\FloatBarrier
\paragraph{Natural test-time augmentation, native scoring rules and consistency of the re-inference.}
The comparison of Section~\ref{sec:tta} follows a protocol specified before its model evaluations. The fixed set of 18 natural transformations is horizontal flip; shifts of $\pm2\%$ and $\pm5\%$ of the width or height in each axis (eight); zoom 1.1 and 1.2 (central crop, resized back); rotations of $\pm5^\circ$ and $\pm10^\circ$; gamma 0.8 and 0.9; flip combined with zoom 1.1; shifts and rotations use reflect padding and every transformation is applied to the original image before the model's own preprocessing. The candidate pool of 30 adds zoom-out to 90\%, gamma 0.6 and 1.25, flip combined with $\pm10^\circ$ rotation, brightness 0.9 and 1.1, contrast 0.9 and 1.1, saturation 0.8 and 1.2, flip combined with a 2\% shift. $\mathrm{TTA}_{18}$ averages the true-class probability over the 18 transformed versions of a clean-correct image, so the class is fixed to the clean prediction and neither the prediction nor the image set is changed by the augmentation; the versions with the untransformed image added as a nineteenth term follow \citet{bahat2020tta}. The development-selected variant $\mathrm{TTA}^{\mathrm{sel}}$ adds transformations from the pool one at a time, each time taking the one that maximizes the mean coverage at $q=0.2$ over all models and conditions of the source on the development images with that source's real failure labels, until 18 are chosen; for the development-half report the selection is redone on evaluation images. Zoom-out was chosen first in three of the four selections (rotation by $+10^\circ$ in the fourth), and the selected sets share 11 to 14 members with the fixed set; the analogous selection of 18 of the 27 Syn-B conditions, denoted $U^{\mathrm{sel}}$, changes $U$ by +0.003 on ES (interval including zero) and +0.025 on Diverse, mostly by admitting severity-5 conditions. Table~\ref{tab:tta-variants} reports the paired coverage differences for these variants at three budgets, Table~\ref{tab:tta-dev} the development-half confirmation, and Table~\ref{tab:tta-curve} the coverage of both averages as a function of the number of versions with the prefixed subsets of Table~\ref{tab:cost-curve} applied by index to both lists.

Every score in this comparison is derived from a common evaluation of the 44 models on the 1,000 base images under 49 input variants: the clean image, 18 corruptions and 30 natural transformations. Model loading, preprocessing, mixed precision and scoring follow the original evaluation procedure. Single-precision cosine similarities for OpenCLIP models, 1,000-class logits for supervised and linear-head models, and learned logit scales and SigLIP biases support the alternative scoring calculations. Three agreement tolerances were specified in advance: at least 99\% prediction agreement, mean absolute probability difference at most 0.01, and clean accuracy difference at most 0.5 percentage points. The repeated evaluation reproduces the recorded clean and corruption class predictions and probabilities exactly for all 44 models and all 19 conditions (Table~\ref{tab:tta-consistency}); the resulting $U$ and $S$ AUROCs agree with Table~\ref{tab:image-level} to $10^{-16}$. This confirms computational consistency on the same GPU model using the original procedure and the documented contemporary library versions. Scores recomputed from the single-precision outputs differ from the recorded half-precision probabilities by at most 0.009 and change the predicted class for at most 0.8\% of images near ties. The common protocol applies softmax to cosine similarities multiplied by 100 for OpenCLIP models and to the 200-class logit subspace for other models. Native scoring uses the learned scale (14 to 117), sigmoid probabilities with the learned scale and bias for SigLIP, and full 1,000-class softmax probabilities for supervised and linear-head models. The text classifier, candidate-space predictions and eligible images remain fixed across scoring rules. Under native scoring, mean mild-condition AUROC decreases by about 0.02 for clean confidence, $U_{18}$ and $\mathrm{TTA}_{18}$ (-0.020, -0.022 and -0.022 on ES), while their gaps remain unchanged. At the primary 20\% recapture budget, the native-rule coverage gaps between $U_{18}$ and $\mathrm{TTA}_{18}$ are 0.073 on ES and 0.072 on Diverse.

\begin{table}[!htbp]
\centering\setlength{\tabcolsep}{2pt}
\caption{Paired coverage differences for the variants of the natural-augmentation comparison at budgets $q\in\{0.1,0.2,0.3\}$ (95\% intervals from resampling model groups): both averages with the untransformed image added; both under the native scoring rule; $U_{18}$ against the development-selected natural set; the development-selected against the fixed natural set (the development-selected Syn-B set against the historical $U$ is given in the text).}
\label{tab:tta-variants}

\scriptsize
\begin{tabular}{@{}llrllll@{}}\toprule
Source & Tier & $q$ & $U^{+0}-\mathrm{TTA}^{+0}$ & native $U-\mathrm{TTA}$ & $U_{18}-\mathrm{TTA}^{\mathrm{sel}}$ & $\mathrm{TTA}^{\mathrm{sel}}-\mathrm{TTA}_{18}$\\\midrule
ES & mild & 0.1 & +0.057 [+0.046,+0.074] & +0.060 [+0.046,+0.082] & +0.054 [+0.043,+0.069] & +0.002 [-0.002,+0.007]\\
ES & mild & 0.2 & +0.076 [+0.062,+0.086] & +0.073 [+0.061,+0.082] & +0.072 [+0.060,+0.081] & +0.004 [+0.002,+0.007]\\
ES & mild & 0.3 & +0.078 [+0.062,+0.090] & +0.076 [+0.062,+0.088] & +0.076 [+0.061,+0.089] & +0.002 [+0.001,+0.003]\\
ES & severe & 0.1 & +0.007 [+0.006,+0.009] & +0.006 [+0.005,+0.007] & +0.007 [+0.005,+0.009] & +0.000 [-0.001,+0.001]\\
ES & severe & 0.2 & +0.012 [+0.009,+0.016] & +0.010 [+0.006,+0.013] & +0.012 [+0.009,+0.015] & +0.001 [+0.000,+0.002]\\
ES & severe & 0.3 & +0.017 [+0.012,+0.021] & +0.014 [+0.008,+0.019] & +0.017 [+0.011,+0.021] & +0.001 [-0.001,+0.002]\\
ES & extreme & 0.1 & +0.001 [+0.001,+0.002] & +0.001 [+0.001,+0.002] & +0.001 [+0.001,+0.002] & -0.000 [-0.000,+0.000]\\
ES & extreme & 0.2 & +0.002 [+0.002,+0.003] & +0.002 [+0.001,+0.003] & +0.002 [+0.001,+0.003] & +0.000 [-0.000,+0.001]\\
ES & extreme & 0.3 & +0.003 [+0.002,+0.004] & +0.003 [+0.001,+0.004] & +0.003 [+0.002,+0.004] & +0.000 [-0.000,+0.001]\\
Diverse & mild & 0.1 & +0.045 [+0.036,+0.052] & +0.040 [+0.031,+0.048] & +0.044 [+0.035,+0.051] & +0.002 [+0.001,+0.004]\\
Diverse & mild & 0.2 & +0.080 [+0.064,+0.095] & +0.072 [+0.052,+0.088] & +0.082 [+0.066,+0.096] & +0.001 [-0.001,+0.003]\\
Diverse & mild & 0.3 & +0.104 [+0.082,+0.121] & +0.090 [+0.064,+0.112] & +0.104 [+0.083,+0.121] & +0.001 [-0.000,+0.003]\\
Diverse & severe & 0.1 & +0.012 [+0.009,+0.014] & +0.010 [+0.008,+0.012] & +0.012 [+0.009,+0.014] & +0.000 [+0.000,+0.001]\\
Diverse & severe & 0.2 & +0.023 [+0.018,+0.027] & +0.021 [+0.015,+0.026] & +0.023 [+0.019,+0.028] & +0.000 [-0.001,+0.001]\\
Diverse & severe & 0.3 & +0.035 [+0.028,+0.041] & +0.030 [+0.020,+0.038] & +0.035 [+0.028,+0.041] & +0.001 [-0.000,+0.001]\\
Diverse & extreme & 0.1 & +0.001 [+0.000,+0.001] & +0.001 [+0.000,+0.001] & +0.001 [+0.000,+0.001] & -0.000 [-0.000,+0.000]\\
Diverse & extreme & 0.2 & +0.001 [+0.001,+0.002] & +0.001 [+0.001,+0.002] & +0.001 [+0.001,+0.002] & -0.000 [-0.000,+0.000]\\
Diverse & extreme & 0.3 & +0.002 [+0.001,+0.002] & +0.002 [+0.001,+0.002] & +0.002 [+0.001,+0.002] & +0.000 [-0.000,+0.000]\\
\bottomrule
\end{tabular}

\end{table}

\begin{table}[!htbp]
\centering\setlength{\tabcolsep}{3pt}
\caption{Agreement between the re-inference and the historical clean and corruption records, by model implementation (minimum prediction agreement, maximum mean absolute probability difference and clean accuracy difference over models and conditions; range of learned logit scales), and the number of models meeting the prespecified agreement tolerances.}
\label{tab:tta-consistency}

\footnotesize
\begin{tabular}{@{}lrrrrrr@{}}\toprule
Implementation & Models & \shortstack{Pred.\ agreement\\(min)} & \shortstack{Mean $|\Delta M|$\\(max)} & \shortstack{Clean acc.\ diff\\(pp, max)} & \shortstack{Scale\\(range)} & Within tolerances\\\midrule
DINOv2 & 1 & 1.0000 & 0.0000 & 0.00 & -- & 1/1\\
DINOv3 & 1 & 1.0000 & 0.0000 & 0.00 & -- & 1/1\\
OpenCLIP & 35 & 1.0000 & 0.0000 & 0.00 & 14.3--117.3 & 35/35\\
timm & 6 & 1.0000 & 0.0000 & 0.00 & -- & 6/6\\
torchvision & 1 & 1.0000 & 0.0000 & 0.00 & -- & 1/1\\
\bottomrule
\end{tabular}

\end{table}

\begin{table}[!htbp]
\centering\setlength{\tabcolsep}{2.2pt}
\caption{Natural-augmentation comparison on the development half of the image pool with the selection of $\mathrm{TTA}^{\mathrm{sel}}$ redone on evaluation images (coverage at $q=0.2$; intervals from resampling model groups and classes).}
\label{tab:tta-dev}

\scriptsize
\begin{tabular}{@{}llrrrllll@{}}\toprule
Source & Tier & $S$ & $U_{18}$ & $\mathrm{TTA}_{18}$ & $U_{18}-\mathrm{TTA}_{18}$ [groups] & [classes] & $\Delta$ AUROC [groups] & $U_{18}-\mathrm{TTA}^{\mathrm{sel}}$\\\midrule
ES & mild & 0.561 & 0.671 & 0.583 & +0.088 [+0.076,+0.103] & [+0.058,+0.125] & +0.065 [+0.060,+0.069] & +0.085 [+0.073,+0.100]\\
ES & severe & 0.228 & 0.241 & 0.229 & +0.012 [+0.009,+0.015] & [+0.006,+0.018] & +0.065 [+0.055,+0.071] & +0.011 [+0.009,+0.015]\\
ES & extreme & 0.204 & 0.207 & 0.204 & +0.002 [+0.002,+0.003] & [+0.001,+0.004] & +0.057 [+0.037,+0.072] & +0.002 [+0.002,+0.003]\\
Diverse & mild & 0.334 & 0.428 & 0.345 & +0.082 [+0.066,+0.109] & [+0.069,+0.100] & +0.117 [+0.108,+0.123] & +0.080 [+0.064,+0.105]\\
Diverse & severe & 0.239 & 0.265 & 0.242 & +0.024 [+0.018,+0.033] & [+0.019,+0.030] & +0.104 [+0.087,+0.118] & +0.024 [+0.018,+0.033]\\
Diverse & extreme & 0.202 & 0.203 & 0.202 & +0.001 [+0.001,+0.002] & [+0.001,+0.002] & +0.035 [+0.021,+0.049] & +0.001 [+0.001,+0.002]\\
\bottomrule
\end{tabular}

\end{table}

\begin{table}[!htbp]
\centering\setlength{\tabcolsep}{2.5pt}
\caption{Coverage at $q=0.2$, gain over $S$ and AUROC of the corruption average and of the natural-augmentation average computed from $m$ versions, mild conditions, with the prefixed subsets of Table~\ref{tab:cost-curve} applied by index to both lists (mean over subsets followed by the minimum and maximum; minimum only for the gain).}
\label{tab:tta-curve}

\scriptsize
\begin{tabular}{@{}llrrllll@{}}\toprule
Source & Pool & $m$ & Subsets & Coverage: mean [min, max] & Gain over $S$: mean [min] & AUROC: mean [min, max]\\\midrule
ES & $U$ & 1 & 18 & 0.597 [0.569,0.619] & +0.019 [-0.009] & 0.799 [0.785,0.816]\\
ES & $U$ & 2 & 20 & 0.625 [0.609,0.638] & +0.047 [+0.030] & 0.818 [0.802,0.833]\\
ES & $U$ & 3 & 20 & 0.640 [0.626,0.652] & +0.062 [+0.048] & 0.830 [0.818,0.844]\\
ES & $U$ & 6 & 20 & 0.654 [0.641,0.666] & +0.076 [+0.063] & 0.840 [0.826,0.852]\\
ES & $U$ & 9 & 20 & 0.663 [0.655,0.669] & +0.085 [+0.077] & 0.849 [0.837,0.854]\\
ES & $U$ & 18 & 1 & 0.672 & +0.094 & 0.854\\
ES & TTA & 1 & 18 & 0.574 [0.553,0.585] & -0.004 [-0.025] & 0.777 [0.767,0.783]\\
ES & TTA & 2 & 20 & 0.584 [0.575,0.592] & +0.006 [-0.003] & 0.783 [0.776,0.788]\\
ES & TTA & 3 & 20 & 0.590 [0.584,0.597] & +0.012 [+0.005] & 0.787 [0.783,0.789]\\
ES & TTA & 6 & 20 & 0.594 [0.589,0.597] & +0.015 [+0.011] & 0.789 [0.785,0.791]\\
ES & TTA & 9 & 20 & 0.596 [0.592,0.598] & +0.017 [+0.014] & 0.789 [0.787,0.791]\\
ES & TTA & 18 & 1 & 0.596 & +0.018 & 0.791\\
Diverse & $U$ & 1 & 18 & 0.399 [0.363,0.449] & +0.042 [+0.006] & 0.730 [0.685,0.792]\\
Diverse & $U$ & 2 & 20 & 0.419 [0.380,0.465] & +0.061 [+0.023] & 0.753 [0.711,0.809]\\
Diverse & $U$ & 3 & 20 & 0.429 [0.395,0.459] & +0.071 [+0.038] & 0.766 [0.726,0.807]\\
Diverse & $U$ & 6 & 20 & 0.442 [0.419,0.457] & +0.084 [+0.062] & 0.784 [0.759,0.804]\\
Diverse & $U$ & 9 & 20 & 0.444 [0.428,0.454] & +0.087 [+0.070] & 0.792 [0.775,0.803]\\
Diverse & $U$ & 18 & 1 & 0.449 & +0.092 & 0.801\\
Diverse & TTA & 1 & 18 & 0.358 [0.351,0.363] & +0.001 [-0.006] & 0.679 [0.675,0.684]\\
Diverse & TTA & 2 & 20 & 0.361 [0.356,0.365] & +0.003 [-0.002] & 0.682 [0.678,0.686]\\
Diverse & TTA & 3 & 20 & 0.363 [0.357,0.366] & +0.006 [+0.000] & 0.685 [0.680,0.688]\\
Diverse & TTA & 6 & 20 & 0.364 [0.361,0.366] & +0.007 [+0.004] & 0.686 [0.684,0.689]\\
Diverse & TTA & 9 & 20 & 0.365 [0.363,0.368] & +0.008 [+0.006] & 0.687 [0.685,0.689]\\
Diverse & TTA & 18 & 1 & 0.366 & +0.009 & 0.688\\
\bottomrule
\end{tabular}

\end{table}

\FloatBarrier
\subsection{A related augmentation-alignment ranking score}\label{app:alignment-baseline}
We evaluate the scoring equations of A3Rank \citep{wei2024a3rank} on the fixed natural and corruption probe sets, separately. This adds one directly related ranking construction while keeping the same 18 probe evaluations and common clean prediction. Its score is evaluated for subsequent recapture failures, whereas the original work targets errors on the scored input and includes a downstream rejection component. We retain the score's sums and three roles, but do not train its rejector or reproduce its model-specific training augmentations. The original score has no tuned weights: its alignment terms are summed with equal weight, and its only model-specific ingredient is the variant set drawn from the classifier's training augmentations, which we replace by the fixed probe sets. PRIMA's learned ranking with input and model mutations uses different resources \citep{wang2021prima}; no performance comparison with that pipeline is claimed here.

Let $p$ be the clean predicted class and $m$ the modal predicted class across the 18 probes. Write $v_a(c)$ for probe $a$'s probability of class $c$ and $v_0(c)$ for the clean probability. Define distractors $\mathcal D=\{a:\arg\max_c v_a(c)\notin\{p,m\}\}$, supporters $\mathcal S=\{a:\arg\max_c v_a(c)=p\}$ and dominators $\mathcal G=\{a:\arg\max_c v_a(c)=m\}$. The adapted score is
\begin{align}
A^3={}&v_0(p)-\sum_{a\in\mathcal D}\bigl[\max_c v_a(c)-v_a(p)\bigr]
 +\sum_{a\in\mathcal S}\bigl[v_a(p)-v_a(m)\bigr]\nonumber\\
 &+\sum_{a\in\mathcal G}\bigl[v_a(p)-v_a(m)+v_0(m)-v_0(p)\bigr].
\end{align}
Low $A^3$ indicates high failure risk. When $p=m$, the supporter and dominator contributions vanish even though their index sets coincide. Modal-count ties are resolved by the mean maximum probability among probes voting for each tied class, then by class index. This specifies a deterministic version of the original confidence-based tie rule.

The calculation uses the common 200-class probability protocol and historical clean-correct eligibility, without real-outcome fitting or score selection. For four eligible model--image pairs, recomputation from the underlying cosine similarities or logits changes the highest-probability clean class. They remain in the primary common image set; excluding them in a predeclared numerical sensitivity leaves the conclusions unchanged. The adaptation and protocol were fixed after the earlier exploration of the image pool and before calculating these baseline results.

Table~\ref{tab:tta-main} gives the numerical results plotted in Figure~\ref{fig:probe-score}.

\begin{table}[!htbp]
\centering\setlength{\tabcolsep}{4pt}
\caption{Two probe sets crossed with two scoring rules on identical mild-condition image sets. Each uses 18 transformed inputs; $A^3$ is the A3Rank scoring adaptation. Recall uses a 20\% recapture budget. Difference intervals are in Appendix~\ref{app:alignment-baseline}.}
\label{tab:tta-main}

\footnotesize
\begin{tabular}{lrrrr}
\toprule
& \multicolumn{2}{c}{Natural transforms} & \multicolumn{2}{c}{Corruptions} \\
Scoring rule & ES & Diverse & ES & Diverse \\
\midrule
\multicolumn{5}{l}{\emph{Failure recall at 20\% budget}} \\
Probability mean & 0.596 & 0.366 & 0.672 & 0.449 \\
$A^3$ score & 0.588 & 0.360 & 0.654 & 0.447 \\
\midrule
\multicolumn{5}{l}{\emph{AUROC}} \\
Probability mean & 0.791 & 0.688 & 0.854 & 0.801 \\
$A^3$ score & 0.784 & 0.680 & 0.836 & 0.775 \\
\bottomrule
\end{tabular}

\end{table}

The same-corruption scores are close in Diverse coverage at the primary budget, while $U$ has higher AUROC on both sources (Table~\ref{tab:alignment-differences}). At 10\% on Diverse, $A^3$ covers about 0.268 of failures versus 0.260 for $U$; at 30\%, $U$ is better (Table~\ref{tab:alignment-budgets}). Thus this comparison supports the usefulness of the tested synthetic responses with more than one scoring construction and locates budget-dependent differences between them. Both probe variants are also evaluated across all tiers.

\begin{table}[!htbp]
\centering\setlength{\tabcolsep}{4pt}
\caption{Mild-condition differences between $U$ and the A3Rank scoring adaptations. Differences and separate 95\% resampling intervals are in percentage points. Group resampling uses 2,000 draws of ten model groups; class resampling uses 200 draws of 200 classes and recomputes metrics on weighted images.}
\label{tab:alignment-differences}

\footnotesize
\begin{tabular}{lllrrr}\toprule
Source & A3 probes & Metric & $U-A^3$ & Group interval & Class interval\\\midrule
ES & Natural & Recall@.2 & +8.38 & [+7.15, +9.40] & [+6.72, +10.06] \\
ES & Natural & AUROC & +6.95 & [+6.30, +7.48] & [+5.51, +8.47] \\
ES & Corruption & Recall@.2 & +1.85 & [+0.73, +2.74] & [+0.69, +2.86] \\
ES & Corruption & AUROC & +1.77 & [+1.37, +2.07] & [+1.30, +2.31] \\
Diverse & Natural & Recall@.2 & +8.91 & [+7.41, +10.45] & [+7.06, +10.97] \\
Diverse & Natural & AUROC & +12.09 & [+10.67, +13.04] & [+10.56, +14.00] \\
Diverse & Corruption & Recall@.2 & +0.23 & [-0.40, +1.18] & [-0.58, +1.09] \\
Diverse & Corruption & AUROC & +2.66 & [+2.24, +3.14] & [+2.20, +3.22] \\
\bottomrule\end{tabular}

\end{table}

\begin{table}[!htbp]
\centering\setlength{\tabcolsep}{3pt}
\caption{Budget sensitivity of the corruption-based ranking adaptation. $U$ and $A^3$ columns report failure recall; differences $U-A^3$ and their 95\% model-group intervals are in percentage points. The primary budget is 20\%; $A^3$ has higher recall on Diverse at 10\%.}
\label{tab:alignment-budgets}

\footnotesize
\begin{tabular}{lrrrrr}\toprule
Source & Budget & $U$ & $A^3$ & Difference & Group interval\\\midrule
ES & 10\% & 0.489 & 0.479 & +0.93 & [-0.84, +2.13] \\
ES & 20\% & 0.672 & 0.654 & +1.85 & [+0.73, +2.74] \\
ES & 30\% & 0.771 & 0.758 & +1.29 & [+0.71, +2.16] \\
Diverse & 10\% & 0.260 & 0.268 & -0.80 & [-1.37, -0.44] \\
Diverse & 20\% & 0.449 & 0.447 & +0.23 & [-0.40, +1.18] \\
Diverse & 30\% & 0.591 & 0.577 & +1.41 & [+0.59, +2.64] \\
\bottomrule\end{tabular}

\end{table}

\FloatBarrier
\subsection{Probe and capture-condition diagnostics}\label{app:probe-diagnostics}
These exploratory diagnostics use the same predictions and development-defined mild tiers as the primary comparison. We score each of the 18 corruptions separately, average the two severities within each of nine types, and average all nine probes of each severity. The clean-correct image sets, target labels, failure recall and tie proration match the primary comparison. These sets use one, two, nine or 18 probe evaluations respectively; only the full corruption and natural sets have equal counts. Table~\ref{tab:probe-severity} gives the severity comparison, and Figure~\ref{fig:probe-diagnostics} displays every individual probe.

\begin{table}[!htbp]
\centering
\caption{Mild screening by probe set. Evaluations count transformed inputs beyond the common clean prediction; the clean-only row counts that prediction. The severity subsets use different numbers of transformed inputs. Differences are in recall units.}
\label{tab:probe-severity}

\footnotesize
\begin{tabular}{llrrrr}
\toprule
Source & Probe set & Evaluations & Recall@.2 & AUROC & $\Delta$ vs. TTA \\
\midrule
ES & Clean & 1 & 0.578 & 0.778 & -0.018 \\
ES & Natural TTA & 18 & 0.596 & 0.791 & +0.000 \\
ES & Severity 2 & 9 & 0.658 & 0.837 & +0.062 \\
ES & Severity 4 & 9 & 0.669 & 0.854 & +0.073 \\
ES & All corruptions & 18 & 0.672 & 0.854 & +0.076 \\
Diverse & Clean & 1 & 0.357 & 0.678 & -0.009 \\
Diverse & Natural TTA & 18 & 0.366 & 0.688 & +0.000 \\
Diverse & Severity 2 & 9 & 0.416 & 0.760 & +0.050 \\
Diverse & Severity 4 & 9 & 0.461 & 0.811 & +0.094 \\
Diverse & All corruptions & 18 & 0.449 & 0.801 & +0.083 \\
\bottomrule
\end{tabular}

\end{table}

\begin{figure}[htbp]
\centering\includegraphics[width=\linewidth]{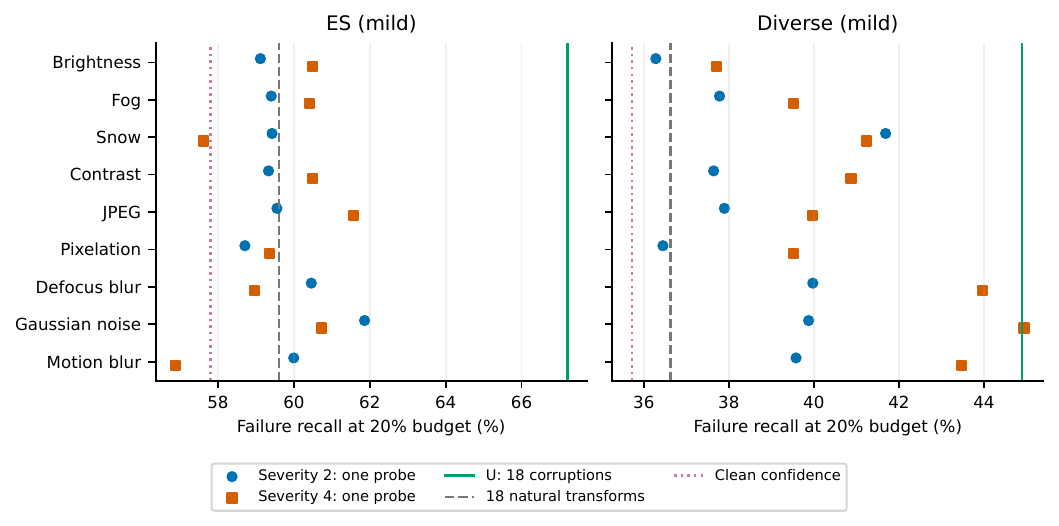}
\caption{Failure recall at a 20\% testing budget for each corruption probe on mild recaptures. Circles and squares distinguish severities 2 and 4. Vertical lines mark the 18-corruption average, 18-transform natural average and clean confidence. Each point uses one transformed input, while the averages use 18. The probe sets differ in both transformation family and response strength.}
\label{fig:probe-diagnostics}
\end{figure}

Every two-severity type average exceeds natural TTA in mean coverage on both sources (9/9 types), whereas only 8/18 ES and 16/18 Diverse individual probes do so. Increasing severity from 2 to 4 improves recall for 5/9 ES and 8/9 Diverse types. Severity-4 averaging reaches 0.669/0.461 compared with 0.672/0.449 for the full set; adding more probes therefore need not improve the score. Mean development error rates, averaged across probes and models on the historical clean-correct set, are 2.03\% for natural TTA, 6.45\% for severity 2 and 13.71\% for severity 4. These rates average the errors of individual transformed predictions. The natural and corruption sets differ in both family and response strength, so the observed advantage is attributed to the tested sets rather than to corruption type independently of dose.

For capture sensitivity, each mild condition contributes its paired $U-S$ and $U-\mathrm{TTA}_{18}$ coverage difference, averaged over models. Both means are positive in all 44 ES and 51 Diverse mild conditions. Grouping by lighting, ISO, shutter or aperture, or deleting any one factor level before averaging, preserves positive means. These factors are crossed acquisition settings, not independent replication units. The positive means occur across the tested conditions; these results establish neither per-model improvement in every condition nor simultaneous significance across conditions. The model-group differences in Appendix~\ref{app:simple-baselines} remain relevant; these diagnostics reuse the same models and images.

Table~\ref{tab:tta-full} reports the corruption and natural-transformation comparison across all tiers.

\begin{table}[!htbp]
\centering\setlength{\tabcolsep}{2.2pt}
\caption{Corruption versus natural-transformation averages across all tiers, using the same model evaluations. Coverage is failure recall at $q=0.2$; paired differences carry separate 95\% group and class intervals.}
\label{tab:tta-full}

\scriptsize
\begin{tabular}{@{}llrrrllrrl@{}}\toprule
& & \multicolumn{3}{c}{Coverage at $q{=}0.2$} & \multicolumn{2}{c}{$U_{18}-\mathrm{TTA}_{18}$} & \multicolumn{3}{c}{AUROC}\\
 Source & Tier & $S$ & $U_{18}$ & $\mathrm{TTA}_{18}$ & [groups] & [classes] & $U_{18}$ & $\mathrm{TTA}_{18}$ & $\Delta$ [groups]\\\midrule
ES & mild & 0.578 & 0.672 & 0.596 & +0.076 [+0.063,+0.086] & [+0.059,+0.088] & 0.854 & 0.791 & +0.063 [+0.058,+0.067]\\
ES & severe & 0.241 & 0.254 & 0.242 & +0.013 [+0.010,+0.016] & [+0.007,+0.018] & 0.647 & 0.596 & +0.051 [+0.041,+0.059]\\
ES & extreme & 0.208 & 0.210 & 0.208 & +0.002 [+0.002,+0.003] & [+0.001,+0.004] & 0.578 & 0.540 & +0.038 [+0.026,+0.052]\\
Diverse & mild & 0.357 & 0.449 & 0.366 & +0.083 [+0.066,+0.098] & [+0.066,+0.103] & 0.801 & 0.688 & +0.114 [+0.100,+0.123]\\
Diverse & severe & 0.248 & 0.274 & 0.250 & +0.023 [+0.019,+0.028] & [+0.018,+0.031] & 0.700 & 0.607 & +0.093 [+0.078,+0.104]\\
Diverse & extreme & 0.203 & 0.204 & 0.203 & +0.001 [+0.001,+0.002] & [+0.001,+0.002] & 0.512 & 0.504 & +0.008 [-0.006,+0.026]\\
\bottomrule
\end{tabular}

\end{table}

\FloatBarrier
\subsection{Simple scoring controls and model-group decomposition}\label{app:simple-baselines}
These exploratory checks compare additional scores on the same image pool and model predictions as the primary analysis. The three added scores were fixed before calculation. Clean margin is the largest minus the second-largest probability in the 200-class candidate space; lower values indicate greater failure risk. Clean entropy is $-\sum_c p_c\log p_c$, with higher values indicating risk. These are simple uncertainty controls for test prioritization \citep{weiss2022simple}. TTA agreement is the fraction of the fixed 18 natural transforms whose candidate prediction equals the recorded clean prediction, with lower agreement indicating risk. The reference is the clean prediction from the common re-inference; eligibility remains the historical clean-correct set. Agreement is measured against this fixed clean prediction.

Scores use double-precision probability calculations from single-precision cosine similarities or logits under the common 200-class protocol. All 9,504 evaluation cells reproduce the existing $S$, $U_{18}$ and $\mathrm{TTA}_{18}$ AUROC and recall at $q\in\{0.1,0.2,0.3\}$ to within $10^{-10}$, with identical sets of cells having undefined metrics. The comparisons use the same cells and tie rules. Table~\ref{tab:simple-baselines} reports the primary mild, $q=0.2$ results. Group intervals use 2,000 whole-group draws; class intervals use 200 shared class draws. These intervals quantify sensitivity within the existing pool. No score direction, subset or transformation is selected using the new outcomes.

\begin{table}[!htbp]
\centering\setlength{\tabcolsep}{3pt}
\caption{Additional simple baselines on mild recaptures. $\Delta$ is the original mean recall difference $U-$baseline at $q=0.2$; the final two columns give separate 95\% group/class resampling intervals for that difference. AUROC is the secondary metric; all three baselines were specified before calculation.}
\label{tab:simple-baselines}

\footnotesize
\begin{tabular}{llrrrrr}
\toprule
Source & Score & Recall@.2 & AUROC & $U-$score & Group CI & Class CI \\
\midrule
ES & $U$ & 0.672 & 0.854 & -- & -- & -- \\
ES & $S$ & 0.578 & 0.778 & +0.094 & [+0.079,+0.106] & [+0.076,+0.115] \\
ES & TTA & 0.596 & 0.791 & +0.076 & [+0.063,+0.086] & [+0.059,+0.088] \\
ES & Margin & 0.577 & 0.779 & +0.096 & [+0.079,+0.110] & [+0.073,+0.115] \\
ES & Entropy & 0.574 & 0.773 & +0.098 & [+0.079,+0.114] & [+0.081,+0.121] \\
ES & Agreement & 0.423 & 0.648 & +0.249 & [+0.208,+0.284] & [+0.212,+0.283] \\
Diverse & $U$ & 0.449 & 0.801 & -- & -- & -- \\
Diverse & $S$ & 0.357 & 0.678 & +0.092 & [+0.076,+0.109] & [+0.074,+0.113] \\
Diverse & TTA & 0.366 & 0.688 & +0.083 & [+0.066,+0.098] & [+0.066,+0.103] \\
Diverse & Margin & 0.357 & 0.679 & +0.092 & [+0.076,+0.107] & [+0.073,+0.113] \\
Diverse & Entropy & 0.355 & 0.676 & +0.094 & [+0.077,+0.114] & [+0.076,+0.113] \\
Diverse & Agreement & 0.264 & 0.558 & +0.185 & [+0.153,+0.227] & [+0.161,+0.214] \\
\bottomrule
\end{tabular}

\end{table}

The margin, entropy and agreement controls do not remove the mean advantage in either mild source. Table~\ref{tab:screening-families} decomposes the original primary comparisons. Every model-group mean is positive for $U-S$ and $U-\mathrm{TTA}_{18}$. At the model level, $U-S$ is positive for all 44 models on both sources; $U-\mathrm{TTA}_{18}$ is positive for 43/44 on ES and 44/44 on Diverse. AlexNet is the ES exception ($-0.00321$ recall); its $U-$margin difference is also negative ($-0.00222$). Group means therefore do not establish improvement for every model. No model is excluded for its comparison outcome.

\begin{table}[!htbp]
\centering\setlength{\tabcolsep}{5pt}
\caption{Mild-condition recall differences at $q=0.2$ by the ten fixed model groups, with cells weighted equally within each group. The statistical groups organize the evaluated model pool and have unequal sizes.}
\label{tab:screening-families}

\footnotesize
\begin{tabular}{llrrr}
\toprule
Source & Group & Cells & $U-S$ & $U-\mathrm{TTA}$ \\
\midrule
Diverse & SSL & 204 & +0.085 & +0.074 \\
Diverse & DataComp & 204 & +0.107 & +0.093 \\
Diverse & DFN & 102 & +0.115 & +0.091 \\
Diverse & LAION & 612 & +0.093 & +0.091 \\
Diverse & MetaCLIP & 153 & +0.123 & +0.115 \\
Diverse & OpenAI & 255 & +0.072 & +0.064 \\
Diverse & Other CLIP & 153 & +0.046 & +0.028 \\
Diverse & PE-Core & 102 & +0.109 & +0.090 \\
Diverse & SigLIP & 204 & +0.145 & +0.131 \\
Diverse & Supervised & 255 & +0.053 & +0.051 \\
ES & SSL & 176 & +0.092 & +0.089 \\
ES & DataComp & 176 & +0.087 & +0.062 \\
ES & DFN & 88 & +0.159 & +0.119 \\
ES & LAION & 528 & +0.105 & +0.086 \\
ES & MetaCLIP & 132 & +0.105 & +0.099 \\
ES & OpenAI & 220 & +0.092 & +0.066 \\
ES & Other CLIP & 132 & +0.056 & +0.026 \\
ES & PE-Core & 88 & +0.057 & +0.052 \\
ES & SigLIP & 176 & +0.115 & +0.087 \\
ES & Supervised & 220 & +0.064 & +0.062 \\
\bottomrule
\end{tabular}

\end{table}

\FloatBarrier
\section{Inference, Splits, and Threshold Conventions}\label{app:protocol}
\paragraph{Source overlap and tier counts.}
Syn and the real pool overlap in 209 images, including 63 Syn-development/real-evaluation overlaps; cross-source image populations are therefore not independent. The development-accuracy tiers contain 44/51 mild, 5/25 severe and 5/86 extreme ES/Diverse conditions. Model-level prediction uses all Syn images or Syn-B development images to predict real evaluation targets.

\paragraph{Inference and scoring.}
We evaluate models over the candidate classes defined for each dataset, using a common scoring protocol to obtain categorical probabilities. The study thus compares models under a shared scoring convention rather than assessing the native calibration of individual checkpoints. Predicted labels, correctness, true-class probabilities, and maximum predicted probabilities are derived from the same inference outputs.

\paragraph{Preprocessing and reproducibility.}
Clean and perturbed images undergo the same model-specific preprocessing. Although the original computational environment and preprocessing configurations were not fully documented, repeating the original inference procedure on the same clean and corrupted images exactly reproduced the stored prediction records for all 44 models, as reported in Appendix~\ref{app:screening}. These records and the accompanying experimental documentation provide the basis for the reported analyses. Table~\ref{tab:models} identifies the evaluated checkpoints.

\paragraph{Perturbations and data splits.}
We consider nine synthetic corruption types---brightness, fog, snow, contrast, JPEG compression, pixelation, defocus blur, Gaussian noise, and motion blur---at two severity levels, yielding 18 experimental conditions. Additional higher-severity Syn-B images are excluded from the primary analyses. The development and evaluation sets are stratified by class, with fixed image assignments. Because the original per-image corruption seeds were not retained, exact regeneration of the perturbed images is not guaranteed. The severity levels represent nominal experimental conditions rather than calibrated measures of physical perturbation intensity. Figure~\ref{fig:perturbation-examples} illustrates the 18 corruption conditions, representative natural transformations used in Section~\ref{sec:prosp}, and mild-condition physical recaptures from ES and Diverse.

\begin{figure}[htbp]
\centering
\includegraphics[width=\linewidth]{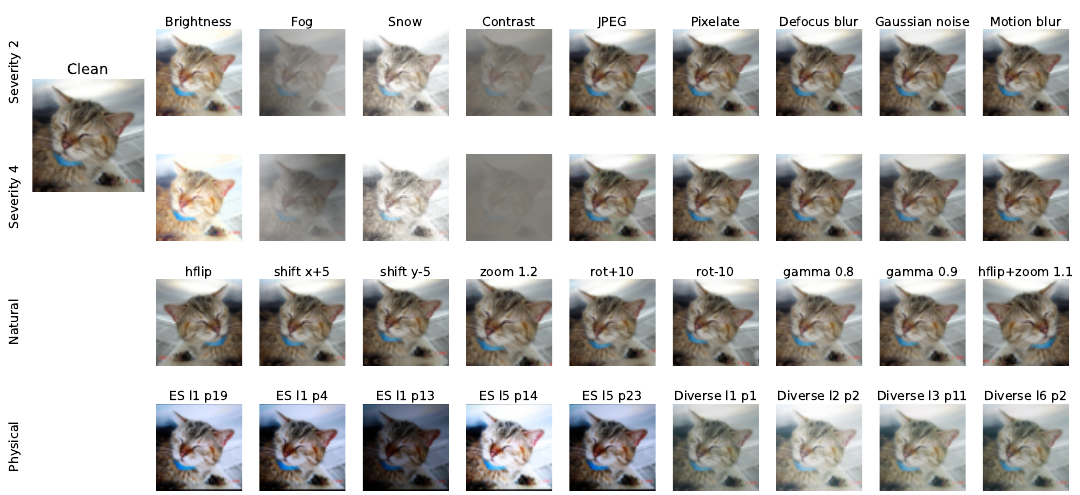}
\caption{Examples of three perturbation families for the same base image. Rows 1--2 show synthetic corruptions used in Syn and Syn-B. Row 3 shows representative natural transformations used for comparison. Row 4 shows original physical recaptures from ImageNet-ES and ES Diverse.}
\label{fig:perturbation-examples}
\end{figure}

\paragraph{Acceptance and ties.}
The primary fixed-coverage analysis evaluates risk among the highest-scoring 80\% of target samples, with ties resolved to maintain the prescribed sample size. The frozen threshold is the lowest score among the highest-scoring 80\% of clean development samples. All scores at or above this threshold are accepted, so ties may increase clean-development coverage beyond the nominal level. Fractional-coverage diagnostics instead average over ties at the acceptance boundary. The supplementary 5\% risk rule selects the largest top-ranked subset with empirical error no greater than 5\% and accepts all scores at or above its cutoff. Including all cutoff ties can change the empirical risk even on the data used to select the threshold. Neither this rule nor the primary frozen threshold guarantees risk control in deployment.

\paragraph{Consistency across evaluations.}
The transition, identification, and binary-score analyses use a single, internally consistent set of predictions and probabilities. Full multiclass Brier scores use true-class and collision probabilities from a separate, internally consistent evaluation, available only for Syn and ES; incompatible probabilities from earlier evaluations are not incorporated into these scores. The supplementary clean probability-gap analysis combines gaps from this second evaluation with outcomes from the primary evaluation. Across 484,000 clean prediction records, predicted labels and correctness agree exactly, with small probability differences confined to DINOv3. This agreement does not establish equality of the full probability vectors, and multiclass scores for perturbed inputs remain unavailable for the other sources.

\paragraph{Probability-gap analysis.}
The probability gap is the difference between the largest and second-largest predicted class probabilities. For each model, we define three groups using the one-third and two-thirds quantiles of the gap distribution among misclassified clean development samples, assigning boundary values to the higher group. A model is excluded from the corresponding analysis when fewer than six such errors are available. These cutpoints are defined relative to clean errors, rather than the overall clean population. The analysis does not control for initial confidence and does not support a causal interpretation.

\subsection{Acceptance protocols and calibration}
Platt scaling (logistic regression on the logit of $S$) and isotonic regression \citep{zadrozny2002isotonic} are fitted on clean development records per model and source and applied to both sides. Under the quantile protocol of Section~\ref{sec:design} (80\% clean-development coverage), a strictly increasing calibrator cannot change any acceptance decision; Platt scaling reproduces the raw decisions in all 88 model--source pairs with zero difference, and isotonic steps would only create ties that inflate coverage. Table~\ref{tab:protocols} shows that calibration does change $H$, $L$ and selective risk under absolute thresholds: at a threshold of 0.9, $H$ falls by roughly 60\% and selective risk from 7.9 to 5.4 points on ES. Whether calibration matters for a rejection rule is therefore decided by the protocol in this setting rather than by the model alone. No calibrator is fitted on the perturbed side.

\begin{table}[!htbp]
\centering\setlength{\tabcolsep}{2.5pt}
\caption{Acceptance events under the quantile protocol (Q80) and absolute-probability thresholds (A$p$), for raw and Platt-scaled scores fitted on clean development records (real sources, all conditions, cells weighted equally). $H$, $L$ and selective risk are expressed as percentages; coverage is the fraction of perturbed inputs accepted. Selective risk is the error rate among accepted perturbed inputs.}
\label{tab:protocols}

\footnotesize
\begin{tabular}{@{}llrrrrrrrr@{}}\toprule
& & \multicolumn{2}{c}{$H$} & \multicolumn{2}{c}{$L$} & \multicolumn{2}{c}{Sel.\ risk} & \multicolumn{2}{c}{Coverage}\\
Source & Protocol & raw & Platt & raw & Platt & raw & Platt & raw & Platt\\\midrule
ES & Q80 & 1.39 & 1.39 & 9.90 & 9.90 & 10.48 & 10.48 & 0.531 & 0.531\\
ES & A0.9 & 0.65 & 0.26 & 9.48 & 10.81 & 7.92 & 5.44 & 0.484 & 0.476\\
ES & A0.7 & 2.31 & 1.82 & 7.48 & 6.45 & 14.03 & 13.26 & 0.602 & 0.608\\
ES & A0.5 & 5.39 & 5.27 & 4.20 & 3.26 & 19.78 & 20.60 & 0.703 & 0.712\\
Diverse & Q80 & 2.05 & 2.05 & 8.16 & 8.16 & 18.95 & 18.95 & 0.186 & 0.186\\
Diverse & A0.9 & 1.09 & 0.34 & 7.61 & 8.40 & 16.54 & 11.24 & 0.166 & 0.153\\
Diverse & A0.7 & 3.13 & 2.17 & 5.83 & 5.48 & 29.21 & 26.75 & 0.233 & 0.226\\
Diverse & A0.5 & 6.81 & 6.23 & 3.86 & 3.15 & 39.88 & 42.70 & 0.305 & 0.304\\
\bottomrule
\end{tabular}

\end{table}

\FloatBarrier
\section{Model Pool and Evidence Provenance}\label{app:models}
ImageNet supplies the classification reference task \citep{russakovsky2015}. Table~\ref{tab:models} lists the evaluated checkpoints. The supervised models use AlexNet \citep{krizhevsky2012alexnet}, ResNet \citep{he2016}, EfficientNet \citep{tan2019efficientnet}, Vision Transformer \citep{dosovitskiy2021} and ConvNeXt \citep{liu2022convnext} architectures. The self-supervised lineages are MAE \citep{he2022mae}, DINOv2 \citep{oquab2024dinov2} and DINOv3 \citep{simeoni2025dinov3}; the identifiers distinguish fine-tuned MAE classifiers from DINO linear classifiers.

The image--text models include OpenAI CLIP \citep{radford2021}, EVA-CLIP \citep{sun2023evaclip}, Perception Encoder (PE-Core) \citep{bolya2025perception}, SigLIP \citep{zhai2023siglip}, SigLIP~2 \citep{tschannen2025siglip2}, MetaCLIP \citep{xu2024metaclip}, Meta CLIP~2 \citep{chuang2025metaclip2}, CLIPA \citep{li2023clipa,li2023clipav2} and CoCa \citep{yu2022coca}. Training-data and checkpoint provenance are distinguished from architecture: the pool includes DataComp \citep{gadre2023datacomp} and Data Filtering Networks (DFN) \citep{fang2024dfn} releases, as well as OpenCLIP checkpoints trained on LAION-400M or LAION-2B \citep{schuhmann2021laion400m,schuhmann2022laion5b,cherti2023scaling}. The two RN50 checkpoints labeled \texttt{cc12m} and \texttt{yfcc15m} are OpenCLIP releases using Conceptual~12M \citep{changpinyo2021cc12m} and a subset of YFCC100M \citep{thomee2016yfcc}, respectively. OpenCLIP \citep{ilharco2021openclip} and timm \citep{wightman2019timm} provide implementations and pretrained-model registries; their software citations are distinct from the original method and dataset references.

These citations document the model, training and data lineages represented; they do not imply that every evaluated checkpoint reproduces the original paper's training recipe. Architectural variety does not make the pool a random sample of classifiers or isolate architecture from pretraining, scale, resolution or classifier construction. Model-group holdouts therefore assess transfer within this defined population.
\begingroup\small
\setlength{\LTcapwidth}{\linewidth}
\begin{longtable}{@{}p{0.68\linewidth}lr@{}}
\caption{Complete model pool with checkpoint identifiers and parameter counts. Groups define the statistical holdouts and may contain checkpoints trained on different datasets.}\label{tab:models}\\
\toprule Model identifier & Group & Params. / M\\ \midrule \endfirsthead
\multicolumn{3}{c}{Table~\thetable\ (continued).}\\[\baselineskip]
\toprule Model identifier & Group & Params. / M\\ \midrule \endhead
\bottomrule \endfoot
\path{DINOv2-vitg14-lc__meta} & SSL & 1144\\
\path{DINOv3-ViT7B16-lc__meta} & SSL & 6716\\
\path{EVA02-L-14__merged2b_s4b_b131k} & other-clip & 428\\
\path{MAE-ViT-B-ft__meta} & SSL & 87\\
\path{MAE-ViT-L-ft__meta} & SSL & 304\\
\path{PE-Core-L-14-336__meta} & pe-core & 671\\
\path{PE-Core-bigG-14-448__meta} & pe-core & 2419\\
\path{RN101-quickgelu__openai} & openai & 120\\
\path{RN50-quickgelu__cc12m} & other-clip & 102\\
\path{RN50-quickgelu__openai} & openai & 102\\
\path{RN50-quickgelu__yfcc15m} & other-clip & 102\\
\path{ViT-B-16-SigLIP__webli} & siglip & 203\\
\path{ViT-B-16-quickgelu__openai} & openai & 150\\
\path{ViT-B-16__datacomp_xl_s13b_b90k} & datacomp & 150\\
\path{ViT-B-16__dfn2b} & dfn & 150\\
\path{ViT-B-16__laion2b_s34b_b88k} & laion & 150\\
\path{ViT-B-16__laion400m_e32} & laion & 150\\
\path{ViT-B-32-quickgelu__metaclip_fullcc} & metaclip & 151\\
\path{ViT-B-32-quickgelu__openai} & openai & 151\\
\path{ViT-B-32__datacomp_xl_s13b_b90k} & datacomp & 151\\
\path{ViT-B-32__laion2b_s34b_b79k} & laion & 151\\
\path{ViT-B-32__laion400m_e32} & laion & 151\\
\path{ViT-H-14-CLIPA__datacomp1b} & datacomp & 968\\
\path{ViT-H-14-quickgelu__dfn5b} & dfn & 986\\
\path{ViT-H-14__laion2b_s32b_b79k} & laion & 986\\
\path{ViT-L-14-quickgelu__metaclip_fullcc} & metaclip & 428\\
\path{ViT-L-14-quickgelu__openai} & openai & 428\\
\path{ViT-L-14__datacomp_xl_s13b_b90k} & datacomp & 428\\
\path{ViT-L-14__laion2b_s32b_b82k} & laion & 428\\
\path{ViT-L-14__laion400m_e32} & laion & 428\\
\path{ViT-SO400M-14-SigLIP__webli} & siglip & 877\\
\path{ViT-SO400M-16-SigLIP2-384__webli} & siglip & 1136\\
\path{ViT-bigG-14-worldwide-378__metaclip2_worldwide} & metaclip & 3631\\
\path{ViT-bigG-14__laion2b_s39b_b160k} & laion & 2540\\
\path{ViT-g-14__laion2b_s34b_b88k} & laion & 1367\\
\path{ViT-gopt-16-SigLIP2-384__webli} & siglip & 1872\\
\path{alexnet-sup__tv} & supervised & 61\\
\path{coca_ViT-L-14__laion2b_s13b_b90k} & laion & 638\\
\path{convnext-b-sup__timm} & supervised & 89\\
\path{convnext_base_w__laion2b_s13b_b82k} & laion & 179\\
\path{convnext_large_d_320__laion2b_s29b_b131k_ft_soup} & laion & 352\\
\path{efficientnet-b4-sup__timm} & supervised & 19\\
\path{resnet50-sup__timm} & supervised & 26\\
\path{vit-b16-sup__timm} & supervised & 87\\
\end{longtable}
\space
\endgroup
Model sources and data-access information accompany the supplementary materials.

The stratified split seed is 20260908. Development thresholds and evaluation targets remain separate. Values are usually rounded to two decimals, with three for selected correlations or proportions; rounding does not change the original decision criteria. Whole-group-removal refitting is reported in Appendix~\ref{app:new-p8}. Human review of recapture semantics and validation on additional models or images remain open directions; Appendix~\ref{app:labels} reports the annotation-sensitivity analyses and planned human review.

\section{Definitions and Prediction Methods}\label{app:definitions}
The remaining-degree-of-freedom calculation supplements Section~\ref{sec:prelim}. The independent counterparts, feature sets and fitting procedures specify the prediction comparisons in Section~\ref{sec:boundaries} and Appendix~\ref{app:new-p8}.

\subsection{The remaining degree of freedom}
The identities in Table~\ref{tab:response-states} and the fact that the five proportions sum to one imply
\begin{equation}
h=a_0-k,\quad r=a_1-k,\quad w=F-a_0-a_1+2k,\quad s=1-F-k.
\end{equation}
Once $a_0,a_1,F$ are included, adding four transition proportions introduces algebraic redundancy. We represent the remaining degree of freedom with $k$. This identity informs the comparison design.

\subsection{Constructing class-conditional independent counterparts}
Fix a model and condition. For true class $y$, let $n_y$ be the number of base images, $a_y,b_y$ the clean and perturbed correct counts, and $\mu_y,\mu_y'$ the mean true-class probabilities among each side's correct images. Independent combination within classes gives
\begin{equation}
dM_{k,\ind}=\frac{\sum_y(a_yb_y/n_y)(\mu_y-\mu_y')}{\sum_y a_yb_y/n_y}.\label{eq:ind}
\end{equation}
Classes with zero correct images on either side receive zero weight and need no corresponding mean. This ratio of expected numerator and denominator is neither the difference between the two global correct-image means nor, in general, the expectation of a conditional mean under random permutation, because the denominator can vary. It requires class labels but no observed correspondence.

Each clean image is combined with every perturbed image of the same class, with weight $1/n_y$ per combination, preserving the total image weight. The empirical control instead uniformly samples within-class permutations, including the identity, moving each complete perturbed prediction record. The number of permutations is $\prod_y n_y!$; even two images per class across $C$ classes yield $2^C$ combinations. Prediction errors are averaged within each model over 20 reproducible permutation replicates before comparison.

\subsection{Feature sequence and comparisons}
Table~\ref{tab:ladder} defines the feature sequence used in these comparisons.
\begin{table}[H]
\caption{Feature definitions. $L_k^{\rm syn}$ is the mass of synthetic retained-correct inputs that lose acceptance. All features are aggregated at the model level.}
\label{tab:ladder}\centering

\small
\begin{tabularx}{\linewidth}{@{}lY@{}}
\toprule Set / step & Features\\ \midrule
$B_0$ & Clean accuracy on synthetic base images and real development base images\\
$B_1$ & $B_0$, synthetic $a_1$, mean $S$ on each side, mean $M$ among each side's correct images, perturbed clipped log loss, and target-matched synthetic risk\\
$B_{\ind}$ & $B_1$ plus independent counterparts of $F,A,Q,k,dM_k,ds_s,ds_w,L_k^{\rm syn}$\\
C1 & Add $F,A,Q$ to $B_1$\\
C2 & Add $k$ to C1\\
C3a & Add $dM_k$ to C2 (primary hypothesis)\\
C3b & Add $ds_s,ds_w$ to C3a\\
C4 & Add $L_k^{\rm syn}$ to C3b\\
\bottomrule
\end{tabularx}
\end{table}

For preceding features $B$ and current additions $q$, the total increment compares $B$ with $B+q$. The primary controlled increment compares $B+q_{\ind}$ with $B+q_{\ind}+q$. The dimension-matched shuffle comparison uses $B+q_{\rm shuf}$ versus $B+q$. Since $B$ may contain preceding paired summaries, the primary sequence is not a comparison against a complete unpaired baseline at every step. C1 against $B_{\ind}$ directly probes that more demanding comparison.

For $R_{80}$ and $E_\tau$, $B_1$ includes the corresponding synthetic protocol quantity. For $L_k$, it includes only perturbed correct-and-rejected mass. No duplicate is added for $a_1$. Clipped log loss is $-\E[\log(\max(M',10^{-6}))]$. The true-class squared error $(1-M)^2$ is excluded from the prediction features; it differs from the full multiclass Brier score.

\FloatBarrier
\subsection{Fitting, aggregation, and intervals}
The ridge grid is $\{0.1,0.3,1,3,10,30,100\}$. Percentage targets are divided by 100, clipped to $[10^{-4},1-10^{-4}]$, and transformed by the standard-normal quantile function. Inner selection minimizes MAE after returning to the original scale. Regularization is selected inside the training groups of each outer fold, so that selection optimism does not enter the outer estimate \citep{cawley2010}. Missing features are imputed using training-column medians, with zero for wholly missing columns. Means and standard deviations are also estimated within training folds. Every feature set uses the same grid and procedure.

Synthetic summaries are computed per condition and averaged equally. Real prediction targets are computed from evaluation-image counts per model and condition, then pooled by source and severity group. Pooling counts is generally different from averaging per-condition risk ratios. The diagnostic figure pools counts across all models; model-level medians are explicitly identified when used. Cross-source analyses use condition-weighted statistics; estimates based on different subsets or aggregation schemes are kept separate.

Group-bootstrap intervals resample groups of stored out-of-fold predictions and average per-model error differences without rerunning the nested fitting procedure. Appendix~\ref{app:new-p8} separately reports whole-group deletion refits that exclude each designated group from both training and evaluation. These differ from the earlier evaluation-row deletion, which did not remove training groups.

\section{Prediction Comparisons and Full Group-Removal Refits}\label{app:new-p8}
Compressing synthetic responses into per-model statistics and predicting mild real $a_1,R_{80},E_\tau,L_k$ across held-out model groups gives target- and baseline-dependent gains. Here, $R_{80}$ denotes the error rate among the highest-scoring 80\% of perturbed inputs, using the number of selected inputs as the denominator. Against a baseline containing all class-conditional independent counterparts, simple paired statistics reduce Diverse accuracy MAE by 1.10 points and $R_{80}$ MAE by 0.38. The ES $R_{80}$ gain is not retained against the simpler unpaired baseline, which is better than the augmented baseline on three of four accuracy/risk targets (Tables~\ref{tab:strong-prediction} and~\ref{tab:baseline}). Retained-correct confidence $dM_k$ shows no stable incremental benefit across the tested comparisons. The comparisons below include feature-order checks and whole-group refits. These comparisons show why the prediction target matters: gains in image-level screening do not imply stable gains in predicting model-level performance.

\subsection{Distinct comparisons and feature order}
Baseline $B_1$ includes clean accuracies, perturbed synthetic accuracy, both sides' confidence, true-class confidence among correct records, log loss, the matching synthetic risk protocol and labeled real development accuracy. For any summary $q$, its class-conditional independent version combines every clean record with every perturbed record of the same class without observed correspondence; the stronger baseline $B_{\ind}$ adds all eight independent counterparts of $F,A,Q,k,dM_k,ds_s,ds_w,L_k$ to $B_1$, and Table~\ref{tab:strong-prediction} compares adding $(F,A,Q)$ or $dM_k$ to this baseline. Nested leave-one-group-out ridge regression selects regularization on training groups only, with imputation and scaling within folds; intervals resample ten groups of stored out-of-fold errors 2,000 times and are conditional summaries, not full-pipeline or new-population guarantees \citep{bengio2004}. Across the ten whole-group deletion refits, gains remain positive for Diverse accuracy (0.50--1.55 points) and $R_{80}$ (0.30--1.67), whereas ES $R_{80}$ gains change sign in three deletions; gains below the original practical threshold (10\% of the target interquartile range) do not establish zero benefit. For $dM_k$, the sequential test leaves ES accuracy compatible with improvement ($+0.52[-0.43,2.56]$) while the other seven interval upper endpoints fall below their thresholds, and paired and independent versions yield similar synthetic-to-real rank correlations and identical choices on the 38 Syn model pairs tested by the fixed tie-breaking rule.

\begin{table}[!htbp]
\centering
\caption{Syn summaries predicting mild real targets. Baseline is $B_{\ind}$; changes are MAE reductions in pp, so positive is better. C1 adds $(F,A,Q)$; the last column instead adds $dM_k$ alone. Brackets are conditional group-reweighting intervals of stored out-of-fold errors, not full fitting-uncertainty intervals. All sources and deletion refits are in Tables~\ref{tab:p8-all} and~\ref{tab:deletion} below.}
\label{tab:strong-prediction}

\small
\begin{tabular}{@{}llrrrr@{}}
\toprule
Target source & Target & Baseline & C1 gain & C1 interval & $dM_k$ gain\\\midrule
ES & $a_1$ & 2.88 & -0.12 & [-0.52, +0.15] & +0.11\\
ES & $R_{80}$ & 3.26 & +0.22 & [+0.12, +0.32] & +0.24\\
ES & $E_\tau$ & 0.61 & +0.00 & [-0.09, +0.07] & -0.02\\
ES & $L$ & 2.42 & -0.06 & [-0.15, -0.00] & -0.00\\
Diverse & $a_1$ & 3.35 & +1.10 & [+0.40, +2.12] & +0.01\\
Diverse & $R_{80}$ & 2.60 & +0.38 & [-0.22, +1.29] & +0.02\\
Diverse & $E_\tau$ & 0.77 & -0.02 & [-0.20, +0.11] & +0.03\\
Diverse & $L$ & 1.65 & +0.03 & [-0.00, +0.10] & +0.03\\
\bottomrule
\end{tabular}
\end{table}

The analysis includes the original C1 comparison, $B_1+(F,A,Q)_{\ind}$ versus that baseline plus $(F,A,Q)$, and the original sequential $dM_k$ comparison after preceding C1 and k features plus $dM_{k,\ind}$. Two further comparisons use the same stronger unpaired baseline: $B_{\ind}$ versus $B_{\ind}+(F,A,Q)$, and separately $B_{\ind}$ versus $B_{\ind}+dM_k$. The latter assesses feature-order sensitivity using the same model pool and outcomes. Compressed independent features do not exhaust the complete marginal records used in the identification theorem.

\begin{table}[!htbp]
\centering
\caption{All primary source/target settings: model-weighted MAE reductions (pp) with conditional OOF group-reweighting intervals. Original $dM_k$ conditions on the preceding paired summaries and its independent counterpart; $B_{\ind}+dM_k$ is a separate feature-order sensitivity. These are correlated comparisons on the same model pool.}
\label{tab:p8-all}

\small
\begin{tabular}{@{}lllrrr@{}}
\toprule
Features & Source & Target & $B_{\ind}+\mathrm{C1}$ & Original $dM_k$ & $B_{\ind}+dM_k$\\\midrule
Syn & ES & $a_1$ & -0.12 [-0.52, +0.15] & +0.52 [-0.43, +2.56] & +0.11 [-0.00, +0.27]\\
Syn & ES & $R_{80}$ & +0.22 [+0.12, +0.32] & +0.05 [-0.07, +0.22] & +0.24 [-0.00, +0.70]\\
Syn & ES & $E_\tau$ & +0.00 [-0.09, +0.07] & -0.04 [-0.11, +0.00] & -0.02 [-0.05, +0.00]\\
Syn & ES & $L$ & -0.06 [-0.15, -0.00] & -0.55 [-1.88, +0.05] & -0.00 [-0.06, +0.04]\\
Syn & Diverse & $a_1$ & +1.10 [+0.40, +2.12] & +0.02 [+0.00, +0.05] & +0.01 [-0.19, +0.13]\\
Syn & Diverse & $R_{80}$ & +0.38 [-0.22, +1.29] & -0.03 [-0.11, +0.01] & +0.02 [-0.04, +0.07]\\
Syn & Diverse & $E_\tau$ & -0.02 [-0.20, +0.11] & -0.01 [-0.15, +0.08] & +0.03 [-0.01, +0.09]\\
Syn & Diverse & $L$ & +0.03 [-0.00, +0.10] & -0.13 [-0.43, +0.00] & +0.03 [-0.02, +0.06]\\
Syn-B & ES & $a_1$ & +0.62 [+0.23, +0.96] & +0.01 [-0.03, +0.07] & -0.06 [-0.18, +0.01]\\
Syn-B & ES & $R_{80}$ & +0.68 [+0.21, +1.03] & +0.00 [-0.03, +0.04] & -0.11 [-0.71, +0.27]\\
Syn-B & ES & $E_\tau$ & -0.29 [-0.90, -0.03] & -0.04 [-0.12, +0.00] & -0.15 [-0.28, -0.04]\\
Syn-B & ES & $L$ & +0.04 [+0.00, +0.09] & -0.01 [-0.05, +0.05] & +0.00 [-0.01, +0.03]\\
Syn-B & Diverse & $a_1$ & +0.31 [-0.23, +0.81] & -0.23 [-0.78, +0.00] & +0.06 [-0.01, +0.19]\\
Syn-B & Diverse & $R_{80}$ & +0.20 [-0.45, +0.75] & -0.04 [-0.08, +0.01] & -0.12 [-0.40, +0.03]\\
Syn-B & Diverse & $E_\tau$ & -0.21 [-0.58, -0.01] & -0.02 [-0.11, +0.03] & -0.24 [-0.60, -0.06]\\
Syn-B & Diverse & $L$ & -0.12 [-0.26, -0.02] & -0.03 [-0.10, +0.01] & -0.04 [-0.15, +0.01]\\
\bottomrule
\end{tabular}
\end{table}

\FloatBarrier
\subsection{Deletion refits and weighting}
For each of ten removals, the designated group is excluded from both outer evaluation and all training/inner-validation sets. The entire nested fit is then repeated on the remaining nine groups for all 16 source--target combinations. An earlier diagnostic deleted only evaluation records while retaining previously fitted predictions; it therefore does not assess dependence on a group included in training.

\begin{table}[!htbp]
\centering
\caption{C1 against $B_{\ind}$: complete-source MAE and the range of MAE reductions after removing each group from both training and evaluation and refitting all inner/outer folds. The deletion range is a sensitivity range, not a confidence interval. Relative gain uses the matched baseline in that row.}
\label{tab:deletion}

\small
\begin{tabular}{@{}lllrrrrr@{}}
\toprule
Features & Source & Target & Baseline & C1 & Gain (\%) & Deletion gain (pp) & Positive\\\midrule
Syn & ES & $a_1$ & 2.88 & 3.00 & -4.1 & [-0.20, +0.19] & 6/10\\
Syn & ES & $R_{80}$ & 3.26 & 3.04 & +6.7 & [-0.24, +0.34] & 7/10\\
Syn & ES & $E_\tau$ & 0.61 & 0.61 & +0.0 & [-0.01, +0.39] & 6/10\\
Syn & ES & $L$ & 2.42 & 2.48 & -2.5 & [-0.16, +0.06] & 2/10\\
Syn & Diverse & $a_1$ & 3.35 & 2.25 & +33.0 & [+0.50, +1.55] & 10/10\\
Syn & Diverse & $R_{80}$ & 2.60 & 2.22 & +14.5 & [+0.30, +1.67] & 10/10\\
Syn & Diverse & $E_\tau$ & 0.77 & 0.80 & -2.9 & [-0.12, +0.02] & 2/10\\
Syn & Diverse & $L$ & 1.65 & 1.61 & +2.1 & [-0.10, +0.09] & 5/10\\
Syn-B & ES & $a_1$ & 4.12 & 3.49 & +15.2 & [-0.35, +0.70] & 9/10\\
Syn-B & ES & $R_{80}$ & 3.83 & 3.15 & +17.7 & [+0.02, +1.09] & 10/10\\
Syn-B & ES & $E_\tau$ & 0.78 & 1.08 & -37.3 & [-0.31, +0.27] & 4/10\\
Syn-B & ES & $L$ & 1.58 & 1.54 & +2.4 & [-0.01, +0.30] & 9/10\\
Syn-B & Diverse & $a_1$ & 2.09 & 1.78 & +15.0 & [+0.08, +1.54] & 10/10\\
Syn-B & Diverse & $R_{80}$ & 2.15 & 1.96 & +9.1 & [-0.06, +2.02] & 8/10\\
Syn-B & Diverse & $E_\tau$ & 1.82 & 2.03 & -11.6 & [-0.63, +0.53] & 4/10\\
Syn-B & Diverse & $L$ & 1.34 & 1.46 & -9.3 & [-0.12, +0.21] & 3/10\\
\bottomrule
\end{tabular}
\end{table}
\begin{figure}[t]\centering
\includegraphics[width=\linewidth]{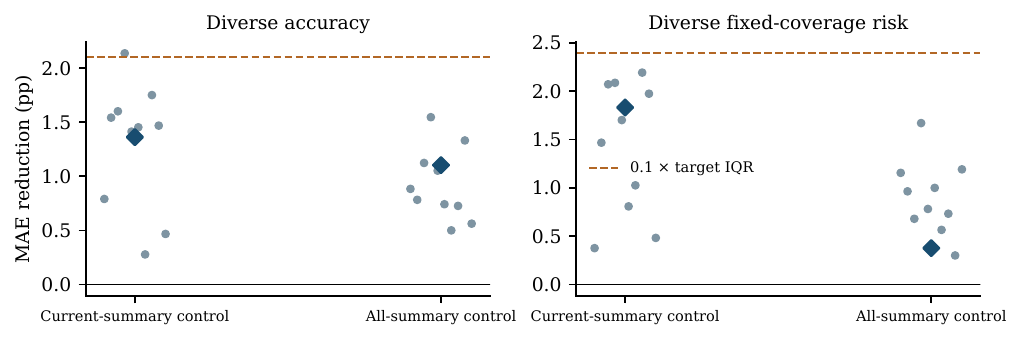}
\caption{Syn-to-Diverse C1 sensitivity. Diamonds use all 44 models; dots remove one group from both training and evaluation and rerun nested fitting. Current-summary control is $B_1+(F,A,Q)_{\ind}$; all-summary control is $B_{\ind}$. The dashed line is the original $0.1\,\mathrm{IQR}$ criterion. Dots show the individual group-deletion refits.}
\label{fig:deletion}\end{figure}

The prediction-error intervals continue to resample stored out-of-fold errors within each fixed fit; they do not repeat training in each bootstrap draw. The deletion experiments (Figure~\ref{fig:deletion}) are explicit changes of the training/evaluation population, not a fully refitted bootstrap. The model-equal summary weights each retained model equally; the group-equal summary first averages within each retained group. Inner regularization selection uses model-equal MAE in both cases. Neither treats the ten groups as guaranteed independent model populations. For Syn to Diverse under the stronger baseline, equal-group C1 gains are $1.41[0.67,2.21]$ points for accuracy and $0.42[-0.18,1.12]$ for $R_{80}$, compared with equal-model gains of 1.10 and 0.38. This weighting change preserves their respective positive and inconclusive conditional-interval patterns.

The original practical threshold is $0.1$ times the full-pool target IQR. Supplementary analyses retain thresholds of 0.025/0.05/0.1/0.2 IQR and 5/10/20\% of the matched baseline MAE. These thresholds assess sensitivity to the chosen practical effect size. For example, the Diverse $R_{80}$ original controlled C1 comparison is 4.36 to 2.53 MAE, a 42\% reduction, whereas the stronger-baseline comparison is 2.60 to 2.22, about 14.5\%. Each relative gain is calculated against its corresponding baseline.

\FloatBarrier
\subsection{Common lost-acceptance target}
The observed L from the identification analysis and the real $L_k$ used for prediction agree across all 88 model/source targets to $3.55\times10^{-15}$ percentage points after restricting identification to the same mild conditions. On these conditions, mean per-cell L widths are 6.14 to 5.01 for ES and 5.68 to 4.05 for Diverse when overlap is added. These are compatible bounds on each fixed cell; a single common permutation across conditions is not required. The Syn-to-real prediction comparison instead receives source summaries and learns from other model groups' real targets. The contrast therefore connects the same event across two different tasks, without equating bound width with forecast error.

\subsection{Numerical verification of prediction comparisons}
The analysis comprises 49,280 out-of-fold model predictions, with finite predictions and interval estimates for every planned fit. The full-cohort C1 and sequential $dM_k$ gains agree with the earlier results to within $10^{-8}$ percentage points. Independent calculations verify training-only preprocessing and nested prediction, including whole-group exclusions; eight independently refitted prediction groups differ by at most $7.11\times10^{-15}$ percentage points. The sensitivity analysis uses the existing model predictions and takes 199.84 seconds on CPU.
\subsection{Unpaired synthetic information is a useful starting point}
Adding unpaired synthetic quantities reduces accuracy-prediction MAE by 2.18 percentage points for ES and 3.58 for Diverse; fixed-coverage risk improves by 2.65 and 4.11 points (Table~\ref{tab:baseline}). Prediction of accepted-error mass and lost acceptance among retained-correct inputs does not show similarly stable gains. These improvements show the predictive value of unpaired synthetic information for these targets.

\begin{table}[!htbp]
\caption{Unpaired synthetic information versus clean-performance features (Syn; nested prediction across held-out model groups). MAEs and improvements are in percentage points; intervals resample stored out-of-fold errors by model group.}
\label{tab:baseline}\centering

\small
\begin{tabular}{@{}llrrr@{}}
\toprule
Target & Source & $B_0$ MAE & $B_1$ MAE & Gain [95\% interval]\\
\midrule
$a_1$ & ES & 4.87 & 2.69 & +2.18 [+0.68, +4.40]\\
$R_{80}$ & ES & 5.32 & 2.68 & +2.65 [+1.19, +5.01]\\
$a_1$ & DIV & 6.62 & 3.03 & +3.58 [+1.54, +6.87]\\
$R_{80}$ & DIV & 8.07 & 3.96 & +4.11 [+1.59, +8.32]\\
\bottomrule
\end{tabular}
\space
\end{table}

\FloatBarrier
\section{Annotation Sensitivity and Shared Persistent Errors}\label{app:labels}
These checks address the label sensitivity discussed in Section~\ref{sec:limits}, covering annotation changes, cross-model error agreement and the status of human recapture review.

\subsection{Matched ReaL evaluation and candidate compatibility}
ReaL provides accepted-label sets for original ImageNet images \citep{beyer2020real}. We map the model's 200-class prediction through its synset to the standard 1,000-class index, retain the original image identifier, and compare original and ReaL correctness on the same evaluation images. Of the shared 500 evaluation images, 473 have nonempty accepted-label sets; 27 are excluded from both sides of this matched comparison. This is distinct from the primary original-label analysis on all 500 images.

Of those 473 images, 435 have at least one accepted label among the 200 candidates and 38 have none. The latter remain in the main matched sensitivity analysis: every available prediction is wrong under their ReaL sets. A separate split gives 424 images whose original label is retained and 49 whose original label is removed. These partitions explain annotation sensitivity; the 435-image subgroup supplements the full matched analysis.

\begin{table}[!htbp]
\centering

\caption{Label sensitivity on the same 473 validation images with nonempty ReaL labels; percentages use the matched cohort rather than all 500 images.}
\label{tab:label-sensitivity}

\small
\begin{tabular}{lrrrrrr}
\toprule
Domain & Matched images & Original acc. & ReaL acc. & $\Delta$ (pp) & Original $s$ & ReaL $s$ \\
\midrule
ES & 473 & 66.88 & 63.33 & -3.55 & 3.99 & 8.01 \\
Diverse & 473 & 26.11 & 24.96 & -1.16 & 1.16 & 2.77 \\
\bottomrule
\end{tabular}
\end{table}

On the 473-image matched set (Table~\ref{tab:label-sensitivity}), mean $s$ rises from 3.99\% to 8.01\% on ES and from 1.16\% to 2.77\% on Diverse. On the 435-image candidate-compatible subgroup, it falls by 0.37/0.15 points. Using the full matched denominator, target accuracy changes decompose as $-3.55=+1.04-4.59$ points for ES and $-1.16=+0.67-1.83$ for Diverse, where the positive and negative terms come from the 435 and 38 images. These are rounded additive contributions, not semantic judgments about the recaptures.

Let $q_{\mathrm{agree}}$ denote the fraction of matched images with unchanged predicted labels. Because multiple accepted labels permit two different correct answers, we split $k$ into $k_{\mathrm{equal}}$ and $k_{\mathrm{different}}$ and verify
\begin{equation}
q_{\mathrm{agree}}-a_1=s-r-k_{\mathrm{different}}.\label{eq:multilabel-agreement}
\end{equation}
The primary five-state identities and finite bounds retain the original single-label definition. As a sensitivity check, we recompute the states with correctness attached to each record under its own ReaL set, on the 473 images with nonempty sets (Table~\ref{tab:real-states}). Relative to original labels, $h$ falls by 2 to 4 points and $s$ roughly doubles; $H$ falls by 6 to 9\% and $L$ by about 5\%; and $dM_k$, restricted to $k$ pairs whose original label is in the ReaL set, changes by less than half a point. Between 10\% and 18\% of original $h$ records and 16\% to 19\% of original $w$ records change state; 10\% of images have an original label absent from their ReaL set. Direct permutation enumeration over all \enumClassesReal{} classes in the \enumCellsReal{} real cells finds identical $L$ and $dM_k$ endpoints at $\mathcal I_k$ and $\mathcal I_5$ in every class, with strictly narrower five-state $H$ bounds in \enumNarrowerReal{} classes. This tests the algebra under record-attached correctness; it does not re-judge correctness after reassignment, and reusing original-image accepted labels on recaptures remains a sensitivity exercise that does not establish that the captured content is recognizable.

\begin{table}[!htbp]
\centering\setlength{\tabcolsep}{2.5pt}
\caption{Five-state shares and acceptance events under original versus ReaL record-attached correctness (original/ReaL, percentages of the 473 matched images, averaged over models and all conditions). The last two columns give the share of original $h$ and $w$ records that change state under ReaL.}
\label{tab:real-states}

\footnotesize
\begin{tabular}{@{}llrrrrrrrrr@{}}\toprule
Source & Labels & $h$ & $r$ & $k$ & $s$ & $w$ & $H$ & $L$ & $h$ moved (\%) & $w$ moved (\%)\\\midrule
ES & original & 24.17 & 1.55 & 65.33 & 3.99 & 4.96 & 1.40 & 9.93 & -- & --\\
ES & ReaL & 22.02 & 1.30 & 62.03 & 8.01 & 6.64 & 1.27 & 9.38 & 18.4 & 16.0\\
Diverse & original & 64.00 & 0.62 & 25.49 & 1.16 & 8.73 & 2.08 & 8.31 & -- & --\\
Diverse & ReaL & 59.65 & 0.56 & 24.40 & 2.77 & 12.63 & 1.95 & 7.87 & 10.2 & 18.6\\
\bottomrule
\end{tabular}

\end{table}

\FloatBarrier
\subsection{Cross-model support for a clean wrong answer}
On the original-label 500-image evaluation pool, we select images a given model gets wrong when clean and measure how many \emph{other} models predict that same clean wrong answer. We compare exclusion of the evaluated model, exclusion of its whole group, and equal weighting of other groups. Low/high support means agreement at most 0.25/at least 0.75; the middle group contains the remaining values. The outcome is whether the model retains that same wrong answer after perturbation, with each group's clean-error events as denominator.

\begin{table}[!htbp]
\centering

\caption{Persistent same-wrong rate conditional on original-label clean-wrong events among all 500 images, stratified by other-family clean-side consensus (high $\geq0.75$; low $\leq0.25$).}
\label{tab:cross-model-consensus}

\small
\begin{tabular}{lrrrr}
\toprule
Domain & High (\%) & Low (\%) & High events & Low events \\
\midrule
ES & 61.45 & 27.10 & 16416 & 79704 \\
Diverse & 18.69 & 6.40 & 49248 & 239112 \\
\bottomrule
\end{tabular}
\end{table}

Excluding the evaluated group (Table~\ref{tab:cross-model-consensus}), high/low-support persistent-error rates are 61.45\%/27.10\% on ES and 18.69\%/6.40\% on Diverse. These rates weight model--condition--image events; they are neither the unconditional state proportion $s$ nor independent-image probabilities. The association does not control clean confidence and cannot establish an independent effect, a new predictive increment, or label correctness. Shared model errors may also reflect shared biases \citep{geirhos2020}.

\FloatBarrier
\subsection{External annotations and pending human review}
We match the detailed ImageNet-Mistakes resource by original image and the specific predicted label, not by image alone \citep{vasudevan2022imagenetmistakes}. The matching yields 16 distinct model--image--resource combinations across two base images, represented by 32 source-specific matches because ES and Diverse share the clean images. This sparse coverage cannot explain persistent errors across the full pool.

The planned two-reviewer assessment comprises a stratified sample of 200 of the 27,537 events across 18 strata, with inclusion probabilities equal to each stratum's sample-to-population ratio. Each event specifies a source, condition, base image and persistent wrong answer, and an image may recur across events. The 400 assigned ratings remain unannotated, so human assessment and adjudication results are pending. Reviewers are masked to model identity, consensus and ReaL annotations and evaluate the displayed candidate answer. Estimation will account for unequal inclusion probabilities and repeated images.


\end{document}